\documentclass{amsart}

\usepackage[margin=1.4in]{geometry}
\usepackage{mathtools}
\usepackage{mathabx}
\usepackage{booktabs}
\usepackage[T1]{fontenc}
\usepackage{amsmath,amsfonts,mathrsfs,amssymb,amsthm}
\usepackage{enumerate}
\usepackage{stmaryrd}
\usepackage{dsfont}
\usepackage{cite}
\usepackage[usenames,dvipsnames]{xcolor}
\definecolor{myblue}{rgb}{0.21, 0.34, 0.74}
\definecolor{mygrey}{rgb}{0.55, 0.57, 0.67}
\definecolor{myred}{rgb}{0.79, 0.0, 0.09}
\definecolor{mygreen}{rgb}{0.05, 0.5, 0.06}

\usepackage[colorlinks=true, citecolor=myblue, linkcolor=black]{hyperref}

\DeclareMathAlphabet{\mathscrbf}{OMS}{mdugm}{b}{n}
\usepackage[scr=boondoxo]{mathalpha}
\usepackage{cleveref}
\usepackage{subcaption}
\DeclareCaptionLabelFormat{empty}{}

\usepackage{comment}

\usepackage{wrapfig}
\usepackage{upgreek}

\newcommand{\calN}{\mathcal{N}}
\newcommand{\bR}{\mathbb{R}}
\newcommand{\R}{\bR}
\newcommand{\bC}{\mathbb{C}}
\newcommand{\bN}{\mathbb{N}}
\newcommand{\bP}{\mathbb{P}}
\newcommand{\p}{\bP}
\newcommand{\bE}{\mathbb{E}}
\newcommand{\calF}{\mathcal{F}}
\newcommand{\calB}{\mathcal{B}}
\newcommand{\norm}[1]{\left\| #1\right\|}
\allowdisplaybreaks[4]
\newcommand{\rs}{\textsf{rs}}
\newcommand{\nrs}{\textsf{nrs}}

\numberwithin{equation}{section}

\usepackage{comment}
\usepackage[font=small,labelfont=bf]{caption}

\theoremstyle{plain}
\newtheorem{theorem}{Theorem}[section]
\newtheorem{definition}[theorem]{Definition}
\newtheorem{proposition}[theorem]{Proposition}
\newtheorem{lemma}[theorem]{Lemma}

\newtheorem{remark}[theorem]{Remark}

\newtheorem{assumption}{Assumption}

\begin{document}

\title[Residual connections provably mitigate oversmoothing in GNNs]{Residual connections provably mitigate oversmoothing in graph neural networks}

\author[Z. Chen]{Ziang Chen}
\address{(ZC) Department of Mathematics, Massachusetts Institute of Technology, 77 Massachusetts Ave, 02139 Cambridge MA, USA} 
\email{ziang@mit.edu}

\author[Z. Lin]{Zhengjiang Lin}
\address{(ZL) Department of Mathematics, Massachusetts Institute of Technology, 77 Massachusetts Ave, 02139 Cambridge MA, USA} 
\email{linzj@mit.edu}

\author[S. Chen]{Shi Chen}
\address{(SC) Department of Mathematics, Massachusetts Institute of Technology, 77 Massachusetts Ave, 02139 Cambridge MA, USA} 
\email{schen636@mit.edu}

\author[Y. Polyanskiy]{Yury Polyanskiy}
\address{(YP) Department of Electrical Engineering and Computer Science, Massachusetts Institute of Technology, 77 Massachusetts Ave, 02139 Cambridge MA, USA}
\email{yp@mit.edu}

\author[P. Rigollet]{Philippe Rigollet}
\address{(PR) Department of Mathematics, Massachusetts Institute of Technology, 77 Massachusetts Ave, 02139 Cambridge MA, USA} 
\email{rigollet@math.mit.edu}


\date{\today}


\maketitle

\begin{abstract}

Graph neural networks (GNNs) have achieved remarkable empirical success in processing and representing graph-structured data across various domains. However, a significant challenge known as ``oversmoothing'' persists, where vertex features become nearly indistinguishable in deep GNNs, severely restricting their expressive power and practical utility. In this work, we analyze the asymptotic oversmoothing rates of deep GNNs with and without residual connections by deriving explicit convergence rates for a normalized vertex similarity measure. Our analytical framework is grounded in the multiplicative ergodic theorem. Furthermore, we demonstrate that adding residual connections effectively mitigates or prevents oversmoothing across several broad families of parameter distributions. The theoretical findings are strongly supported by numerical experiments.
\end{abstract}

\tableofcontents

\newpage
\section{Introduction}

Graph neural networks (GNNs) \cite{scarselli2008graph,wu2020comprehensive,zhou2020graph,kipf2016semi,velivckovic2017graph} have achieved significant empirical success in learning and representing graph data, with broad applications across fields such as physics \cite{shlomi2020graph}, bioinformatics \cite{zhang2021graph}, finance \cite{wang2021review}, electronic engineering \cite{liao2021review, he2021overview, lee2022graph}, and operations research \cite{gasse2019exact}.

Among many types of GNN architectures, the most foundational one is the message-passing graph neural network (MP-GNN), which employs a message-passing mechanism \cite{gilmer2017neural} to update each vertex's feature by aggregating information from its neighbors. Specifically, given an undirected and unweighted graph \( G = (V, E) \), where \( V = \{v_1,\dots,v_n\} \) is the vertex set and \( E \subseteq V\times V \) represents the edges, the vertex features $\{x_i(t):i=1,\dots,n\}$ at the \( t \)-th layer are computed as:
\begin{equation}\label{eq:MP}
    x_i{(t+1)} = f^{(t)}\left(x_i{(t)}, \mathrm{AGGREGATE}\left(\left\{\left\{g^{(t)}(x_j{(t)}):j\in\calN_i\right\}\right\}\right)\right),
\end{equation}
where \( f^{(t)} \) and \( g^{(t)} \) are trainable functions of local updates, \( \mathrm{AGGREGATE} \) is an aggregation function (e.g., sum, average, max, attention), \( \calN_i = \{j\in\{1,\dots,n\} : (v_i,v_j)\in E\} \) is the index set of neighbors of \( v_i \), and \( \{\{\dots\}\} \) denotes a multiset allowing repeated elements.

The widespread adoption and practical success of MP-GNNs can be attributed to two key features. The first is their permutation-invariant/equivariant property, which ensures that when vertices are relabeled, the vertex features undergo a corresponding relabeling, preserving the MP-GNN's outputs on isomorphic graphs up to a permutation. The second is their efficient and scalable implementation, where vertex features are updated according to~\eqref{eq:MP} using only local functions and neighborhood information from a small subgraph. Together, these properties make MP-GNNs particularly well-suited for graph learning tasks, allowing them to handle graphs of varying sizes and scale effectively across large datasets.

Despite their strengths, MP-GNNs face certain limitations that constrain their applications. A primary challenge is the \emph{oversmoothing} phenomenon in deep GNNs, which is well documented empirically~\cite{rusch2023survey,zhang2023comprehensive,chen2020measuring,li2018deeper,oono2019graph,wu2022non}. More specifically, vertex features tend to become increasingly similar as the network depth grows. Although researchers have proposed various mitigation strategies including residual connections, attention mechanisms, and normalization layers \cite{wu2024demystifying,chen2020measuring,rusch2023survey,scholkemper2024residual,dezoort2023principles}—oversmoothing remains a persistent challenge. This phenomenon indicates that deep MP-GNNs struggle to maintain vertex distinctiveness even when the underlying properties differ significantly. The resulting homogeneity in vertex features creates numerical difficulties during training, often forcing the use of shallow architectures and thereby limiting MP-GNNs' expressiveness and, in turn, hinders their performance in complex, large-scale applications.

This work focuses on understanding and mitigating oversmoothing in deep GNNs. Specifically, we conduct a theoretical investigation into the oversmoothing phenomenon in deep MP-GNNs, offering practical insights and guidance.

We analyze the asymptotic behavior of oversmoothing in two widely used families of MP-GNN architectures. The first architecture is given by
\begin{equation}\label{eq:GNN}\tag{\nrs}
    x_i{(t+1)} = \sum_{j=1}^n P_{ij}^{(t)} W^{(t)} \sigma\left( x_j{(t)}\right),
\end{equation}
where $P_{ij}^{(t)}\geq 0$ and $P_{ij}^{(t)}>0$ if and only if $j\in\calN_i$.
This architecture encompasses graph convolutional networks (GCNs) \cite{kipf2016semi}, where the aggregation coefficient matrix $P^{(t)}=P$ is time-independent and is given by \(P = D^{-1} A\) or $P = D^{-1/2} A D^{-1/2}$, with \(A \in \{0,1\}^{n \times n}\) as the symmetric adjacency matrix and \(D \in \mathbb{R}^{n \times n}\) as the degree matrix, a diagonal matrix whose diagonal entries represent the vertex degrees, and graph attention networks (GATs) \cite{velivckovic2017graph}, where $P^{(t)}$ is time-varying and is determined by the attention mechanism on the graph. 

The second architecture employs a residual connection, and is defined as
\begin{equation}\label{eq:res_GNN}\tag{\rs}
    x_i{(t+1)} = x_i{(t)} +\alpha\sum_{j=1}^n P_{ij}^{(t)} W^{(t)} \sigma\left( x_j{(t)}\right).
\end{equation}

In \eqref{eq:GNN} and \eqref{eq:res_GNN}, $W^{(t)}\in\bR^{d\times d}$ is a matrix of learnable parameters with $d$ being the dimension of vertex feature, $P_{ij}^{(t)}$ represents the coefficients in information aggregation, $\sigma:\bR\to\bR$ is an activation function, and $\alpha>0$ is a constant. By examining these architectures, we aim to provide a deeper understanding of oversmoothing and strategies to alleviate its impact.

\medskip

\paragraph{\textbf{Related work}} Several works in the literature have analyzed the oversmoothing phenomenon in deep GNNs from a theoretical perspective. 

For GNNs without residual connections, \cite{li2018deeper,keriven2022not} prove that for GCNs with time-independent $P^{(t)} = P$, the vertex features converge to a subspace of identical or highly correlated features. This analysis is further refined in \cite{oono2019graph,cai2020note}, showing that the convergence occurs at an exponential rate, determined by the eigengap of the normalized graph Laplacian. Additionally, non-asymptotic analysis for GCNs is provided in \cite{wu2022non}. For GATs without residual connections and with time-varying $P^{(t)}$, \cite{wu2024demystifying} demonstrates that they also suffer from oversmoothing at an exponential rate. Moreover, the upper bound of the convergence rate established in \cite{wu2024demystifying} is shown to be greater than or equal to the rate derived for GCNs in \cite{oono2019graph,cai2020note}. While this observation pertains only to upper bounds that could be loose, it points to a potential mitigation of oversmoothing in GATs even in absence of a residual connection.

Previous works have also analyzed GNNs with residual connections. \cite{scholkemper2024residual} demonstrates that oversmoothing can be prevented by introducing a residual connection between each hidden layer and the \emph{initial} layer which amounts to adding a constant drift $x_i(0)$ to the right-hand side of~\eqref{eq:GNN}. The nature of this remedy is in sharp contrast to \eqref{eq:res_GNN}, where the residual connection links consecutive layers as in classical ResNets~\cite{he2016deep}. 

Closer to our setting in \eqref{eq:res_GNN}, \cite{dezoort2023principles} investigates the oversmoothing phenomenon in residual GNNs when $\sigma$ is a ReLU activation and the entries of $W^{(t)}$ are i.i.d. Gaussian. They prove that for a deep residual GNN with $T$ is message-passing layers, the oversmoothing issue can be prevented if the stepsize $\alpha$ is chosen as small, say $\alpha=\Theta(1/T)$. 
\medskip

\paragraph{\textbf{Our contribution}} We define a normalized vertex similarity measure $\mu(x) \in[0,1]$ for $x=(x_1,\dots,x_n)\in\bR^{d\times n}$, where $\mu(x)=0$ indicates if and only if all vertex features are identical, and analyze the asymptotic behavior of $\mu(x_{\nrs}(t))$ and $\mu(x_{\rs}(t))$, where $\{x_{\nrs}(t)\}_{t\in\bN}$ and $\{x_{\rs}(t)\}_{t\in\bN}$ are generated by \eqref{eq:GNN} and \eqref{eq:res_GNN}, respectively. 

Our analysis relies on three main assumptions: (i) $\sigma=\mathrm{Id}$, (ii) $P^{(t)}\equiv P\in \bR^{n\times n}$  is a nonnegative matrix for which the Perron-Frobenius theorem applies with leading eigenvector $\mathbf{1}=(1,\dots,1)^\top\in\bR^n$ associated to the leading eigenvalue assumed equal to $1$, and (iii) $W^{(t)},\ t\in\bN$ are sampled i.i.d. from some probability distribution on $\bR^{d\times d}$. Note that (ii) holds when $P$ is the probability transition matrix of a simple random walk on a graph that is irreducible and aperiodic, as in~\cite{norris1998markov}. A significant strength of our analysis is that it extends readily to asymmetric matrices $P$.

Our main contributions are summarized informally as follows:
\begin{itemize}
    \item \emph{Non-residual dynamics:} we prove that $\mu(x_{\nrs}(t))^{1/2t}\to \max_{\lambda\in \mathrm{spec}(P)\backslash \{1\}}|\lambda| < 1$ as $t\to+\infty$ almostly surely, and the limit is \emph{independent} of the distribution of $W^{(t)}$. {This result confirms the findings of \cite{oono2019graph, cai2020note} that $\mu(x_{\nrs}(t))$ converges to zero exponentially fast. It offers greater precision by establishing the exact rate of convergence, whereas \cite{oono2019graph, cai2020note} only present upper bounds.}
    \item \emph{Residual dynamics:} We show that almost surely, $\lim_{t\to\infty} \mu(x_{\rs}(t))^{1/2t}$ is lower bounded by a constant depending on the spectrum of $P$ and the distribution of $W^{(t)}$. The lower bound is achieved if $P$ is diagonalizable in $\bC$. Furthermore, this bound can be explicitly computed or estimated for several commonly used distributions of $W^{(t)}$. {In comparison to \cite{dezoort2023principles}, our analysis encompasses a broader class of parameter distributions, not limited to i.i.d. Gaussian. Additionally, we impose no requirement for $\alpha$ to depend on the number of layers, whereas \cite{dezoort2023principles} carefully selects $\alpha$ based on the predetermined depth of GNNs.}
    \item \emph{Comparison of non-residual and residual dynamics:} For several distributions of $W^{(t)}$, we rigorously demonstrate that $\lim_{t\to\infty} \mu(x_{\rs}(t))^{1/2t}>\lim_{t\to\infty} \mu(x_{\nrs}(t))^{1/2t}$, and in some cases, $\lim_{t\to\infty} \mu(x_{\rs}(t))^{1/2t}=1$, which indicates that $\mu(x_{\rs}(t)$ can only converge to zero at a subexponential rate. This shows that residual connections effectively mitigate the oversmoothing issue.
\end{itemize}
{These results rely on a new and unified analytical framework that rigorously demonstrates that residual connections mitigate or prevent oversmoothing in deep GNNs.
Our analysis accounts for the presence of complex eigenvalues in $P$, whereas prior works \cite{oono2019graph, cai2020note, dezoort2023principles} assume symmetry or a real spectrum for $P$.} Our proof techniques draw inspiration from the multiplicative ergodic theorem \cite{oseledets, Arnold_RDS}, which precisely characterizes the asymptotic behavior of linear random dynamical systems or products of random matrices.

\medskip

\paragraph{\textbf{Organization}} The rest of this paper will be organized as follows. We state our main theory on the asymptotic oversmoothing rate and their applications for several commonly used distributions in \Cref{sec:main_results} . All proofs are presented in \Cref{sec:proof} and numerical experiments are conducted in \Cref{sec:numerics}. \Cref{sec:conclude} concludes the paper.

\medskip

\section{Main Results}
\label{sec:main_results}

This section presents our main results. In \Cref{sec:oversmoothing_rate}, we establish the asymptotic oversmoothing rates for \eqref{eq:GNN} and \eqref{eq:res_GNN}. In \Cref{sec:application}, we compare the oversmoothing rates of \eqref{eq:GNN} and \eqref{eq:res_GNN} for several specific distributions.

\subsection{Asymptotic oversmoothing rate}
\label{sec:oversmoothing_rate}
We will use the (normalized) vertex similarity measure introduced below to quantify the degree of oversmoothing.

\begin{definition}[Vertex similarity measure]\label{def:vertex sim}
    For $x = (x_1,\dots,x_n)\in\bR^{d\times n}$, we define (normalized) vertex similarity measure as
    \begin{equation}\label{eq:vertex_sim}
        \mu(x) = \frac{\sum_{i=1}^n \|x_i - \Bar{x}\|_2^2}{\sum_{i=1}^n \|x_i\|_2^2},
    \end{equation}
    where $\Bar{x} = \frac{1}{n}\sum_{i=1}^n x_i$.
\end{definition}

It is straightforward to check that $\mu(x)\in[0,1]$ for any $x\in\bR^{d\times n}\backslash\{0\}$ and that  $\mu(x)=0$ if and only if $x_1=\cdots=x_n$. Next, we outline the assumptions used in our analysis. The first assumption is that $\sigma=\mathrm{Id}$, which simplifies \eqref{eq:GNN} and \eqref{eq:res_GNN} to linear dynamics.

\begin{assumption}[Activation function]\label{asp:sigma}
    We assume that the activation function $\sigma:\bR\to\bR$ is the identity map, i.e., the dynamics in \eqref{eq:GNN} and \eqref{eq:res_GNN} are linear.
\end{assumption}

The second assumption concerns the aggregation coefficients $P^{(t)}$.

\begin{assumption}[Aggregation coefficients]\label{asp:P}
    We assume that the aggregation coefficients $P^{(t)}\in\bR^{n\times n}$ satisfy the followings:
    \begin{itemize}
        \item[(i)] $P^{(t)} = P$ is a constant matrix in $t\in\mathbb{N}$.
        \item[(ii)] $P$ is a primitive matrix:  $\exists \ k\ge 1$ such that  for any $i,j$,  $P_{ij}\geq 0$ and $P_{ij}^k>0$.
        \item[(iii)] $P_{ij}>0$ if and only if $j\in\calN_i$, and for any $i\in\{1,\dots,n\}$, it holds that
        \begin{equation}\label{eq:P_row_sum}
            \sum_{j=1}^n P_{ij} = \sum_{j\in\calN_i} P_{ij} = 1.
        \end{equation}
    \end{itemize}
\end{assumption}

\Cref{asp:P} states that $P^{(t)} = P\in\bR^{n\times n}$ is a probability transition matrix of an irreducible aperiodic Markov chain on the underlying graph $G$.
Moreover, for a nonnegative matrix $P\geq 0$ with $P_{ij}>0$ if and only if $j\in\calN_i$, the primitivity of $P$, or equivalently the irreducibility and aperiodicity of $P$, can be ensured by specific properties of the graph $G$: namely, $G$ is connected and contains at least one odd-length cycle.

\begin{remark}\label{rmk:spectrum_P}
    Suppose that \Cref{asp:P} holds. By the Perron-Frobenius theorem, the spectral radius of the matrix $P$ is $1$, and $\lambda_1=1$ is an eigenvalue with multiplicity being one. One can see from \eqref{eq:P_row_sum} that $\mathbf{1}=(1,\dots,1)^\top\in\bR^n$ is an eigenvector corresponding to $\lambda_1=1$. Moreover, the magnitudes of all other eigenvalues of $P$ are strictly smaller than $1$, i.e.,
    \begin{equation*}
        |\lambda| < 1,\quad\forall~\lambda\in\mathrm{spec}(P)\backslash\{1\},
    \end{equation*}
    where $\mathrm{spec}(P)\subseteq\bC$ is the spectrum of $P$. Note that we do not assume that $P$ is symmetric as in \cite{dezoort2023principles}, and hence the spectrum of $P$ might be complex.
\end{remark}

We also assume that the weight matrices $W_1, W_2, \ldots$ are i.i.d. from some ensemble.

\begin{assumption}[Weight matrices]\label{asp:W}
    We assume that $W^{(t)},\ t\in\mathbb{N}$ are i.i.d. drawn from a Borel probability measure $\bP_W$ over $\bR^{d\times d}$ that satisfies that
    \begin{equation}\label{eq:logW_L1}
        \bE\left[\max\left\{\log \|W^{(t)}\|_2,0\right\}\right]<+\infty.
    \end{equation}
\end{assumption}

Moreover, we require that the probability distribution $\bP_W$ satisfies some non-degenerate conditions. To state them rigorously, we present the next proposition that will be proved in \Cref{sec:MET}.

\begin{proposition}\label{prop:R_PW}
    Suppose that \Cref{asp:W} holds. The followings are true.
    \begin{itemize}
        \item[(i)] There exists a constant $R(\bP_W)\geq 0$ depending on the probability distribution $\bP_W$, such that
        \begin{equation}\label{eq:R_PW}
            \lim_{t\to+\infty} \norm{W^{(t-1)} \cdots W^{(1)} W^{(0)}}_2^{1/t} = R(\bP_W), \quad \text{a.s.}
        \end{equation} 
        \item[(ii)] For any $\beta\in\bC$, there exists a constant $R(\beta,\bP_W)\geq 0$ depending on $\beta$ and the probability distribution $\bP_W$, such that
        \begin{equation}\label{eq:R_beta_PW}
            \lim_{t\to+\infty} \norm{\left(\mathrm{Id}+\beta W^{(t-1)}\right) \cdots \left( \mathrm{Id}+\beta W^{(1)}\right) \left( \mathrm{Id} + \beta W^{(0)}\right)}_2^{1/t} = R(\beta,\bP_W), \quad \text{a.s.}
        \end{equation}
    \end{itemize}
\end{proposition}

Our last set of assumptions is on the constants in \Cref{prop:R_PW}.

\begin{assumption}\label{asp:R_PW}
    The constant $R(\bP_W)$ defined in \Cref{prop:R_PW} (i) satisfies that $R(\bP_W)>0$.
\end{assumption}

\begin{assumption}\label{asp:R_beta_PW}
    The spectrum of $P$ and the constants in \Cref{prop:R_PW} (ii) satisfy that $\max_{\lambda\in \mathrm{spec}(P)} R(\alpha\lambda,\bP_W) > 0$, where $\alpha > 0$ is defined in~\eqref{eq:res_GNN}.
\end{assumption}

With the assumptions outlined above, we are now prepared to present our main results. The first result characterizes the asymptotic oversmoothing rate of deep GNNs described by \eqref{eq:GNN} without residual connections.

\begin{theorem}[Asymptotic oversmoothing rate of deep non-residual GNNs]\label{thm:nonres_GNN}
    Consider the dynamics $\{x_{\mathrm{nrs}}(t)\}_{t\in\bN}$ generated by \eqref{eq:GNN}, and suppose that Assumptions~\ref{asp:sigma}, \ref{asp:P}, \ref{asp:W}, and \ref{asp:R_PW} hold. With probability one, we have for almost every $x_{\mathrm{nrs}}(0)\in\bR^{d\times n}$  that
    \begin{equation}\label{eq:GNN_rate}
        \lim_{t\to+\infty} \mu(x_{\mathrm{nrs}}{(t)})^{1/2t} = \max_{\lambda\in \mathrm{spec}(P)\backslash \{1\}}|\lambda| < 1.
    \end{equation}
\end{theorem}

\Cref{thm:nonres_GNN} establishes that \(\mu(x_{\mathrm{nrs}}{(t)})\) converges to \(0\) at an exponential rate determined solely by the spectrum of \(P\), independent of the distribution of \(W^{(t)}\). Specifically, the convergence rate is governed by the eigenvalue \(\lambda\) with the second largest magnitude. Notably, \Cref{thm:nonres_GNN} remains valid even when \(P\) has complex eigenvalues. The proof of \Cref{thm:nonres_GNN} are provided in \Cref{section:proof of main GNN}, following an introduction to linear random dynamical systems and to the multiplicative ergodic theorem in \Cref{sec:MET}.

In the special case of GCNs where \(P = D^{-1} A\), the gap \(1 - \max_{\lambda \in \mathrm{spec}(P) \setminus \{1\}} |\lambda|\) corresponds to the second smallest eigenvalue of the normalized graph Laplacian \(\Delta = I - D^{-1/2} A D^{-1/2}\). Moreover, recall that for graphs satisfying \Cref{asp:P}, all eigenvalues of $\Delta$ lie in $[0,2)$, and there is exactly one eigenvalue equal to $0$. Hence, the convergence rate established in \Cref{thm:nonres_GNN} strengthens the upper bounds of \cite{oono2019graph, cai2020note} in the case of an identity activation function.

Our second main theorem characterizes the asymptotic oversmoothing rate of deep residual GNNs as described by \eqref{eq:res_GNN}.

\begin{theorem}[Asymptotic oversmoothing rate of deep residual GNNs]\label{thm:res_GNN}
    Consider the dynamics $\{x_{\mathrm{rs}}(t)\}_{t\in\bN}$ generated by \eqref{eq:res_GNN}, and suppose that Assumptions~\ref{asp:sigma}, \ref{asp:P}, \ref{asp:W}, and \ref{asp:R_beta_PW} hold.
   With probability one, we have for almost every  $x_{\mathrm{rs}}(0)\in\bR^{d\times n}$ that 
    \begin{equation}\label{eq:resGNN_rate_lb}
        \lim_{t\to+\infty} \mu(x_{\mathrm{rs}}(t))^{1/2t} \geq \frac{ \max\limits_{\lambda\in \mathrm{spec}(P)\backslash\{1\}} R(\alpha\lambda,\bP_W)}{ \max\limits_{\lambda\in \mathrm{spec}(P)} R(\alpha\lambda,\bP_W)}.
    \end{equation}
    Moreover, if $P$ is diagonalizable in $\bC$, then the inequality above becomes an equality. 
\end{theorem}

\Cref{thm:res_GNN} establishes a lower bound for \(\lim_{t \to +\infty} \mu(x_{\mathrm{rs}}(t))^{1/2t}\), which depends on the distribution \(\mathbb{P}_W\), the spectrum of \(P\), and the step size \(\alpha\) in the residual connection. This lower bound is achieved when \(P\) is diagonalizable in \(\mathbb{C}\). The proof of \Cref{thm:res_GNN} is presented in \Cref{section:proof of main res GNN} and is also based on linear random dynamical systems and the multiplicative ergodic theorem.

From \Cref{thm:nonres_GNN} and \Cref{thm:res_GNN}, one can conclude that the oversmoothing issue is mitigated if the lower bound in \eqref{eq:resGNN_rate_lb} is strictly greater than \(\max_{\lambda \in \mathrm{spec}(P) \setminus \{1\}} |\lambda|\). Below, we discuss several commonly used probability distributions \(\mathbb{P}_W\) for which the lower bound in \eqref{eq:resGNN_rate_lb} can be explicitly computed or estimated. In these cases, it is either strictly greater than \(\max_{\lambda \in \mathrm{spec}(P) \setminus \{1\}} |\lambda|\) or equal to \(1\).

To conclude this section, we note that if we do not assume the row sums of \(P\) are equal to 1, as in \eqref{eq:P_row_sum}, our theory still holds for a modified vertex similarity measure, given by
\begin{equation}\label{eq2:vertex_sim}
    \mu(x) = \frac{\|x - x \pi_1 \pi_1^\top\|_F^2}{\|x\|_F^2},
\end{equation}
where \(\|\cdot\|_F\) denotes the Frobenius norm, and \(\pi_1 \in \mathbb{R}^n\) is the eigenvector corresponding to the leading eigenvalue of \(P\), with \(\|\pi_1\|_2 = 1\) and strictly positive entries, as guaranteed by the Perron-Frobenius theorem. The proofs almost follow the same lines. If \eqref{eq:P_row_sum} is assumed, then \(\pi_1 = \frac{1}{\sqrt{n}} (1, \dots, 1)^\top\), and \eqref{eq2:vertex_sim} coincides with \eqref{eq:vertex_sim}. If the entries of $\pi_1$ are not identical, then $\mu(x)=0$ does not imply that $x_1, \ldots, x_n$ are equal but rather that they all lie on a one-dimensional subspace of $\bR^d$.

\subsection{Weight matrices}
\label{sec:application}

Comparing Theorems~\ref{thm:nonres_GNN} and~\ref{thm:res_GNN} indicates that the distribution $\bP_W$ of the weight matrix controls the oversmoothing of residual GNNs whereas it has no impact on their nonresidual counterparts. In this section, we examine the effect of this distribution in  more details in the context of specific choices for $\bP_W$. 

In all the cases considered here indicate that residual connections effectively mitigate oversmoothing over a wide class of distributions. All the proofs are deferred to \Cref{sec:pf_application}.

\subsubsection{Deterministic}
We begin by considering the case where $W^{(t)}=W$ is deterministic, that is where $\p_W$ is a point mass at $W \in \R^{d \times d}$. In this case, equality holds in  \eqref{eq:resGNN_rate_lb} even if $\bP$ is not diagonalizable in $\bC$.

\begin{theorem}\label{thm:deterministic_W}
     Suppose that Assumptions~\ref{asp:sigma}, \ref{asp:P}, \ref{asp:W}, and \ref{asp:R_beta_PW} hold with $W^{(t)} = W$ almost surely, where $W$ is a matrix with constant entries. Let $\{x_{\mathrm{rs}}(t)\}_{t\in\bN}$ be   generated by \eqref{eq:res_GNN}. Then we have for almost every $x_{\mathrm{rs}}(0)\in\bR^{d\times n}$ that 
    \begin{equation*}
        \lim_{t\to+\infty} \mu(x_{\mathrm{rs}}{(t)})^{1/2t} = \frac{\max\limits_{\lambda\in \mathrm{spec}(P)\backslash\{1\}} \rho(\mathrm{Id}+\alpha\lambda W)}{\max\limits_{\lambda\in \mathrm{spec}(P)} \rho(\mathrm{Id}+\alpha\lambda W)},
    \end{equation*}
    where $\rho(A) = \max\limits_{\lambda\in\mathrm{spec}(A)}|\lambda|$ is the spectral radius of a matrix $A$.
\end{theorem}

The following theorem characterizes several cases with deterministic $W^{(t)}$ in which the oversmoothing problem is provably mitigated or prevented.

\begin{theorem}\label{thm2:deterministic_W}
    Suppose that Assumptions~\ref{asp:sigma}, \ref{asp:P}, \ref{asp:W}, \ref{asp:R_PW}, and \ref{asp:R_beta_PW} hold
    with $W^{(t)} = W$ almost surely, and in addition that $P$ has a real spectrum.
    Let $\{x_{\mathrm{nrs}}(t)\}_{t\in\bN}$ and $\{x_{\mathrm{rs}}(t)\}_{t\in\bN}$ be   generated by \eqref{eq:GNN} and \eqref{eq:res_GNN}, respectively. We also assume that $\rho(I+\alpha W)$ is achieved by $|1+\alpha\mu_0|$ for  $\mu_0 \in\mathrm{spec}(W)$.

    \begin{itemize}
        \item[(i)] If $\max\limits_{\lambda\in \mathrm{spec}(P)\backslash \{1\}}|\lambda|$ is achieved by $\lambda_0\in \mathrm{spec}(P)\backslash \{1\}$ and $1+\lambda_0(2\alpha\mathrm{Re}(\mu_0)+1)>0$, then for almost every $x_{\mathrm{nrs}}(0)\in\bR^{d\times n}$ and $x_{\mathrm{rs}}(0)\in\bR^{d\times n}$,
        \begin{equation*}
            \lim_{t\to+\infty} \mu(x_{\mathrm{rs}}{(t)})^{1/2t} > \max_{\lambda\in \mathrm{spec}(P)\backslash \{1\}}|\lambda| = \lim_{t\to+\infty} \mu(x_{\mathrm{nrs}}{(t)})^{1/2t}.
        \end{equation*}
        \item[(ii)] If $2\mathrm{Re}(\mu_0)+ \alpha (\lambda+1)|\mu_0|^2\leq 0$ for some $\lambda\in \mathrm{spec}(P)\backslash \{1\}$, then for almost every $x_{\mathrm{rs}}(0)\in\bR^{d\times n}$,
        \begin{equation*}
            \lim_{t\to+\infty} \mu(x_{\mathrm{rs}}{(t)})^{1/2t} = 1.
        \end{equation*}
    \end{itemize}
\end{theorem}

We remark that the condition that $P$ has a real spectrum holds for GCNs with $P = D^{-1} A$ since $\mathrm{spec}(D^{-1} A) = \mathrm{spec}(D^{-1/2} A D^{-1/2}) \subset (-1,1]$. 

\subsubsection{Ginibre ensemble} Next, consider the case where the entries of \(W^{(t)}\) are i.i.d. Gaussian with a mean of 0 and variance \(\tau^2\) (Ginibre ensemble) as in~\cite{dezoort2023principles}. The following theorem provides a characterization of the lower bound in \eqref{eq:resGNN_rate_lb} in terms of one-dimensional normal distribution and chi-squared distribution.

\begin{theorem}\label{thm:gaussian_W}
    Suppose that Assumptions~\ref{asp:sigma}, \ref{asp:P}, \ref{asp:W}, and \ref{asp:R_beta_PW} hold with entries of $W^{(t)}$ i.i.d. drawn from $\calN(0,\tau^2)$, and in addition that $P$ has a real spectrum. Let $\{x_{\mathrm{rs}}(t)\}_{t\in\bN}$ be   generated by \eqref{eq:res_GNN}. Then with probability one, for almost every $x_{\mathrm{rs}}(0)\in\bR^{d\times n}$,
    \begin{equation}\label{eq:gaussian_W}
        \lim_{t\to+\infty} \mu(x_{\mathrm{rs}}(t))^{1/2t} \geq \frac{ \max\limits_{\lambda\in \mathrm{spec}(P)\backslash\{1\}} \exp\left(\frac{1}{2} \bE \log \left((1  + \alpha\lambda\tau \xi) ^2 + \alpha^2\lambda^2\tau^2 \chi_{d-1} ^2 \right)\right)}{ \max\limits_{\lambda\in \mathrm{spec}(P)} \exp\left(\frac{1}{2} \bE \log \left((1  + \alpha\lambda\tau \xi) ^2 + \alpha^2\lambda^2\tau^2 \chi_{d-1} ^2 \right)\right)},
    \end{equation}
    where $\xi\sim\calN(0,1)$ is a one-dimensional standard real Gaussian distribution, and $\chi^2_{d-1}$ is a chi-squared random variable with $d-1$ degrees of freedom that is independent of $\xi$. Additionally, if $P$ is diagonalizable in $\bC$, then the equality in \eqref{eq:gaussian_W} holds.
\end{theorem}

With \eqref{eq:gaussian_W}, \(\lim_{t \to +\infty} \mu(x_{\mathrm{rs}}(t))^{1/2t}\) can be shown as strictly larger than \(\lim_{t \to +\infty} \mu(x_{\mathrm{nrs}}(t))^{1/2t}\), for any $\alpha>0$ and $\tau>0$.

\begin{theorem}\label{thm2:gaussian_W}
    Suppose that Assumptions~\ref{asp:sigma}, \ref{asp:P}, \ref{asp:W}, \ref{asp:R_PW}, and \ref{asp:R_beta_PW} hold with entries of $W^{(t)}$ i.i.d. drawn from $\calN(0,\tau^2)$, and in addition that $P$ has a real spectrum. Let $\{x_{\mathrm{nrs}}(t)\}_{t\in\bN}$ and $\{x_{\mathrm{rs}}(t)\}_{t\in\bN}$ be   generated by \eqref{eq:GNN} and \eqref{eq:res_GNN}, respectively. It holds for all $\alpha>0$ and $\tau>0$ that, with probability one, for almost every $x_{\mathrm{nrs}}(0)\in\bR^{d\times n}$ and $x_{\mathrm{rs}}(0)\in\bR^{d\times n}$,
    \begin{equation*}
        \lim_{t\to+\infty} \mu(x_{\mathrm{rs}}{(t)})^{1/2t} > \max_{\lambda\in \mathrm{spec}(P)\backslash \{1\}}|\lambda| = \lim_{t\to+\infty} \mu(x_{\mathrm{nrs}}{(t)})^{1/2t}.
    \end{equation*}
\end{theorem}

\subsubsection{Bounded norm} Useful quantitative lower bounds can be extracted under minimal assumptions on the distribution of weight matrices. In this section, we examine compactly supported distributions for which it can be proven that \(\lim_{t \to +\infty} \mu(x_{\mathrm{rs}}(t))^{1/2t}\) is strictly greater than \(\lim_{t \to +\infty} \mu(x_{\mathrm{nrs}}(t))^{1/2t}\), provided that \(\alpha\) is sufficiently small.

\begin{theorem}\label{thm:bounded_W}
    Suppose that Assumptions~\ref{asp:sigma}, \ref{asp:P}, \ref{asp:W}, \ref{asp:R_PW}, and \ref{asp:R_beta_PW} hold with $\|W^{(t)}\|\leq r_W$ almost surely. Let $\{x_{\mathrm{nrs}}(t)\}_{t\in\bN}$ and $\{x_{\mathrm{rs}}(t)\}_{t\in\bN}$ be   generated by \eqref{eq:GNN} and \eqref{eq:res_GNN}, respectively. Then it holds for any $\alpha\leq 1/r_W$ that, with probability one, for almost every $x_{\mathrm{rs}}(0)\in\bR^{d\times n}$,
    \begin{equation*}
        \lim_{t\to+\infty} \mu(x_{\mathrm{rs}}{(t)})^{1/2t} \geq \frac{1-\alpha r_W  \min_{\lambda\in\mathrm{spec}(P)\backslash\{1\}}|\lambda|}{1+\alpha r_W}.
    \end{equation*}
    As a corollary, if 
            \begin{align*}
           \frac{\displaystyle 1-\max_{\lambda\in \mathrm{spec}(P)\backslash \{1\}}|\lambda|}{\displaystyle 1-\min_{\lambda\in \mathrm{spec}(P)\backslash \{1\}}|\lambda|}> \frac{\alpha r_W}{1+\alpha r_W},
        \end{align*}
    then with probability one, we have for almost every $x_{\mathrm{nrs}}(0)\in\bR^{d\times n}$ and $x_{\mathrm{rs}}(0)\in\bR^{d\times n}$ that
    \begin{equation*}
        \lim_{t\to+\infty} \mu(x_{\mathrm{rs}}{(t)})^{1/2t} > \lim_{t\to+\infty} \mu(x_{\mathrm{nrs}}{(t)})^{1/2t}.
    \end{equation*}
\end{theorem}

\subsubsection{Simultaneously diagonalizable}

Finally, we consider the scenario where all matrices in the support of \(\bP\) are simultaneously diagonalizable.

\begin{assumption}\label{asp:diagonal_W}
    We assume that $W^{(t)}\sim\bP_W$ satisfies
    \begin{equation*}
        W^{(t)} = Q^{-1} \text{diag}\left(w_1^{(t)},\dots, w_d^{(t)}\right) Q,
    \end{equation*}
    where $Q\in \bR^{d\times d}$ is a fixed invertible matrix and $w_i^{(t)},\ t\in\bN$ are i.i.d. drawn from some Borel probability distribution $\bP_i$ over $\bR$ satisfying
    \begin{equation*}
    \log \big|1+\alpha\lambda w_i^{(t)}\big| |\in L^1(\bP_i),\quad\forall~\lambda\in\mathrm{spec}(P).
    \end{equation*}
    for $i=1,\dots,d$.
\end{assumption}

Under \Cref{asp:diagonal_W}, the following theorem represents the constant \(R(\alpha\lambda, \bP_W)\) in terms of one-dimensional distributions.

\begin{theorem}\label{thm:diagonal_W}
    Suppose that Assumptions~\ref{asp:sigma}, \ref{asp:P}, \ref{asp:W}, \ref{asp:R_beta_PW}, and \ref{asp:diagonal_W} hold. Let $\{x_{\mathrm{rs}}{(t)}\}_{t\in\bN}$ be   generated by \eqref{eq:res_GNN}. Then with probability one, we have for almost every $x_{\mathrm{rs}}(0)\in\bR^{d\times n}$ that
    \begin{equation}\label{eq:diagonal_W}
        \lim_{t\to+\infty} \mu(x_{\mathrm{rs}}{(t)})^{1/2t} \geq \frac{\max\limits_{\lambda\in\mathrm{spec}(P)\backslash\{1\}} \max\limits_{1\leq i\leq d} \exp\left( \bE\left[\log \left|1+\alpha\lambda w_i^{(t)}\right|\right]\right)}{\max\limits_{\lambda\in\mathrm{spec}(P)} \max\limits_{1\leq i\leq d} \exp\left( \bE\left[\log \left|1+\alpha\lambda w_i^{(t)}\right|\right]\right)}.
    \end{equation}
    Additionally, if $P$ is diagonalizable in $\bC$, then the equality in \eqref{eq:diagonal_W} holds.
\end{theorem}

Furthermore, when $P$ has a real spectrum and the distribution $\bP_1,\dots,\bP_d$ are all symmetric, then $\mu(x_{\mathrm{rs}}{(t)})^{1/2t}$ converges to $1$ almost surely as stated in the next theorem.

\begin{theorem}\label{thm2:diagonal_W}
    Suppose that assumptions made in \Cref{thm:diagonal_W} are satisfied. Suppose in addition that $P$ has a real spectrum and that $\bP_i$ is symmetric, i.e., $\bP_i(B) = \bP_i(-B)$ for any Borel subset $B\subseteq \bR$, for $i=1,2,\dots, d$. Let $\{x_{\mathrm{rs}}{(t)}\}_{t\in\bN}$ be generated by \eqref{eq:res_GNN}.
    If $|w_i^{(t)}|\leq r_i$ almost surely, $i=1,\dots,d$, and $\alpha \max_{1\leq i\leq d}r_i<1$, it holds with probability one that, for almost every $x_{\mathrm{rs}}(0)\in\bR^{d\times n}$,
    \begin{equation*}
        \lim_{t\to+\infty} \mu(x_{\mathrm{rs}}{(t)})^{1/2t} = 1.
    \end{equation*}
\end{theorem}

\section{Proofs}
\label{sec:proof}
We collect all proofs in this section.

\subsection{Preliminary on linear random dynamical system} 
\label{sec:MET}
This subsection reviews some definitions and fundamental results for linear random dynamical systems; see \cite{Arnold_RDS} and references therein for more details. 

\begin{definition}[Metric dynamical system] A metric dynamical system on a probability space $(\Omega,\mathcal{F},\bP)$ is a family of maps $\{\theta(t):\Omega\rightarrow\Omega\}_{t\in\bN}$ satisfying that
\begin{itemize}
	\item[(i)] The mapping $\bN\times\Omega\rightarrow\Omega,\ (t,\omega)\mapsto\theta(t)\omega$ is measurable.
			
	\item[(ii)] It holds that
    \begin{equation}\label{eq:theta_tpluss}
        \theta(0)=\mathrm{Id}_{\Omega}\quad\text{and}\quad\theta(t+s)=\theta(t)\circ\theta(s),\ \forall~s,t\in\bN.
    \end{equation}
			
	\item[(iii)] $\theta(t)$ is $\bP$-preserving for any $t\in\bN$, i.e.,
\begin{equation*}
    \bP(\theta(t)^{-1}B)=\bP(B),\quad \forall~B\in\mathcal{F},\ t\in\bN.
\end{equation*}
\end{itemize}
\end{definition}

In the above definition, the index $t\in\bN$ serves as the notion of one-sided discrete time. Though this paper focuses only on time represented by non-negative integers, we remark that in general the metric dynamical system is defined with a time set $\mathbb{T}$ that can be a semigroup with $0\in\mathbb{T}$, including two-sided discrete time $\mathbb{T}=\mathbb{Z}$, one-sided continuous time $\mathbb{T}=\bR_{\geq 0}$, and two-sided continuous time $\mathbb{T}=\bR$.

In the case of $\mathbb{T}=\bN$, one can set $\theta=\theta(1)$ and \eqref{eq:theta_tpluss} implies that $\theta(t)$ is the $t$-fold composition of $\theta$, namely, $\theta(t)=\theta^t = \theta\circ\cdots\circ\theta$.

\begin{definition}[Linear random dynamical system]\label{def:linear_RDS}
Let $\{\theta^t:\Omega\rightarrow\Omega\}_{t\in\bN}$ be a metric dynamical system on $(\Omega,\mathcal{F},\bP)$ and let $A:\Omega\to \bR^{m\times m}$ be measurable. A linear random dynamical system on $\bR^m$ 
over $\{\theta^t\}_{t\in\bN}$, induced by $A$, is the measurable map 
\begin{equation*}
\begin{split}
	\Phi:\bN\times\Omega\times \bR^m&\to \quad \bR^m,\\
	(t,\omega,u)\quad  & \mapsto \Phi(t,\omega) u,
\end{split}
\end{equation*}
where 
\begin{equation}\label{eq:Phi_t_omega}
	\Phi(t,\omega)=A(\theta^{t-1}\omega)\cdots A(\theta\omega)A(\omega).
\end{equation}
\end{definition}

We have slightly abused the notation by writing $\Phi(t,\omega,u)=\Phi(t,\omega) u$. 
The asymptotic behavior of a linear random dynamical system can be characterized precisely by the celebrated multiplicative ergodic theorem in the following.

\begin{theorem}[Multiplicative ergodic theorem, {\cite[Theorem 3.4.1]{Arnold_RDS}}]
\label{thm: MET}
Let $\{\theta^t:\Omega\rightarrow\Omega\}_{t\in\bN}$ be a metric dynamical system on $(\Omega,\mathcal{F},\bP)$ and let $A:\Omega\to \bR^{m\times m}$ be measurable. Consider a linear random dynamical system $\Phi(t,w)$ as in \eqref{eq:Phi_t_omega}.
Suppose that 
\begin{equation*}
	\max\left\{\log\norm{A(\cdot)}_2,0\right\}\in L^1(\Omega,\mathcal{F},\bP).
\end{equation*}
There exists $\widetilde{\Omega}\in \mathcal{F}$ with $\theta\widetilde{\Omega}\subseteq\widetilde{\Omega}$ and $\bP(\widetilde{\Omega})=1$, such that the followings hold for any $\omega\in\widetilde{\Omega}$:
\begin{itemize}
	\item[(i)] It holds that the limit
	\begin{equation*}
    \Lambda(\omega)=\lim_{t\rightarrow\infty}\left(\Phi(t,\omega)^{\top}\Phi(t,\omega)\right)^{1/2t},
	\end{equation*}
	exists and is a positive semidefinite matrix.
	\item[(ii)] Suppose  $\Lambda(\omega)$ has $q(\omega)$ distinct eigenvalues, which are ordered as $e^{\nu_1(\omega)}>\cdots> e^{\nu_{q(\omega)}(\omega)}\geq 0$. Denote $E_i(\omega)$ the corresponding eigenspace, with dimension $d_i(\omega)>0$, for  $i=1,\dots,q(\omega)$. Then the functions $q(\cdot)$, $\nu_i(\cdot)$, and $d_i(\cdot)$, $i=1,\dots,p(\cdot)$, are all measurable and $\theta$-invariant on $\widetilde{\Omega}$:
    \begin{equation*}
        q(\theta\omega) = q(\omega),\ \nu_i(\theta\omega) = \nu_i(\omega),\ \text{and}\ d_i(\theta\omega) = d_i(\omega).
    \end{equation*}
	\item[(iii)] Set $U_i(\omega)=\bigoplus_{j\geq i}E_j(\omega),\ i=1,\dots,q(\omega)$,  
	and $U_{q(\omega)+1}(\omega)=\{0\}$. Then it holds that
	\begin{equation}\label{conv:Lya-exp}
		\lim_{t\rightarrow\infty}\frac{1}{t}\log\norm{\Phi(t,\omega)u}_2=\nu_i(\omega),\quad\forall~u\in U_i(\omega)\backslash U_{i+1}(\omega),
	\end{equation}  
	for $i=1,\dots,q(\omega)$. The maps $E_i(\cdot)$ and $U_i(\cdot)$ from $\widetilde{\Omega}$ to the Grassmannian manifold are measurable. 
	\item[(iv)] It holds that
	\begin{equation*}
		A(\omega)U_i(\omega)\subseteq U_i(\theta\omega),\quad i=1,\dots,q(\omega).
	\end{equation*}
    \item[(v)] When $(\Omega,\mathcal{F},\bP,\theta)$ is ergodic, i.e., every $B\in\mathcal{F}$ with $\theta^{-1}B=B$ satisfies $\bP(B)=0$ or $\bP(B)=1$, the functions  
    $q(\cdot)$, $\nu_i(\cdot)$, and $d_i(\cdot)$, $i=1,\dots,p(\cdot)$, are constants on $\widetilde{\Omega}$.
\end{itemize}
\end{theorem}

In \Cref{thm: MET}, $\nu_1(\omega)>\cdots>\nu_{q(\omega)}(\omega)$ are known as Lyapunov exponents and it is possible that $\lambda_{q(\omega)}(\omega)=-\infty$. Moreover, the filtration 
\begin{equation*}
    \{0\}=U_{q(\omega)+1}(\omega)\subseteq U_{q(\omega)}(\omega)\subseteq \cdots\subseteq U_1(\omega) = \bR^m
\end{equation*}
is known as the Oseledets filtration and \Cref{thm: MET} is also named as Oseledets theorem. In particular, V.I. Oseledets proves the first multiplicative ergodic theorem in \cite{oseledets} and since then, there has been a rich literature in this field, see e.g., \cite{}{raghunathan, Ruelle-1979, Walters-93}.

\begin{remark}
    Although \Cref{thm: MET} is stated for only systems on $\bR^m$, we remark that the same result also holds for complex space $\bC^m$ as mentioned in \cite[Remark 3.4.10]{Arnold_RDS}.
\end{remark}

We prove the following general proposition, which will be used frequently in the proof of results in \Cref{sec:main_results}.
\begin{proposition}\label{prop:leading Lya on invariant subspaces}
    Adopting the notation in \Cref{thm: MET} in the case of $\bC^m$. Assume that $\bC^m = Q_1 \oplus \cdots \oplus Q_k$, where each $Q_i$ is almost surely a $\Phi(t,\omega)$-invariant $\bC$-linear subspace. That is, with probability one, for any $t \in \bN$ and any $u_i \in Q_i$, we have that $\Phi(t,\omega) u_i \in Q_i$. We define $\nu_1(Q_i,\omega)$ as the largest Lyapunov exponent of the restriction of $\Phi(t,\omega)$ on $Q_i$. Then, for any fixed $\omega \in \widetilde \Omega$, there exists a proper $\bC$-linear subspace of $\bC^{m}$, $U(\omega)$, such that
        \begin{align*}
                \lim_{t \to +\infty} \frac{1}{t} \log \|\Phi(t,\omega) u\|_2 = \max_{1\leq i\leq k} \nu_1(Q_i,\omega), \quad \forall~u \in \bC^{m} \backslash U(\omega).
        \end{align*}
    \begin{proof}
        Without loss of generality, we assume that $\nu_1(Q_1,\omega) = \max_{1\leq i\leq k} \nu_1(Q_i,\omega)$. We then let $U(\omega) \subsetneqq Q_1$ be the corresponding $U_2(\omega)$ in \eqref{conv:Lya-exp}, part (iii) of \Cref{thm: MET}, when we consider the restriction of $\Phi(t,\omega)$ on $Q_1$. $U(\omega)$ is also a $\bC$-linear subspace of $\bC^m$. 
        
        For any $u \in \bC^m$, we can write $u = \sum_{i=1} ^k u_i$ with $u_i \in Q_i$ for each $i$.  We are going to prove that if $u_1 \notin U(\omega)$, then
            \begin{align*}
                \lim_{t \to +\infty} \frac{1}{t} \log \left( \sum_{i=1} ^k \|\Phi(t,\omega) u_i\|_2 \right) = \nu_1(Q_1,\omega).
            \end{align*}
        First, by the definition of $U(\omega)$ and \eqref{conv:Lya-exp}, we see that 
            \begin{align}\label{eq:lower bound inv space}
                \lim_{t \to +\infty} \frac{1}{t} \log \left( \|\Phi(t,\omega) u_1\|_2 \right) = \nu_1(Q_1,\omega).
            \end{align}
        On the other hand, by our assumption, we see that for any $i \in \{1,\dots,k\}$, 
            \begin{align*}
                \lim_{t \to +\infty} \frac{1}{t} \log \left( \|\Phi(t,\omega) u_i\|_2 \right) \leq \nu_1(Q_1,\omega) .
            \end{align*}
        Hence, fix any $\epsilon >0$, there is a $T(\epsilon)$, such that when $t >T(\epsilon)$, we have that 
        \begin{align*}
            \frac{1}{t}\log\| \Phi(t,\omega) u_i\|_2 \leq \nu_1(Q_1,\omega)+ \epsilon,
        \end{align*}
    for any $i \in \{1,\dots,k\}$.
    So, when $t >T(\epsilon)$, 
        \begin{align*}
            \sum_{i=1} ^k \|\Phi(t,\omega) u_i\|_2  \leq k \cdot \text{exp}\left(t (\nu_1(Q_1,\omega) + \epsilon) \right),
        \end{align*}
    and then
        \begin{align}\label{eq:upper bound inv space}
            \limsup_{t \to +\infty} \frac{1}{t} \log \left( \sum_{i=1} ^k \|\Phi(t,\omega) u_i\|_2 \right) \leq \nu_1(Q_1,\omega) + \epsilon.
        \end{align}
    Because $\epsilon>0$ is arbitrary, together with \eqref{eq:lower bound inv space}, we can conclude that
        \begin{align*}
            \lim_{t \to +\infty} \frac{1}{t} \log \left( \sum_{i=1} ^k \|\Phi(t,\omega) u_i\|_2 \right) = \nu_1(Q_1,\omega) .
        \end{align*}

    Finally, we see that there is a constant $C>0$ only depending on the decomposition $\bC^m = Q_1 \oplus \cdots \oplus Q_k$, such that for any $v \in \bC^m$ uniquely written as $v = \sum_{i=1} ^k v_i$ with $v_i \in Q_i$ for each $i$, we have that
        \begin{align*}
            C^{-1} \|v\|_2 \leq \sum_{i=1} ^k \|v_i\|_2 \leq C \|v\|_2.
        \end{align*}
    Then, if $u_1 \notin U(\omega)$, we conclude that
        \begin{align*}
            \begin{split}
                \lim_{t \to +\infty} \frac{ \log \left( \sum_{i=1} ^k \|\Phi(t,\omega) u_i\|_2 \right) - \log C}{t} &\leq \lim_{t \to +\infty} \frac{\log \|\Phi(t,\omega) u\|_2}{t}  
                \\  &\leq \lim_{t \to +\infty} \frac{ \log \left( \sum_{i=1} ^k \|\Phi(t,\omega) u_i\|_2 \right) + \log C}{t},
            \end{split}
        \end{align*}
    and in both sides, the limits are $\nu_1(Q_1,\omega) = \max_{1\leq i\leq k} \nu_1(Q_i,\omega)$.
    \end{proof}
\end{proposition}

Next, because of \Cref{asp:W} and the assumptions in \Cref{thm: MET}, we can  work with $(\Omega,\calF,\bP)$ that is the product probability space of $(\bR^{d\times d},\calB(\bR^{d\times d}),\bP_W)_{l\in\bN}$, where $\calB(\bR^{d\times d})$ is the Borel $\sigma$-algebra of $\bR^{d\times d}$. Specifically, $\Omega = (\bR^{d\times d})^\bN$ and a sample $\omega\in\Omega$ is a sequence of matrices in $\bR^{d\times d}$, namely,
\begin{equation}\label{eq:omege}
    \omega=(W^{(0)},W^{(1)},\dots)\in\Omega.
\end{equation}
We consider the metric dynamical system $(\Omega,\calF,\bP)$ induced by the shift operator $\theta:\Omega\to\Omega$ that maps a sample $\omega$ as in \eqref{eq:omege} to
\begin{equation*}
    \theta\omega = (W^{(1)},W^{(2)},\dots),
\end{equation*}
which is clearly $\bP$-preserving. Additionally, $(\Omega,\calF,\bP,\theta)$ is ergodic by the Kolmogorov's zero-one law. 

With the preparations above, \Cref{prop:R_PW} can be proved straightforwardly.

\begin{proof}[Proof of \Cref{prop:R_PW}]
    (i) \eqref{eq:R_PW} is a direct corollary of \Cref{thm: MET} and the ergodicity of $(\Omega,\calF,\bP,\theta)$. In particular, $\log R(\bP_W)$, where $R(\bP_W)$ is the limit in \eqref{eq:R_PW}, is the leading Lyapunov exponent of the linear random dynamical system induced by $A(\omega) = W$ with $W$ being the first matrix in the sequence $\omega$.

    \noindent (ii) To prove \eqref{eq:R_beta_PW}, one needs to verify the condition \eqref{eq:logW_L1} for $\mathrm{Id} + \beta W^{(t)}$ provided \Cref{asp:W}. We use $\mathds{1}(\cdot)$ as the indicator function. Then, we see that
    \begin{align*}
        & \bE\left[\max\left\{\log \|\mathrm{Id}+\beta W^{(t)}\|_2,0\right\}\right] \leq \bE\left[\log\left(1+|\beta|\|W^{(t)}\|_2\right)\right] \\
        & \qquad \leq \bE\left[\log(1+|\beta|)\mathds{1}\left(\|W^{(t)}\|_2\leq 1\right)\right] + \bE\left[\log\left((1+|\beta|)\|W^{(t)}\|_2\right)\mathds{1}\left(\|W^{(t)}\|_2\geq 1\right)\right] \\
        &\qquad \leq 2\log(1+|\beta|) + \bE\left[\log\|W^{(t)}\|_2\cdot \mathds{1}\left(\|W^{(t)}\|_2\geq 1\right)\right] \\
        &\qquad = 2\log(1+|\beta|) + \bE\left[\max\left\{\log\|W^{(t)}\|_2, 0\right\}\right] < \infty.
    \end{align*}
    Therefore, one can apply \Cref{thm: MET} to the linear random dynamical system induced by $A(\omega) = I+\beta W$, with $W$ being the first matrix in the sequence $\omega$, and \eqref{eq:R_beta_PW} follows directly from \Cref{thm: MET} and the ergodicity of $(\Omega,\calF,\bP,\theta)$. In particular, the largest Lyapunov exponent of this linear random dynamical system is $\log R(\beta,\bP_W)$, where $R(\beta,\bP_W)$ is the limit in \eqref{eq:R_beta_PW}.
\end{proof}

\subsection{Preliminary on tensor product space}
This subsection sets up more preparations for presenting the proofs of \Cref{thm:nonres_GNN} and \Cref{thm:res_GNN}.

    For a $y \in \bC^d$ and a $\varphi \in \bC^n$, we use the notation 
        \begin{align*}
            y \otimes \varphi \coloneqq y \varphi ^\top
        \end{align*}
    to denote the $d \times n$ dimensional matrix $y \varphi ^\top$, which can also be identified with $\bC^{d \times n}$. On the other hand, any element in $\bC^{d \times n}$ can be written as a finite linear combination of some $y_i \otimes \varphi_i$'s for $y_i \in \bC^d$ and  $\varphi_i \in \bC^n$. That is, for any $x \in \bC^{d \times n}$, one can assume that $x$ has a form of 
        \begin{align*}
            x = \sum_{i} y_i \otimes \varphi_i.
        \end{align*}
    We remark that, in general, the above linear combinations of $y_i \otimes \varphi_i$ may not be unique. To get a unique decomposition, we need to fix a basis of either $\bC^d$ or $\bC^n$.

    Suppose that $P$ has $p$ Jordan blocks with eigenvlues $\lambda_1,\dots,\lambda_p$ and sizes $n_1,\dots,n_p$.
    By Remark~\ref{rmk:spectrum_P}, without loss of generalization we can assume that
    \begin{equation*}
        1=\lambda_1 > |\lambda_2| \geq \cdots\geq |\lambda_p|,\quad\text{and}\quad n_1=1.
    \end{equation*}
    Denote $V_1,\dots,V_p\subseteq\bC^n$ as the corresponding generalized left eigenspaces for $\lambda_i$ of $P$. For each $V_i$, $i \in \{1,\dots,p\}$, there exists a basis of $V_i$, $\varphi _{i,1},\dots,\varphi_{i,n_i}$, satisfying that
    \begin{equation}\label{eq:Jordan}
         P \varphi_{i,j} = \lambda_i \varphi_{i,j} + \varphi_{i,j+1},\ j=1,\dots,n_i-1,\quad\text{and}\quad P \varphi_{i,n_i} = \lambda_i \varphi_{i,n_i}.
    \end{equation}
    In particular, according to \Cref{rmk:spectrum_P}, we see that $V_1=\text{span}_\bC\{\varphi_{1,1}\}$, with $\varphi_{1,1} = \mathbf{1}_n \coloneqq (1,\dots,1)^\top \in\bC^n$.
    After fixing the basis $\{\varphi_{i,j}\}$ of $\bC^n$, we have the decomposition
    \begin{equation}\label{eq:decomp_Cdn}
        \bC^{d\times n} = \bC^{d} \otimes \bC^{n} =  (\bC^d\otimes V_1) \oplus  \cdots \oplus (\bC^d\otimes V_p).
    \end{equation}
    That is, for any $x \in \bC^{d \times n}$, there is a unique series of $\{y_{i,j}\}$, such that 
        \begin{align}\label{eq:unique decomp x}
            x = \sum_{i=1} ^p \sum_{j=1} ^{n_i} y_{i,j} \otimes \varphi_{i,j}.
        \end{align}
    For any $x \in \bC^{d \times n}$ with $x$ uniquely written in the form of \eqref{eq:unique decomp x}, we can define a norm $\|x\|_D$ by
        \begin{align}\label{eq:decomp norm}
            \|x\|_D  \coloneqq \sum_{i=1} ^p \sum_{j=1} ^{n_i} \|y_{i,j}\|_2,
        \end{align}
    where the norm $\|y_{i,j}\|_2$ is the usual $L^2$-norm for a vector in $\bC^d$. Recall that the Frobenius norm of $x$ is defined by, if we write $x = (x_1 , \dots ,x_n)$ for $x_i$'s in $\bC^d$, then
        \begin{align*}
            \|x\|_F \coloneqq \left( \sum_{i=1} ^n \|x_i\|_2 ^2 \right) ^{\frac{1}{2}},
        \end{align*}
    where the norm $\|x_i\|_2$ is also the usual $L^2$-norm for a vector in $\bC^d$. Because the correspondence between $x$ and the series $\{y_{i,j}\}$ forms a linear isomorphism on $\bC^{d \times n}$, we know that the norm $\|x \|_D$ is equivalent to the norm $\|x \|_F$. More precisely, we have the following proposition.
\begin{proposition}\label{prop:equivalent norm}
    There is a constant $C_P>0$ depending on the matrix $P$, such that for any $x \in \bC^{d \times n}$, we have that
        \begin{align*}
            \frac{1}{C_{P}} \|x\|_F \leq \|x\|_D \leq C_P \|x\|_F .
        \end{align*}
\end{proposition}
We omit the proof of this proposition. But for our purpose to estimate the vertex similarity measure as defined in \Cref{def:vertex sim}, we need the following definition and proposition.
\begin{definition}\label{def:lower order part of x}
    For any $x \in \bC^{d \times n}$ uniquely written in the form of \eqref{eq:unique decomp x}, we define
        \begin{align*}
            \widetilde x \coloneqq x- y_{1,1} \otimes \varphi_{1,1} = \sum_{i=2} ^p \sum_{j=1} ^{n_i} y_{i,j} \otimes \varphi_{i,j}.
        \end{align*}
\end{definition}
In \Cref{def:vertex sim}, we defined the quantity $\sum_{i=1}^n \|x_i - \Bar{x}\|_2 ^2$ with $\Bar{x} = \frac{1}{n}\sum_{i=1}^n x_i = \frac{1}{n} x \varphi_{1,1}\in \bC^d$. This quantity equals to $\|x - \Bar{x}\otimes\varphi_{1,1}\|_F ^2$. We remark that, in general, $y_{1,1}$ for $x$ may not necessarily equal to $\Bar{x}$, and hence $\|x - \Bar{x}\otimes\varphi_{1,1}\|_F$ may not be the same as $\|\widetilde x\|_F$. On the other hand, these two quantities are equivalent in the following sense:
\begin{proposition}\label{prop:equivalent lower order norm}
    There is a constant $C_P>0$ depending on the matrix $P$, such that for any $x \in \bC^{d \times n}$,
        \begin{align}\label{eq:equivalent lower order norm}
            \frac{1}{C_P} \|\widetilde x\|_F \leq \|x- \Bar{x} \otimes \varphi_{1,1}\|_F \leq C_P \|\widetilde x\|_F.
        \end{align}
\end{proposition}
\begin{proof}
    We define a linear map $T$ on $\bC^{d \times n}$ by $T(x) \coloneqq x- \Bar{x} \otimes \varphi_{1,1}$. Then we see that the kernel of $T$ is $\text{Ker}(T)=\bC^d \otimes V_1$. Hence, $T$ induces a linear isomorphism from $\bC^{d \times n} / \text{Ker}(T)$ to $\text{Im}(T)$. Because we have the unique decomposition \eqref{eq:decomp_Cdn} and \eqref{eq:unique decomp x}, the norm $\|\widetilde x\|_F$ can induce a norm on the quotient space $\bC^{d \times n} / \text{Ker}(T)$. Hence, $\|\widetilde x\|_F$ is equivalent to $\|T(x)\|_F$ up to a constant depending on $T$. This finishes the proof of \eqref{eq:equivalent lower order norm}.
\end{proof}

\subsection{Proof of \Cref{thm:nonres_GNN}}\label{section:proof of main GNN} 

This subsection proves \Cref{thm:nonres_GNN}.

\subsubsection{Random GNN dynamical system on tensor product space}
    We consider a linear random dynamical system $\Phi(t,\omega)$ on $\bC^{d\times n}$ over $(\Omega,\calF,\bP,\theta)$, induced by
    \begin{align}\label{e:action of A on tensor}
        \begin{split}
            A(\omega):\qquad \bC^{d\times n} = \bC^d \otimes \bC^n \qquad & \to\qquad\ \bC^{d\times n} = \bC^d \otimes \bC^n ,\\
            x = \sum_{i=1} ^p \sum_{j=1} ^{n_i} y_{i,j} \otimes \varphi_{i,j} &\mapsto  \sum_{i=1} ^p \sum_{j=1} ^{n_i} (W y_{i,j}) \otimes (P \varphi_{i,j}), 
        \end{split}
    \end{align}
    where $W$ is the first matrix in the sequence $\omega$. This linear map is well-defined, as if we regard $x$ as a $d \times n$ dimensional matrix, then $A(\omega) x = W x P^\top$. Then, we see that the dynamics \eqref{eq:GNN} with $\sigma$ being the identity map is precisely driven by $\Phi(t,w)$. More precisely, if we write that, by abusing our notations,
        \begin{align}\label{eq:x(0) decomp}
            x_{\mathrm{nrs}}{(0)} = \sum_{i=1} ^p \sum_{j=1} ^{n_i} y_{i,j} (0) \otimes \varphi_{i,j},
        \end{align}
    then for $t \in \bN$,
    \begin{align}\label{eq:dyn in decomp GNN 1}
        \begin{split}
            x_{\mathrm{nrs}}{(t)} &= \Phi(t,\omega) x{(0)} = A(\theta^{t-1}\omega)\cdots A(\theta\omega)A(\omega) x{(0)} 
            \\  &= \sum_{i=1} ^p \sum_{j=1} ^{n_i} \left(W^{(t-1)}\cdots W^{(1)}W^{(0)} y_{i,j}(0) \right) \otimes P^t \varphi_{i,j}.
        \end{split}
    \end{align}
    For notation purpose, we let $W^{(-1)}$ be the identity matrix.

    By the definition of the action of $A(\omega)$ in \eqref{e:action of A on tensor} and \eqref{eq:Jordan}, one can easily see that $\bC^d\otimes V_i$ is an invariant subspace of $A(\omega)$ and $\Phi(t,\omega)$ for any $i\in \{1,\dots,p\}$ and any $t\in\bN,\omega\in\Omega$. Furthermore, by induction, one can show that for each $i \in \{1,\dots,p\},j \in \{1,\dots,n_i\}$, 
        \begin{equation}\label{eq:Jordan t times}
            P^t \varphi_{i,j}= \sum_{s=0}^{n_i - j} \lambda_i^{t-s} \binom{t}{s} \varphi_{i,j+s}.
        \end{equation}
    Hence, we see that \eqref{eq:dyn in decomp GNN 1} becomes
        \begin{align}\label{eq:dyn in decomp GNN 2}
            x_{\mathrm{nrs}}{(t)} = \sum_{i=1} ^p \sum_{j=1} ^{n_i} \left( \sum_{s=1} ^j  \lambda_i^{t-(j-s)}\binom{t}{j-s} W^{(t-1)}\cdots W^{(1)} W^{(0)} y_{i,s}(0) \right) \otimes \varphi_{i,j} .
        \end{align}
    To simplify our notations, we define
        \begin{align}\label{eq:dyn in decomp GNN 3}
            y_{i,j}(t) \coloneqq \sum_{s=1} ^j  \lambda_i^{t-(j-s)}\binom{t}{j-s} W^{(t-1)}\cdots W^{(1)} W^{(0)} y_{i,s}(0) .
        \end{align}
    For notation purpose, we let $\binom{0}{i} =0$ if $i >0$ and let $\binom{j}{0} = 1$ if $j \geq 0$.
    
The forms of \eqref{eq:dyn in decomp GNN 2} and \eqref{eq:dyn in decomp GNN 3} suggests us to consider the following linear random dynamical system on $\bC^d$: for any $y \in \bC^d$, we define $\Phi_W(t,\omega)y$ by
    \begin{equation*}
        \Phi_W(t,\omega) y \coloneqq W^{(t-1)}\cdots W^{(1)} W^{(0)} y.
    \end{equation*}
    \Cref{thm: MET} applies for $\Phi_W(t,\omega)$ since \eqref{eq:logW_L1} is assumed. Denote $\nu_1 ^W > \cdots > \nu_{q^W} ^W \geq -\infty$ as the ordered Lyapunov exponents of $\Phi_W(t,\omega)$ on $\bC^d$ as in \Cref{thm: MET}, which are almost surely constants on $\Omega$ according to the part (v) of \Cref{thm: MET} since $(\Omega,\calF,\bP,\theta)$ is ergodic. We then have that for any $\omega \in \widetilde{\Omega}$,
    \begin{align}\label{eq:Phi_W Lya exponents}
        \lim_{t\to+\infty}\frac{1}{t}\log\| \Phi_W(t,\omega) y\|_2 = \nu_i ^W, \quad \forall~y \in U_i ^W (\omega)\backslash U_{i+1} ^W (\omega), \quad i \in \{1,\dots,q^W\}
    \end{align}
    In particular, we see that $U_1 ^W = \bC^d$.
    
Next, similar to \Cref{def:lower order part of x}, we can define
    \begin{align}\label{eq:dyn in decomp GNN low order}
        \widetilde x_{\mathrm{nrs}}(t) \coloneqq \widetilde {x_{\mathrm{nrs}}(t)} = \sum_{i =2} ^p \sum_{j=1} ^{n_i} y_{i,j}(t) \otimes \varphi_{i,j}.
    \end{align}
We remark that under this definition, we have that $\widetilde x_{\mathrm{nrs}}(t) = \Phi(t,\omega) (\widetilde{x_{\mathrm{nrs}}(0)})$. We have the following lemma for $\| \widetilde x_{\mathrm{nrs}}(t) \|_F$:
\begin{lemma}\label{lemma:lower part Lya GNN}
    Fix an $\omega \in \widetilde{\Omega}$. For any $x(0)$ in the form of \eqref{eq:x(0) decomp}, if $y_{2,1}(0) \in U_1 ^W (\omega)\backslash U_{2} ^W (\omega) = \bC^d \backslash U_{2} ^W (\omega)$, then for the corresponding $x(t)$ in \eqref{eq:dyn in decomp GNN 2}, we have that
        \begin{align}\label{eq:lower part Lya GNN}
            \lim_{t \to +\infty} \frac{1}{t} \log \|\widetilde x_{\mathrm{nrs}}(t)\|_F = \log |\lambda_2| + \nu_1 ^W .
        \end{align}
\end{lemma}

\begin{proof}
    We are going to prove that $\lim_{t \to +\infty} \frac{1}{t} \log \| \widetilde x_{\mathrm{nrs}}(t)\|_D = \log |\lambda_2| + \nu_1 ^W$ first. According to \eqref{eq:decomp norm}, \eqref{eq:dyn in decomp GNN 3}, and \eqref{eq:dyn in decomp GNN low order}, we see that 
        \begin{align*}
                \|\widetilde x_{\mathrm{nrs}}(t)\|_D &= \sum_{i=2} ^p \sum_{j=1} ^{n_i} \left\| y_{i,j}(t) \right\|_2
                = |\lambda_2| ^t \|\Phi_W(t,\omega) y_{2,1}(0)\|_2 +  \sum_{j=2} ^{n_2} \|y_{2,j}(t)\|_2+\sum_{i=3} ^p \sum_{j=1} ^{n_i} \|y_{i,j}(t)\|_2.
        \end{align*}
    Because we assumed that $y_{2,1}(0) \in  \bC^d \backslash U_{2} ^W (\omega)$ and by \eqref{eq:Phi_W Lya exponents}, we first see that 
        \begin{align}\label{eq:lower bound GNN}
            \lim_{t \to +\infty} \frac{1}{t} \log \left(  |\lambda_2| ^t \|\Phi_W(t,\omega) y_{2,1}(0)\|_2 \right) = \log |\lambda_2| + \nu_1 ^W.
        \end{align}

    On the other hand, \eqref{eq:Phi_W Lya exponents} also implies that $\lim_{t\to+\infty}\frac{1}{t}\log\| \Phi_W(t,\omega) y\|_2 \leq \nu_1 ^W$ for any $y \in \bC^d$. Fix any $\epsilon >0$, there is a $T(\epsilon)$, such that when $t >T(\epsilon)$, we have that 
        \begin{align*}
            \frac{1}{t}\log\| \Phi_W(t,\omega) y_{i,j}(0)\|_2 \leq \nu_1 ^W + \epsilon,
        \end{align*}
    for any $i \in \{2,\dots,p\}$ and $j \in \{1,\dots,n_i\}$.
    By \eqref{eq:dyn in decomp GNN 3}, we see that for any $i \in \{2,\dots,p\}$ and $j \in \{1,\dots,n_i\}$,
        \begin{align*}
            \|y_{i,j}(t)\|_2 \leq \sum_{s=1} ^j  |\lambda_i|^{t-(j-s)}\binom{t}{j-s} \left\| \Phi_W(t,\omega) y_{i,s}(0) \right\|_2 \leq n|\lambda_2|^t t^n \text{exp}\left(t (\nu_1 ^W + \epsilon) \right).
        \end{align*}
    So, when $t >T(\epsilon)$, 
        \begin{align*}
            \|\widetilde x_{\mathrm{nrs}}(t)\|_D \leq n^2|\lambda_2|^t t^n \text{exp}\left(t (\nu_1 ^W + \epsilon) \right),
        \end{align*}
    and then
        \begin{align}\label{eq:upper bound GNN}
            \limsup_{t \to +\infty} \frac{1}{t} \log \|\widetilde x_{\mathrm{nrs}}(t)\|_D \leq \log |\lambda_2| + \nu_1 ^W + \epsilon.
        \end{align}
    Because $\epsilon>0$ is arbitrary, together with \eqref{eq:lower bound GNN}, we can conclude that
        \begin{align*}
            \lim_{t \to +\infty} \frac{1}{t} \log \| \widetilde x_{\mathrm{nrs}}(t)\|_D = \log |\lambda_2| + \nu_1 ^W.
        \end{align*}
    Then, by \Cref{prop:equivalent norm}, we see that
        \begin{align*}
            \lim_{t \to +\infty} \frac{\log \| \widetilde x_{\mathrm{nrs}}(t)\|_D - \log C_P}{t} \leq \lim_{t \to +\infty} \frac{1}{t} \log \| \widetilde x_{\mathrm{nrs}}(t)\|_F \leq \lim_{t \to +\infty} \frac{\log \| \widetilde x_{\mathrm{nrs}}(t)\|_D + \log C_P}{t},
        \end{align*}
    and in both sides, the limits are $\log |\lambda_2| + \nu_1 ^W$.
\end{proof}
Similar to the proof for \Cref{lemma:lower part Lya GNN}, one can directly get the following lemma for $\|x_{\mathrm{nrs}}(t)\|_F$ and $\mu(x_{\mathrm{nrs}}(t))$.
\begin{lemma}\label{lemma:leading part Lya GNN}
    Fix an $\omega \in \widetilde{\Omega}$. For any $x_{\mathrm{nrs}}(0)$ in the form of \eqref{eq:x(0) decomp}, if $y_{1,1}(0) \in U_1 ^W (\omega)\backslash U_{2} ^W (\omega) = \bC^d \backslash U_{2} ^W (\omega)$, then for the corresponding $x_{\mathrm{nrs}}(t)$ in \eqref{eq:dyn in decomp GNN 2}, we have that
        \begin{align}\label{eq:leading part Lya GNN}
            \lim_{t \to +\infty} \frac{1}{t} \log \| x_{\mathrm{nrs}}(t)\|_F = \log |\lambda_1| + \nu_1 ^W =  \nu_1 ^W.
        \end{align}
    If we further assume that $y_{2,1}(0) \in  \bC^d \backslash U_{2} ^W (\omega)$ and $\nu_1^W>-\infty$, then
        \begin{align}\label{eq:vertex similarity GNN}
            \lim_{t \to +\infty} \frac{1}{2t} \log \mu(x_{\mathrm{nrs}}(t)) = \log |\lambda_2|.
        \end{align}
\end{lemma}
\begin{proof}
    The proof for \eqref{eq:leading part Lya GNN} is the same as the proof for \Cref{lemma:lower part Lya GNN}. For \eqref{eq:vertex similarity GNN}, we notice that 
        \begin{align*}
            \mu(x_{\mathrm{nrs}}(t)) = \frac{\|x_{\mathrm{nrs}}(t)- \overline{x_{\mathrm{nrs}}(t)} \otimes \varphi_{1,1}\|_F ^2}{\|x_{\mathrm{nrs}}(t)\|_F ^2},
        \end{align*}
    as discussed after \Cref{def:lower order part of x}. By \Cref{prop:equivalent lower order norm}, we see that
        \begin{align*}
            \lim_{t \to +\infty} \frac{\log \| \widetilde x_{\mathrm{nrs}}(t)\|_F - \log C_P}{t} & \leq \lim_{t \to +\infty} \frac{1}{t} \log \| x_{\mathrm{nrs}}(t)- \overline{x_{\mathrm{nrs}}(t)} \otimes \varphi_{1,1}\|_F \\
            & \leq \lim_{t \to +\infty} \frac{\log \| \widetilde x_{\mathrm{nrs}}(t)\|_F + \log C_P}{t},
        \end{align*}
    and in both sides, the limits are $\log |\lambda_2| + \nu_1 ^W$ by \Cref{lemma:lower part Lya GNN}. Hence, we can conclude \eqref{eq:vertex similarity GNN}.
\end{proof}

\subsubsection{Proof of \Cref{thm:nonres_GNN}}
We notice that \Cref{lemma:leading part Lya GNN} with $\nu_1^W = \log R(\bP_W)$ already implies \Cref{thm:nonres_GNN} for almost every $x_{\mathrm{nrs}}(0) \in \bC^{d \times n}$, because if we use the form \eqref{eq:unique decomp x}, for any fixed $\omega \in \widetilde \Omega$, the set
    \begin{align}
        \left\{x \in \bC^{d \times n} \ | \ y_{1,1} \in U_{2} ^W (\omega) \text{ or } y_{2,1} \in U_{2} ^W (\omega)  \right\}
    \end{align}
is the union of two proper $\bC$-linear subspaces of $\bC^{d \times n}$, and hence has zero Lebesgue measure in $\bC^{d \times n}$.
We can also get the conclusion for almost every $x(0) \in \bR^{d \times n}$ once we have the following fact:
\begin{lemma}\label{lem:proper_subspace_Rm}
    Assume that $U \subsetneqq \bC^{d \times n}$ is a proper $\bC$-linear subspace of $\bC^{d \times n}$, then $U \cap \bR^{d \times n}$ is a proper $\bR$-linear subspace of $\bR^{d \times n}$.
\end{lemma}

\begin{proof}
    Let $\{e_i\}_{i=1} ^{d \times n}$ be the standard coordinates vectors of $\bR^{d \times n}$, that is, $1$ on the $i$-th entry and $0$ on the other entries. If $\bR^{d \times n} = U \cap \bR^{d \times n}$, then we have that $\{e_i\}_{i=1} ^{d \times n} \subseteq U$. As a consequence,
        \begin{align}
            \bC^{d \times n} = \text{span}_\bC \{ e_i\}_{i=1} ^{d \times n} \subseteq U,
        \end{align}
    which is a contradiction.
\end{proof}

\subsection{Proof of \Cref{thm:res_GNN}}\label{section:proof of main res GNN}
This subsection proves \Cref{thm:res_GNN}.

\subsubsection{Random residual GNN dynamical system on tensor product space}
 The linear random dynamical system $\Phi(t,\omega)$ associated to \eqref{eq:res_GNN} is defined on $\bC^{d\times n}$ over $(\Omega,\calF,\bP,\theta)$, induced by
    \begin{align}\label{e:action of A on tensor res GNN}
        \begin{split}
            A(\omega):\qquad \bC^{d\times n} = \bC^d \otimes \bC^n \qquad & \to\qquad\ \bC^{d\times n} = \bC^d \otimes \bC^n ,\\
            x =  \sum_{i=1} ^p \sum_{j=1} ^{n_i} y_{i,j} \otimes \varphi_{i,j} &\mapsto  x + \alpha \sum_{i=1} ^p \sum_{j=1} ^{n_i} (W y_{i,j}) \otimes (P \varphi_{i,j}), 
        \end{split}
    \end{align}
    where we still used the unique decomposition \eqref{eq:unique decomp x} for any $x \in \bC^{d\times n}$, and $W$ is the first matrix in the sequence $\omega$. This linear map is well-defined, as if we regard $x$ as a $d \times n$ dimensional matrix, then $A(\omega) x = x + \alpha W x P^\top$. Then, we see that the dynamics \eqref{eq:res_GNN} with $\sigma$ being the identity map is precisely driven by $\Phi(t,w)$. For any $x_{\mathrm{rs}}(0) \in \bC^{d \times n}$, $ x_{\mathrm{rs}}{(t)} = \Phi(t,\omega) x_{\mathrm{rs}}{(0)} = A(\theta^{t-1}\omega)\cdots A(\theta\omega)A(\omega) x_{\mathrm{rs}}{(0)}$ in the dynamics \eqref{eq:res_GNN}. For notation purpose, we write
        \begin{align}\label{eq:dyn in decomp res GNN}
             x_{\mathrm{rs}}(t) = \sum_{i=1} ^p \sum_{j=1} ^{n_i} y_{i,j}(t) \otimes \varphi_{i,j}, \quad x_{\mathrm{rs}}(0) = \sum_{i=1} ^p \sum_{j=1} ^{n_i} y_{i,j}(0) \otimes \varphi_{i,j}.
        \end{align}

     By the definition of the action of $A(\omega)$ in \eqref{e:action of A on tensor res GNN} and \eqref{eq:Jordan}, one can easily see that $\bC^d\otimes V_i$ is an invariant subspace of $A(\omega)$ and $\Phi(t,\omega)$ for any $i\in \{1,\dots,p\}$ and any $t\in\bN,\omega\in\Omega$. Hence, we can define for each $i\in \{1,\dots,p\}$, we let $\nu_1(\bC^d\otimes V_i)$ be the largest Lyapunov exponent of $\Phi(t,\omega)$ restricted on $\bC^d \otimes V_i$, which is a constant for $\omega \in \widetilde \Omega$ by ergodicity and part (v) of \Cref{thm: MET}.
     
     Next, similar to \eqref{eq:dyn in decomp GNN low order}, we still define
    \begin{align}\label{eq:dyn in decomp res GNN low order}
        \widetilde x_{\mathrm{rs}}(t) \coloneqq \widetilde {x_{\mathrm{rs}}(t)} = \sum_{i =2} ^p \sum_{j=1} ^{n_i} y_{i,j}(t) \otimes \varphi_{i,j}.
    \end{align}
We remark that under this definition, we have that $\widetilde x_{\mathrm{rs}}(t) = \Phi(t,\omega) (\widetilde{x_{\mathrm{rs}}(0)})$.

     \begin{lemma}\label{lemma:leading part Lya res GNN}
         Fix an $\omega \in \widetilde{\Omega}$. There is a set $U(\omega)$ consisting of the union of two proper $\bC$-linear subspace of $\bC^{d \times n}$, such that for any $x_{\mathrm{rs}}(0) \in \bC^{d \times n} \backslash U(\omega)$, we have that
            \begin{align}\label{eq:leading part Lya res GNN}
                \lim_{t \to +\infty} \frac{1}{t} \log \|x_{\mathrm{rs}}(t)\|_F = \max_{1\leq i\leq p} \nu_1(\bC^d\otimes V_i), \quad \lim_{t \to +\infty} \frac{1}{t} \log \|\widetilde x_{\mathrm{rs}}(t)\|_F = \max_{2\leq i\leq p} \nu_1(\bC^d\otimes V_i) .
            \end{align}
        As a corollary, if $\max_{1\leq i\leq p} \nu_1(\bC^d\otimes V_i)>-\infty$,
            \begin{align}\label{eq:vertex similarity res GNN}
                \lim_{t \to +\infty} \frac{1}{2t} \log \mu(x_{\mathrm{rs}}(t)) = \min \left\{\max_{2\leq i\leq p} \nu_1(\bC^d\otimes V_i) - \nu_1(\bC^d\otimes V_1) , 0 \right\}.
            \end{align}
     \end{lemma}
     \begin{proof}
         \eqref{eq:leading part Lya res GNN} is a direct application of \Cref{prop:leading Lya on invariant subspaces}, and the proof of \eqref{eq:vertex similarity res GNN} is the same as the proof of  \eqref{eq:vertex similarity GNN} in \Cref{lemma:leading part Lya GNN}.
     \end{proof}

     \subsubsection{Proof of \Cref{thm:res_GNN}}
     Now, we are going to use \Cref{lemma:leading part Lya res GNN} and \eqref{e:action of A on tensor res GNN} to estimate $\nu_1(\bC^d\otimes V_i)$'s for each $i \in \{1,\dots,p\}$. Again, assume that $x(t)$ has the form \eqref{eq:dyn in decomp res GNN}.
     
     First, \eqref{eq:Jordan} gives that $P \varphi_{i,n_i} = \lambda_i \varphi_{i,n_i}$. Hence, $A(\omega) (y_{i,n_i} \otimes \varphi_{i,n_i}) = ((\mathrm{Id}+\alpha \lambda_i W)y_{i,n_i}) \otimes \varphi_{i,n_i}$. Hence, apply \Cref{thm: MET} and \Cref{prop:R_PW} (ii) to the linear random dynamic system $A_i(\omega)y = (\mathrm{Id}+\alpha \lambda_i W)y$, we see that there is a proper $\bC$-linear subspace $U_i(\omega)$ of $\bC^{d}$, such that if $y_{i,n_i}(0) \in \bC^{d} \backslash U_i(\omega)$, then
     \begin{equation}\label{eq:Lya_exponent_lb}
        \begin{split}
            &\quad \lim_{t\to+\infty}\frac{1}{t} \log\|y_{i,n_i}(t)  \otimes \varphi_{i,n_i}\|_F = \lim_{t\to+\infty}\frac{1}{t} \log\|y_{i,n_i}(t)  \otimes \varphi_{i,n_i}\|_D 
            \\  &= \lim_{t\to+\infty}\frac{1}{t} \log\left\|\left(\mathrm{Id}+\alpha\lambda_i W^{(t-1)}\right)\cdots \left(\mathrm{Id}+\alpha\lambda_i W^{(1)}\right)\left(\mathrm{Id}+\alpha\lambda_i W^{(0)}\right)y_{i,n_i}(0)\right\|_2 \\  &= \log R(\alpha\lambda_i,\bP_W),
        \end{split}
    \end{equation}
    where the first equality is by \Cref{prop:equivalent norm} and the last equality is by \Cref{prop:R_PW} (ii).
    Because for any $i\in \{1,\dots,p\}$ and any $t\in\bN,\omega\in\Omega$, $\bC^d\otimes V_i$ is an invariant subspace of $A(\omega)$ and $\Phi(t,\omega)$, we see that $\nu_1(\bC^d\otimes V_i) \geq \log R(\alpha\lambda_i,\bP_W)$, and if $n_i =1$, we have that $\nu_1(\bC^d\otimes V_i) =\log R(\alpha\lambda_i,\bP_W)$.

    Hence, by \eqref{eq:vertex similarity res GNN}, because $\lambda_1 = 1$, we see that 
        \begin{align*}
            \begin{split}
                \lim_{t \to +\infty} \frac{1}{2t} \log \mu(x_{\mathrm{rs}}(t)) &= \min \left\{\max_{2\leq i\leq p} \nu_1(\bC^d\otimes V_i) - \nu_1(\bC^d\otimes V_1) , 0 \right\}
                \\  &= \min \left\{\max_{2\leq i\leq p} \nu_1(\bC^d\otimes V_i) - \log R(\alpha,\bP_W) , 0 \right\}
                \\  &\geq \min \left\{\max_{2\leq i\leq p} \log R(\alpha\lambda_i,\bP_W) - \log R(\alpha,\bP_W) , 0 \right\},
            \end{split}
        \end{align*}
    where the last inequality is an equality if $n_i =1$ for all $i \in \{1,\dots,p\}$, that is, if $P$ is diagonalizable in $\bC$.
    This finishes the proof of \Cref{thm:res_GNN}.

\subsection{Proofs for \Cref{sec:application}}
\label{sec:pf_application}

This subsection collects proofs for results in \Cref{sec:application}.

\subsubsection{Proofs of \Cref{thm:deterministic_W} and \Cref{thm2:deterministic_W}}

\begin{proof}[Proof of \Cref{thm:deterministic_W}]
    Again, because of \eqref{e:action of A on tensor res GNN}, we see that $\bC^d\otimes V_i$'s are invariant subspaces of the action $A: x \mapsto x + \alpha W x P^{\top}$ for any $i\in \{1,\dots,p\}$. According to \Cref{lemma:leading part Lya res GNN}, we only need to estimate $\nu_1(\bC^d\otimes V_i)$ for each $i \in \{1,\dots,p\}$. We fix an $i \in \{1,\dots,p\}$.

    Hence, as in \eqref{eq:dyn in decomp res GNN}, we assume that 
        \begin{align*}
            x_{\mathrm{rs}}(t) =  \sum_{j=1} ^{n_i} y_{i,j}(t) \otimes \varphi_{i,j}, \quad x_{\mathrm{rs}}(0) = \sum_{j=1} ^{n_i} y_{i,j}(0) \otimes \varphi_{i,j}.
        \end{align*}
    We can explicitly compute $y_{i,j}(t)$'s in the following: by induction, one can first show that
        \begin{align*}
            x_{\mathrm{rs}}(t) = \sum_{k=0} ^t \binom{t}{k} \alpha^k W^k x_{\mathrm{rs}}(0) (P^k)^{\top}= \sum_{k=0} ^t \binom{t}{k} \alpha^k \sum_{j=1} ^{n_i} (W^k y_{i,j}(0)) \otimes (P^k \varphi_{i,j}).
        \end{align*}
    By \eqref{eq:Jordan t times}, we see that the right hand side of the above equation equals to
        \begin{align*}
                &\sum_{k=0} ^t \binom{t}{k} \alpha^k \sum_{j=1}^{n_i} \left [(W^k y_{i,j}(0)) \otimes \sum_{s=j} ^{n_i} \lambda_i ^{k-(s-j)} \binom{k}{s-j} \varphi_{i,s} \right]
                \\ &\qquad = \sum_{k=0} ^t \binom{t}{k} \alpha^k \sum_{s=1} ^{n_i} \left[\sum_{j=1} ^s \lambda_i ^{k-(s-j)} \binom{k}{s-j}  W^k y_{i,j}(0) \right]\otimes \varphi_{i,s}
                \\ &\qquad = \sum_{s=1} ^{n_i} \left[ \sum_{j=1} ^s\sum_{k=0} ^t \binom{t}{k} \alpha^k \lambda_i ^{k-(s-j)} \binom{k}{s-j}  W^k y_{i,j}(0) \right]\otimes \varphi_{i,s}
                \\ &\qquad = \sum_{s=1} ^{n_i} \left[ \sum_{j=1} ^s\sum_{k=s-j} ^t \binom{t}{k} \alpha^k \lambda_i ^{k-(s-j)} \binom{k}{s-j}  W^k y_{i,j}(0) \right]\otimes \varphi_{i,s},
        \end{align*}
    where we used the notation that $\binom{a}{b} = 0$ if $b <a$. We notice that
        \begin{align*}
            \binom{t}{k} \binom{k}{s-j} = \frac{t!}{(t-k)! k!} \frac{k!}{(k-(s-j))! (s-j)!} = \binom{t}{s-j} \binom{t-(s-j)}{k-(s-j)}.
        \end{align*}
    So,
        \begin{align*}
            \begin{split}
                y_{i,s}(t) & = \sum_{j=1} ^s\sum_{k=s-j} ^t \binom{t}{k} \alpha^k \lambda_i ^{k-(s-j)} \binom{k}{s-j}  W^k y_{i,j}(0) 
                \\  &= \sum_{j=1} ^s (\alpha
            W)^{s-j} \binom{t}{s-j} \sum_{k=s-j} ^t \binom{t-(s-j)}{k-(s-j)} (\alpha \lambda_i W) ^{k-(s-j)} y_{i,j}(0) 
            \\  &=\sum_{j=1} ^s (\alpha
            W)^{s-j} \binom{t}{s-j} (\mathrm{Id} + \alpha \lambda_i W)^{t-(s-j)} y_{i,j}(0).
            \end{split}
        \end{align*}
    Hence, by \eqref{eq:decomp norm}, we see that
        \begin{align*}
            \begin{split}
                \|x_{\mathrm{rs}}(t)\|_D &\leq \sum_{s=1} ^{n_i} \|y_{i,s}(t)\|_2 \leq n_i  t^{n_i} \max_{1\leq s\leq n_i} \sum_{j=1} ^s \|\alpha W\|_2^{s-j} \| (\mathrm{Id} + \alpha \lambda_i W)^{t-(s-j)} \|_2 \|y_{i,j}(0)\|_2 \\   &\leq C(n_i,\alpha,\lambda,W,\{\|y_{i,j}(0)\|_2\}) \cdot t^{n_i} \| (\mathrm{Id} + \alpha \lambda_i W)^{t-n_i} \|_2,
            \end{split}
        \end{align*}
    where we used $C(n_i,\alpha,\lambda,W,\{\|y_{i,j}(0)\|_2\})>0$ to denote a constant depending on $n_i,\alpha,\lambda, \|W\|_2$, and $\{\|y_{i,j}(0)\|_2\}$'s, which is independent of $t$. Hence,
        \begin{align*}
           \begin{split}
                \nu_1(\bC^d\otimes V_i) &= \limsup_{t \to +\infty} \frac{1}{t} \log \|x(t)\|_D \leq \limsup_{t \to +\infty} \frac{1}{t} \log \| (\mathrm{Id} + \alpha \lambda_i W)^{t-n_i} \|_2
                \\  &= \log\rho(\mathrm{Id}+\alpha\lambda_i W),
           \end{split}
        \end{align*}
    where the last equality is because of Gelfand's formula. Together with the proof for \Cref{thm:res_GNN}, we see that for each $i\in \{1,\dots,p\}$, $\nu_1(\bC^d\otimes V_i) = \rho(\mathrm{Id}+\alpha\lambda_i W)$. This fact together with \Cref{lemma:leading part Lya res GNN} finishes the proof.
\end{proof}

\begin{proof}[Proof of \Cref{thm2:deterministic_W}]
    (i) Take $\lambda_0\in \mathrm{spec}(P)\backslash \{1\}\subseteq(-1,1)$ achieving $\max\limits_{\lambda\in \mathrm{spec}(P)\backslash \{1\}}|\lambda|$. It follows from $1+\lambda_0(2\alpha\mathrm{Re}\mu_0+1)>0$ that
    \begin{align*}
        (1+\alpha\lambda_0\mathrm{Re}(\mu_0))^2 - (\lambda_0+\alpha\lambda_0\mathrm{Re}(\mu_0))^2 = (1-\lambda_0) (1+\lambda_0(2\alpha\mathrm{Re}\mu_0+1)) > 0,
    \end{align*}
    which implies that
    \begin{align*}
        \rho(I+\alpha\lambda_0 W) & \geq |1+\alpha_0\lambda_0\mu_0| = \left((1+\alpha\lambda_0\mathrm{Re}(\mu_0))^2 + (\alpha\lambda_0\mathrm{Im}(\mu_0))^2\right)^{1/2} \\
        & > \left((\lambda_0+\alpha\lambda_0\mathrm{Re}(\mu_0))^2 + (\alpha\lambda_0\mathrm{Im}(\mu_0))^2\right)^{1/2} = \lambda_0 \cdot |1+\alpha \mu_0| = \lambda_0 \cdot \rho(I+\alpha W),
    \end{align*}
    where the inequality is guaranteed by $\alpha\mathrm{Re}(\mu_0) +1\geq 0$. Therefore, if 
        \begin{align*}
            \max_{\lambda\in \mathrm{spec}(P)\backslash\{1\}} \rho(I+\alpha\lambda W) \geq \rho(I+\alpha W),
        \end{align*}
    then,
    \begin{align*}
        \lim_{t\to+\infty} \mu(x_{\mathrm{rs}}{(t)})^{1/2t}& \geq \frac{\max_{\lambda\in \mathrm{spec}(P)\backslash\{1\}} \rho(I+\alpha\lambda W)}{\max_{\lambda\in \mathrm{spec}(P)} \rho(I+\alpha\lambda W)} =1 \\
        & > \max_{\lambda\in \mathrm{spec}(P)\backslash \{1\}}|\lambda| = \lim_{t\to+\infty} \mu(x_{\mathrm{nrs}}{(t)})^{1/2t}.
    \end{align*}
    Otherwise,
    \begin{equation*}
        \begin{split}
            \lim_{t\to+\infty} \mu(x_{\mathrm{rs}}{(t)})^{1/2t}&\geq \frac{\max_{\lambda\in \mathrm{spec}(P)\backslash\{1\}} \rho(I+\alpha\lambda W)}{\max_{\lambda\in \mathrm{spec}(P)} \rho(I+\alpha\lambda W)} \geq \frac{\rho(I+\alpha\lambda_0 W)}{ \rho(I+\alpha W)} 
            \\  &> \lambda_0 = \max_{\lambda\in \mathrm{spec}(P)\backslash \{1\}}|\lambda| = \lim_{t\to+\infty} \mu(x_{\mathrm{nrs}}{(t)})^{1/2t}.
        \end{split}
    \end{equation*}

\noindent  (ii) Take a $\lambda\in \mathrm{spec}(P)\backslash \{1\}\subseteq(-1,1)$ with $2\mathrm{Re}\mu_0+\alpha(\lambda+1)|\mu_0|^2\leq 0$, we see that
    \begin{align*}
        & \rho(I+\alpha\lambda W)^2 - \rho(I+\alpha W)^2 \\
        &\qquad \geq \left((1+\alpha\lambda\mathrm{Re}(\mu_0))^2 + (\alpha\lambda\mathrm{Im}(\mu_0))^2\right) - \left((1+\alpha \mathrm{Re}(\mu_0))^2 + (\alpha \mathrm{Im}(\mu_0))^2\right) \\
        &\qquad = 2\alpha(\lambda-1) \mathrm{Re}\mu_0 + \alpha^2(\lambda^2-1)|\mu_0|^2 = \alpha(\lambda-1)\left(2\mathrm{Re}\mu_0+\alpha(\lambda+1)|\mu_0|^2\right) \geq 0,
    \end{align*}
    which implies that
    \begin{equation*}
        \rho(I+\alpha\lambda W) \geq \rho(I+\alpha W).
    \end{equation*}
    We can thus conclude that
    \begin{equation*}
        \lim_{t\to+\infty} \mu(x_{\mathrm{rs}}{(t)})^{1/2t} \geq \frac{\max_{\lambda\in \mathrm{spec}(P)\backslash\{1\}} \rho(I+\alpha\lambda W)}{\max_{\lambda\in \mathrm{spec}(P)} \rho(I+\alpha\lambda W)} = 1 > \max_{\lambda\in \mathrm{spec}(P)\backslash \{1\}}|\lambda|,
    \end{equation*}
    which completes the proof together with \Cref{thm:nonres_GNN}.
\end{proof}

\subsubsection{Proofs of \Cref{thm:gaussian_W} and \Cref{thm2:gaussian_W}}\label{section: proof of iid Gaussian res GNN}

We first present a proposition that implies \Cref{thm:gaussian_W} immediately.

\begin{proposition}
\label{prop: iid Gaussian W}
    Assume that \Cref{asp:W} holds and in addition that $W^{(t)}$'s entries are drawn i.i.d. from $\calN(0,\tau^2)$. Then, for $\beta \in \bR$, the $R(\beta,\bP_W)$ defined in \Cref{prop:R_PW} (ii) satisfies that
        \begin{align*}
            \log R(\beta,\bP_W) = \frac{1}{2} \bE \log \left((1  + \beta\tau \xi) ^2 + \beta^2\tau^2 \chi_{d-1} ^2 \right),
        \end{align*}
    where $\xi\sim\calN(0,1)$ is a $1$-dimensional standard real Gaussian distribution, and $\chi^2_{d-1}$ is a chi-squared distribution with $d-1$ degrees of freedom that is independent of $\xi$.
\end{proposition}

\begin{proof}[Proof of \Cref{thm:gaussian_W}]
    \Cref{thm:gaussian_W} is a direct corollary of \Cref{thm:res_GNN} and \Cref{prop: iid Gaussian W}.
\end{proof}

We then present the proof of \Cref{prop: iid Gaussian W}, with the main tool being Proposition 2.1 in \cite{Cohen84}. For the convenience of the readers, we adopt the statement of this proposition in the settings of our \Cref{prop:R_PW} (ii).

\begin{theorem}[{\cite[Proposition 2.1]{Cohen84}}]\label{thm: explicit Lya exponent}
    Suppose \Cref{asp:W} holds. If $\beta\in\bR$ and $\| (\mathrm{Id} + \beta W^{(t)})x\|_2$ has a distribution which does not depend on $x \in \mathbb{S}^{d-1} \subseteq \bR^{d}$, where $\mathbb{S}^{d-1}$ is the unit sphere in $\bR^{d}$, then
        \begin{align*}
            \log R(\beta, \bP_W) = \bE \log \|(\mathrm{Id} +\beta W^{(t)})e_1\|_2,
        \end{align*}
    where $e_1 = (1,0,\dots, 0) ^{\top} \in \bR^d$.
\end{theorem}

\begin{proof}[Proof of \Cref{prop: iid Gaussian W}]
    By \Cref{thm: explicit Lya exponent}, we need to verify that $\| \mathrm{Id} + \beta W^{(t)}x\|_2$ has a distribution which does not depend on $x \in \mathbb{S}^{d-1}$. We notice that for any $x \in \mathbb{S}^{d-1}$, there is an orthonormal matrix $Q\in\bR^{d \times d}$ such that $x = Qe_1$, where $e_1 = (1,0,\dots, 0) ^{\top} \in \mathbb{S}^{d-1}$. We are going to show that the entries of $ W^{(t)}Q$, which we denote as $\{Z_{ij}^{(t)}\}$'s, are also i.i.d. Gaussian distributions $\calN(0,\tau^2)$. By definition, we see that $Z_{ij}^{(t)} = \sum_{k=1} ^d W^{(t)}_{ik}Q_{kj}$,
    which is a linear combination of Gaussian distributions, and hence $Z_{ij}^{(t)}$ is also a Gaussian distribution. Also, for $(i_1,j_1) , (i_2,j_2) \in \{1,\dots,d\} \times \{1,\dots,d\}$, we have that
        \begin{align*}
            \bE \left( Z_{i_1 j_1}^{(t)}  Z_{i_2 j_2}^{(t)}\right) = \bE \left( \sum_{k,l=1}^d W^{(t)}_{i_1 k}Q_{k j_1} W^{(t)}_{i_2 l} Q_{l j_2}\right) = \sum_{k,l=1}^d Q_{k j_1} Q_{l j_2} \bE\left(W^{(t)}_{i_1 k} W^{(t)}_{i_2 l}\right).
        \end{align*}
    If $i_1 \neq i_2$, then $\bE(W^{(t)}_{i_1 k} W^{(t)}_{i_2 l}) = 0$ for any $k,l \in \{1,\dots,d\}$ because $W^{(t)}_{i_1 k}$ is independent of $W^{(t)}_{i_2 l}$. If $i_1 = i_2$, then $\bE(W^{(t)}_{i_1 k} W^{(t)}_{i_2 l}) = \tau^2 \delta_{kl}$, and then $\bE \left( Z^{(t)}_{i_1 j_1}  Z^{(t)}_{i_1 j_2}\right) = \tau^2 \sum_{k,l} Q_{k j_1} Q_{l j_2} \delta_{kl} = \tau^2 \sum_k Q_{k j_1} Q_{k j_2}$. Because $Q$ is an orthonormal matrix, its columns consist of an orthonormal basis of $\bR^d$. Hence, $\bE \left( Z^{(t)}_{i_1 j_1}  Z^{(t)}_{i_1 j_2}\right) = \tau^2 \sum_k Q_{k j_1} Q_{k j_2} = \tau^2 \delta_{j_1 j_2}$. We can then conclude that the entries of $W^{(t)}Q$ are also i.i.d. standard Gaussian distributions $\calN(0,\tau^2)$. Repeating the same arguments yields that the entries of $Q^\top W^{(t)}Q$ are still i.i.d. standard Gaussian distributions $\calN(0,\tau^2)$, which implies that the distribution of $\|(\mathrm{Id} + \beta W^{(t)})x\|_2 = \|(\mathrm{Id} + \beta Q^\top W^{(t)}Q)e_1\|_2$ is the same as the distribution of $\|(\mathrm{Id} + \beta W^{(t)})e_1\|_2$ and is hence independent of $x\in\mathbb{S}^{d-1}$.

    Therefore, we can apply \Cref{thm: explicit Lya exponent} and conclude that 
    \begin{equation*}
        \log R(\beta, \bP_W) = \bE \log \|(\mathrm{Id} +\beta W^{(t)})e_1\|_2 = \frac{1}{2} \bE \log \left((1  + \beta\tau \xi) ^2 + \beta^2\tau^2 \chi_{d-1} ^2 \right),
    \end{equation*}
    where $\xi$ is a $1$-dimensional standard Gaussian distribution $\calN(0,1)$, and $\chi^2_{d-1}$ is a chi-squared distribution with $d-1$ degrees of freedom that is independent of $\xi$.
\end{proof}

Next, we present the proof of \Cref{thm2:gaussian_W}, for which we need the following lemma.

\begin{lemma}\label{lem:monotone_gaussian}
    Let $\xi\sim\calN(0,1)$ and fix $b\geq 0$. The function
    \begin{equation*}
        f(a):= \bE \log\left((a+\xi)^2 + b\right)
    \end{equation*}
    is strictly increasing on $(0,+\infty)$.
\end{lemma}

\begin{proof}
    It can be computed that
    \begin{align*}
        f(a) = \frac{1}{\sqrt{2\pi}} \int_\bR \log\left((a+\eta)^2 + b\right) e^{-\frac{\eta^2}{2}}\mathrm{d}\eta,
    \end{align*}
    and that
    \begin{align*}
        f'(a) & = \frac{1}{\sqrt{2\pi}} \int_\bR \frac{2(a+\eta)}{(a+\eta)^2 + b} e^{-\frac{\eta^2}{2}}\mathrm{d}\eta = \frac{\sqrt{2}}{\sqrt{\pi}} \int_\bR \frac{\eta}{\eta^2 + b} e^{-\frac{(\eta-a)^2}{2}}\mathrm{d}\eta \\
        &  = \frac{\sqrt{2}}{\sqrt{\pi}}\left(\int_0^{+\infty} \frac{\eta}{\eta^2 + b} e^{-\frac{(\eta-a)^2}{2}}\mathrm{d}\eta + \int_0^{+\infty} \frac{-\eta}{\eta^2 + b} e^{-\frac{(-\eta-a)^2}{2}}\mathrm{d}\eta\right) \\
        & = \frac{\sqrt{2}}{\sqrt{\pi}} \int_0^{+\infty} \frac{\eta}{\eta^2 + b} \left(e^{-\frac{(\eta-a)^2}{2}} - e^{-\frac{(\eta+a)^2}{2}} \right)\mathrm{d}\eta > 0,
    \end{align*}
    which completes the proof.
\end{proof}

\begin{proof}[Proof of \Cref{thm2:gaussian_W}]
    By using \Cref{lem:monotone_gaussian}, one can compute for any $\lambda\in\mathrm{spec}(P)\subseteq(-1,1)$ that
    \begin{align*}
        & \bE \log \left((1+\alpha\lambda\tau\xi)^2 + \alpha^2\lambda^2\tau^2\xi_{d-1}^2\right) = \bE \log \left((1+\alpha|\lambda|\tau\xi)^2 + \alpha^2\lambda^2\tau^2\chi_{d-1}^2\right) \\
        & \qquad = \bE_{\chi_{d-1}}\bE_\xi \log \left((1+\alpha|\lambda|\tau\xi)^2 + \alpha^2\lambda^2\tau^2\chi_{d-1}^2\right) \\
        &\qquad = \log(\alpha^2\lambda^2\tau^2) + \bE_{\chi_{d-1}}\bE_\xi \log \left(\left(\frac{1}{\alpha|\lambda|\tau}+\xi\right)^2 + \chi_{d-1}^2\right) \\
        &\qquad > \log(\alpha^2\lambda^2\tau^2) + \bE_{\chi_{d-1}}\bE_\xi \log \left(\left(\frac{1}{\alpha\tau}+\xi\right)^2 + \chi_{d-1}^2\right) \\
        & \qquad = 2\log|\lambda| + \bE_{\chi_{d-1}}\bE_\xi \log \left((1+\alpha\tau\xi)^2 + \alpha^2\tau^2\chi_{d-1}^2\right).
    \end{align*}
    Therefore, we have
    \begin{equation*}
        \exp\left(\frac{1}{2}\bE \log \left((1+\alpha\lambda\tau\xi)^2 + \alpha^2\lambda^2\tau^2\xi_{d-1}^2\right)\right) > |\lambda| \exp\left(\frac{1}{2}\bE \log \left((1+\alpha\tau\xi)^2 + \alpha^2\tau^2\xi_{d-1}^2\right)\right),
    \end{equation*}
    which combined with \eqref{eq:gaussian_W} leads to that
    \begin{align*}
         \lim_{t\to+\infty} \mu(x_{\mathrm{rs}}(t))^{1/2t} & \geq \frac{ \max\limits_{\lambda\in \mathrm{spec}(P)\backslash\{1\}} \exp\left(\frac{1}{2} \bE \log \left((1  + \alpha\lambda\tau \xi) ^2 + \alpha^2\lambda^2\tau^2 \chi_{d-1} ^2 \right)\right)}{ \max\limits_{\lambda\in \mathrm{spec}(P)} \exp\left(\frac{1}{2} \bE \log \left((1  + \alpha\lambda\tau \xi) ^2 + \alpha^2\lambda^2\tau^2 \chi_{d-1} ^2 \right)\right)} \\
         & > \max_{\lambda\mathrm{spec}(P)\backslash\{1\}} |\lambda| = \lim_{t\to+\infty} \mu(x_{\mathrm{nrs}}(t))^{1/2t}.
    \end{align*}
    \end{proof}

\subsubsection{Proof of \Cref{thm:bounded_W}}
    For any $\lambda\in\mathrm{spec}(P)$, one can estimate that
    \begin{align*}
        R(\alpha\lambda,\bP_W) & = \lim_{t\to+\infty} \norm{\left(I+\alpha\lambda W^{(t-1)}\right) \cdots \left( I+\alpha\lambda W^{(1)}\right) \left( I + \alpha\lambda W^{(0)}\right)}^{1/t} \\
        & \leq 1+\alpha|\lambda|r_W \leq 1+\alpha r_W.
    \end{align*}
    Moreover, for $\lambda\in\mathrm{spec}(P)\backslash\{1\}$ and $\alpha\leq 1/r_W$, it follows from
    \begin{equation*}
        \norm{(I+\alpha \lambda W^{(t)}) y} \geq \norm{y} - \alpha |\lambda|\norm{W^{(t)} y} \geq (1-\alpha |\lambda| r_W) \norm{y},
    \end{equation*}
    that 
    \begin{equation*}
        \norm{\left(I+\alpha\lambda W^{(t-1)}\right) \cdots \left( I+\alpha\lambda W^{(1)}\right) \left( I + \alpha\lambda W^{(0)}\right) y}\geq (1-\alpha |\lambda| r_W)^t \norm{y}.
    \end{equation*}
    One can thus conclude that
    \begin{equation*}
        R(\alpha\lambda,\bP_W) = \lim_{t\to+\infty} \norm{\left(I+\alpha\lambda W^{(t-1)}\right) \cdots \left( I+\alpha\lambda W^{(1)}\right) \left( I + \alpha\lambda W^{(0)}\right)}^{1/t} \geq 1-\alpha |\lambda| r_W,
    \end{equation*}
    and hence that
    \begin{equation*}
        \lim_{t\to+\infty} \mu(x_{\mathrm{rs}}{(t)})^{1/2t} \geq \frac{\max_{\lambda\in \mathrm{spec}(P)\backslash\{1\}} R(\alpha\lambda,\bP_W)}{\max_{\lambda\in \mathrm{spec}(P)} R(\alpha\lambda,\bP_W)} \geq \frac{1-\alpha r_W\min_{\lambda\in\mathrm{spec}(P)\backslash\{1\}}|\lambda|}{1+\alpha r_W}.
    \end{equation*}
    If we further assume that 
    \begin{align*}
        \alpha r_W \left(\min_{\lambda\in \mathrm{spec}(P)\backslash \{1\}}|\lambda|+\max_{\lambda\in \mathrm{spec}(P)\backslash \{1\}}|\lambda|\right) <1 - \max_{\lambda\in \mathrm{spec}(P)\backslash \{1\}}|\lambda|,
    \end{align*}
    then we notice that
        \begin{align*}
            \frac{1-\alpha r_W\min_{\lambda\in\mathrm{spec}(P)\backslash\{1\}}|\lambda|}{1+\alpha r_W}  >  \max_{\lambda\in\mathrm{spec}(P)\backslash\{1\}}|\lambda|.
        \end{align*}
    This together with \Cref{thm:nonres_GNN} finishes the proof.
 
\subsubsection{Proofs of \Cref{thm:diagonal_W} and \Cref{thm2:diagonal_W}}

\begin{proof}[Proof of \Cref{thm:diagonal_W}]
    It can be computed directly that
    \begin{align*}
        \log R(\alpha\lambda, \bP_W) & = \lim_{t\to+\infty} \frac{1}{t}\log \norm{\left(I+\alpha\lambda W^{(t-1)}\right) \cdots \left( I+\alpha\lambda W^{(1)}\right) \left( I + \alpha\lambda W^{(0)}\right)} \\
        & =\max_{1\leq i\leq d} \lim_{t\to+\infty}\frac{1}{t} \sum_{j=0}^{t-1}\log \left|1+\alpha\lambda w_i^{(j)}\right| = \max_{1\leq i\leq d} \bE\left[\log \left|1+\alpha\lambda w_i^{(t)}\right|\right],
    \end{align*}
    where we used \Cref{prop:leading Lya on invariant subspaces} and the last step is from the strong law of large numbers. The \Cref{thm:diagonal_W} follows directly from \Cref{thm:res_GNN}.
\end{proof}

\begin{proof}[Proof of \Cref{thm2:diagonal_W}]
    For any $\lambda\in\mathrm{spec}(P)\backslash\{1\}\subseteq(-1,1)$ and any $i\in\{1,\dots,d\}$, because $\alpha \max_{1\leq i\leq d}r_i<1$ and $\bP_i$ is symmetric, we have that
    \begin{align*}
        &\bE\left[\log \left|1+\alpha\lambda w_i^{(t)}\right|\right]  = \bE\left[\log \left(1+\alpha\lambda w_i^{(t)}\right)\right] \\
        &\qquad = \frac{1}{2} \bE\left[\log \left(1+\alpha\lambda w_i^{(t)}\right)\right] +\frac{1}{2} \bE\left[\log \left(1-\alpha\lambda w_i^{(t)}\right)\right] = \frac{1}{2} \bE\left[\log \left(1-\alpha^2\lambda^2 (w_i^{(t)})^2\right)\right] \\
        &\qquad  \geq \frac{1}{2} \bE\left[\log \left(1-\alpha^2 (w_i^{(t)})^2\right)\right] = \frac{1}{2} \bE\left[\log \left(1+\alpha w_i^{(t)}\right)\right] +\frac{1}{2} \bE\left[\log \left(1-\alpha w_i^{(t)}\right)\right] \\
        & \qquad = \bE\left[\log \left(1+\alpha w_i^{(t)}\right)\right] = \bE\left[\log \left|1+\alpha w_i^{(t)}\right|\right].
    \end{align*}
    This implies that
    \begin{equation*}
        \lim_{t\to+\infty} \mu(x_{\mathrm{rs}}{(t)})^{1/2t} \geq \frac{\max\limits_{\lambda\in\mathrm{spec}(P)\backslash\{1\}} \max\limits_{1\leq i\leq d} \exp\left( \bE\left[\log \left|1+\alpha\lambda w_i^{(t)}\right|\right]\right)}{\max\limits_{\lambda\in\mathrm{spec}(P)} \max\limits_{1\leq i\leq d} \exp\left( \bE\left[\log \left|1+\alpha\lambda w_i^{(t)}\right|\right]\right)} = 1.
    \end{equation*}
    The proof is thus completed.
\end{proof}

\section{Numerical Experiments}
\label{sec:numerics}

In this section, we validate our theoretical results and discuss some practical insights via numerical experiments.

\subsection{Experimental setup}

Our experiments utilize three citation network datasets: Cora, CiteSeer, and PubMed~\cite{yang2016revisiting}. In these datasets, vertices represent publications, and edges denote citation relationships. Features are encoded as binary bag-of-words representations based on the corpus vocabulary, where a feature value of $1$ signifies the presence of a word in the paper and $0$ indicates its absence. Each vertex is labeled with a class corresponding to the category of the publication. For all three datasets, the graphs are treated as undirected, with symmetric adjacency matrices $A$. We use the dataset versions provided in the \textsc{PyTorch Geometric} (\textsc{PyG}) library~\cite{fey2019fast}.

For each dataset, we implement GCNs and residual GCNs with $P = D^{-1}A$, considering the linear activation $\sigma = \mathrm{Id}$ and two nonlinear activation functions: ReLU and LeakyRuLU. The negative slope in LeakyRuLU is set to be $0.8$. In all models, each message-passing layer, i.e., \eqref{eq:GNN} or \eqref{eq:res_GNN}, is configured with $32$ hidden dimensions, preceded by an embedding MLP layer for vertex feature encoding and followed by an output MLP layer for classification. We set $\alpha = 0.1$ in all residual connections.

The properties of the datasets are summarized in \Cref{tab:citation-networks}.

\begin{table}[htb!]
  \centering
  \label{tab:summary}
  \begin{tabular}{ccccccc}
    \hline
    Dataset & \#Vertices & \#Edges & \#Features & \#Classes & \#CC & $\lambda_{2,\mathrm{LCC}}(P)$ \\
    \hline
    Cora & 2708(2485) & 5278(5069) & 1433 & 7 & 78 & 0.99638\\
    Citeseer & 3327(2120) & 4552(3679) & 3703 & 6 & 438 & 0.99874 \\
    PubMed & 19717 & 44324 & 500 & 3 & 1 & 0.99052 \\
    \hline
  \end{tabular}
  \caption{Properties of citation network datasets: Cora, CiteSeer, and PubMed~\cite{yang2016revisiting}. We use the versions available in \textsc{PyG}. Abbreviations: CC = connected components, LCC = largest connected components. The numbers in parentheses indicate the numbers of vertices and edges in the LCC of each graph. $\lambda_{2,\mathrm{LCC}}(P) = \max_{\lambda\in\mathrm{spec}(P|_{\mathrm{LCC}})\backslash\{1\}} |\lambda|$ is the second-largest magnitude of eigenvalues of $P$ on LCC.}
  \label{tab:citation-networks}
\end{table}

\subsection{Vertex similarity of initialized GCNs and residual GCNs}
We begin by comparing the vertex similarity at initialization with the standard method in the \textsc{PyG} library. Specifically, the weights are initialized using the Xavier uniform distribution~\cite{glorot2010understanding}.

In \Cref{fig:init_node_sim}, we illustrate the evolution of both the normalized vertex similarity measure, $\mu(x(t))$, as defined in \Cref{def:vertex sim}, and the unnormalized vertex similarity measure, $\tilde{\mu}(x(t))$, defined as
\begin{equation}\label{eq:unormalized_vertex_sim}
    \Tilde{\mu}(x) = \sum_{i=1}^n \|x_i - \Bar{x}\|_2^2,
\end{equation}
on the largest connected component of each graph as the input $x(t)$ propagates through the model with $1000$ message-passing layers. Each experiment is repeated 15 times. The solid lines in \Cref{fig:init_node_sim} represent $\exp(\bE[\log \mu(x(t))])$ and $\exp(\bE[\log \tilde{\mu}(x(t))])$, where the expectation is taken over the 15 independent trials. The shaded regions indicate the standard deviation across these trials.

\begin{figure}[t]
     \centering
     \begin{subfigure}[b]{0.49\textwidth}
         \centering
         \includegraphics[width=0.7\textwidth]{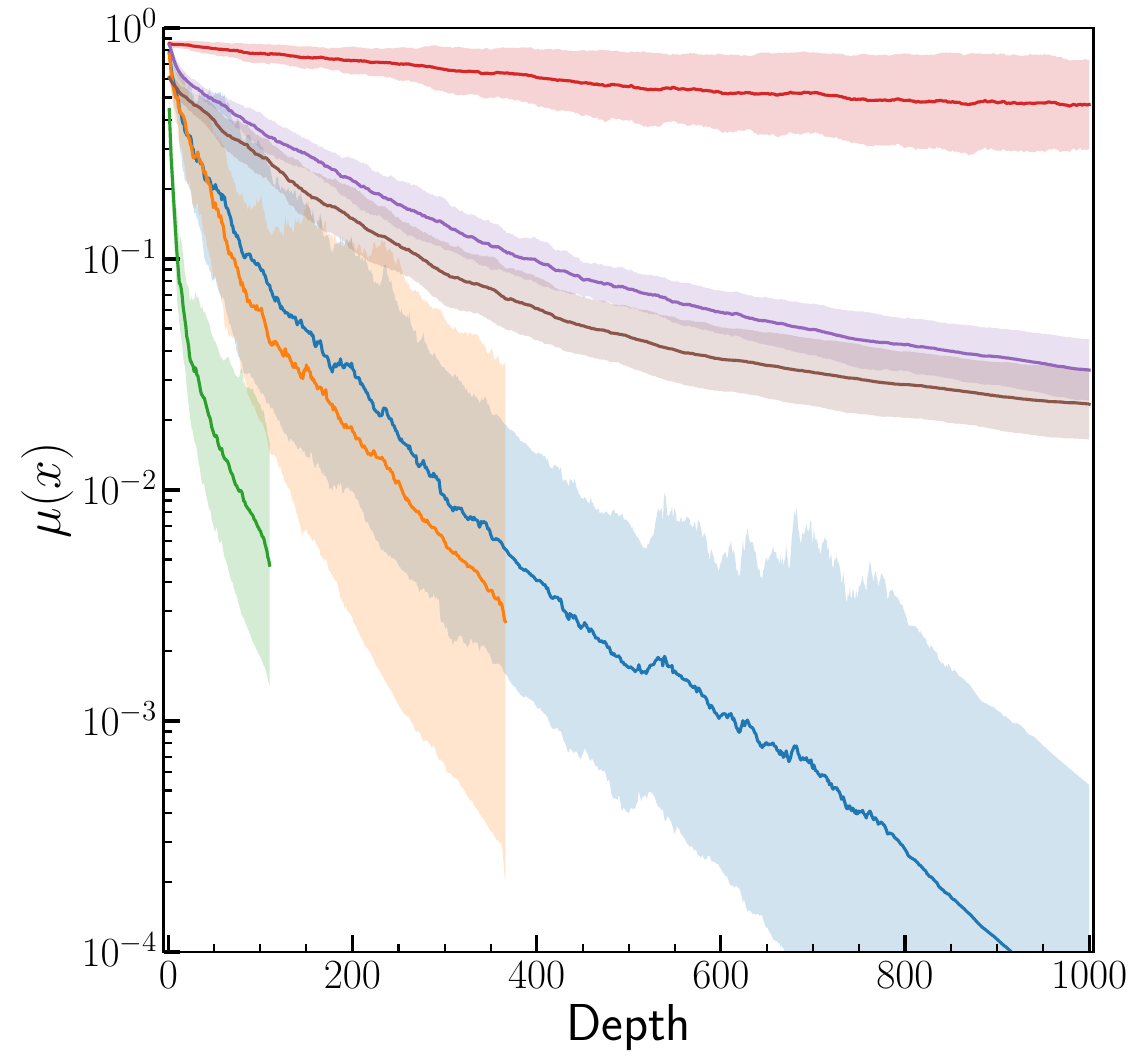}
         \caption{normalized vertex similarity on Cora}
     \end{subfigure}
     \hfill
     \begin{subfigure}[b]{0.49\textwidth}
         \centering
         \includegraphics[width=0.7\textwidth]{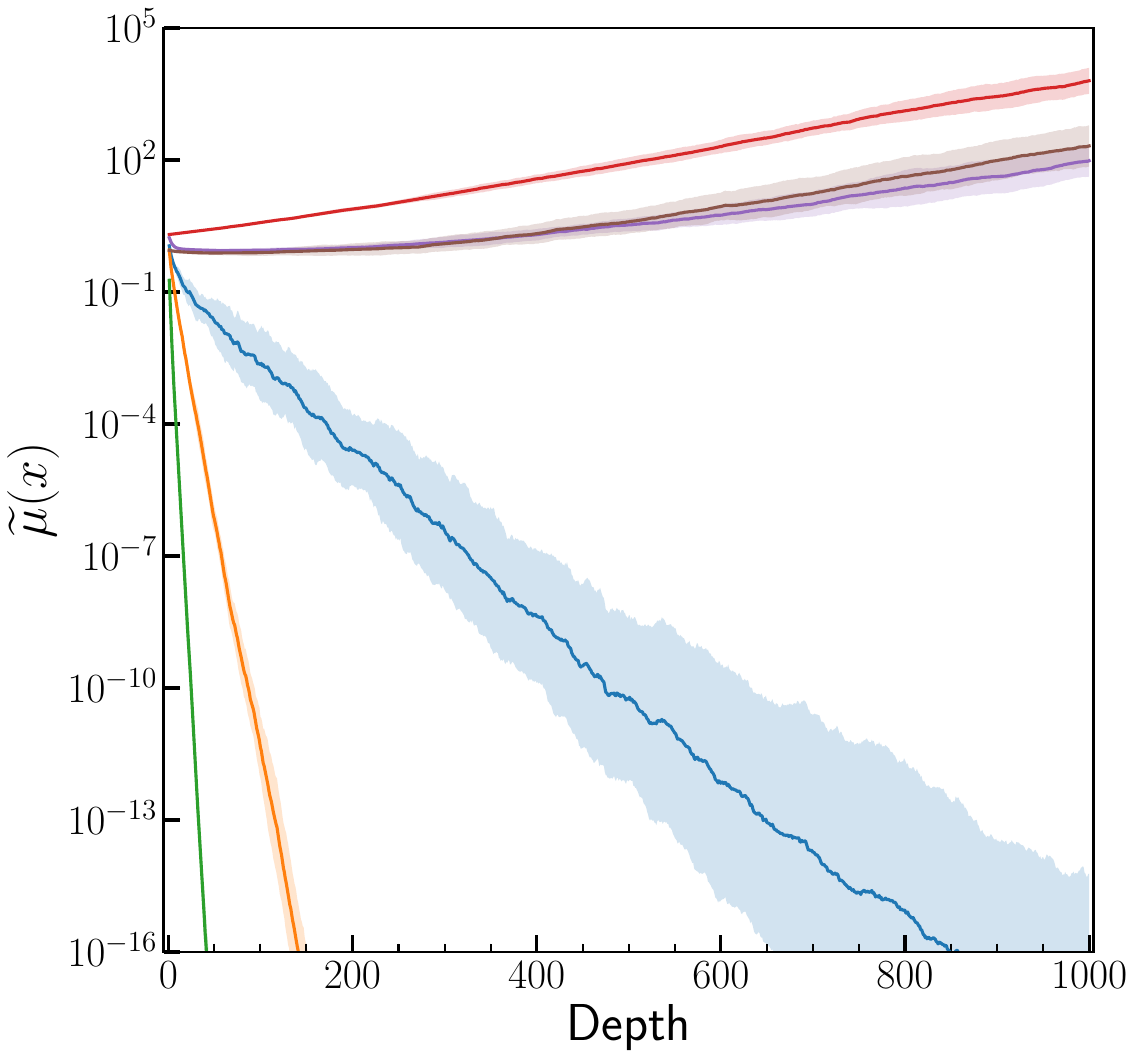}
         \caption{unnormalized vertex similarity on Cora}
     \end{subfigure}
     \begin{subfigure}[b]{0.49\textwidth}
         \centering
         \includegraphics[width=0.7\textwidth]{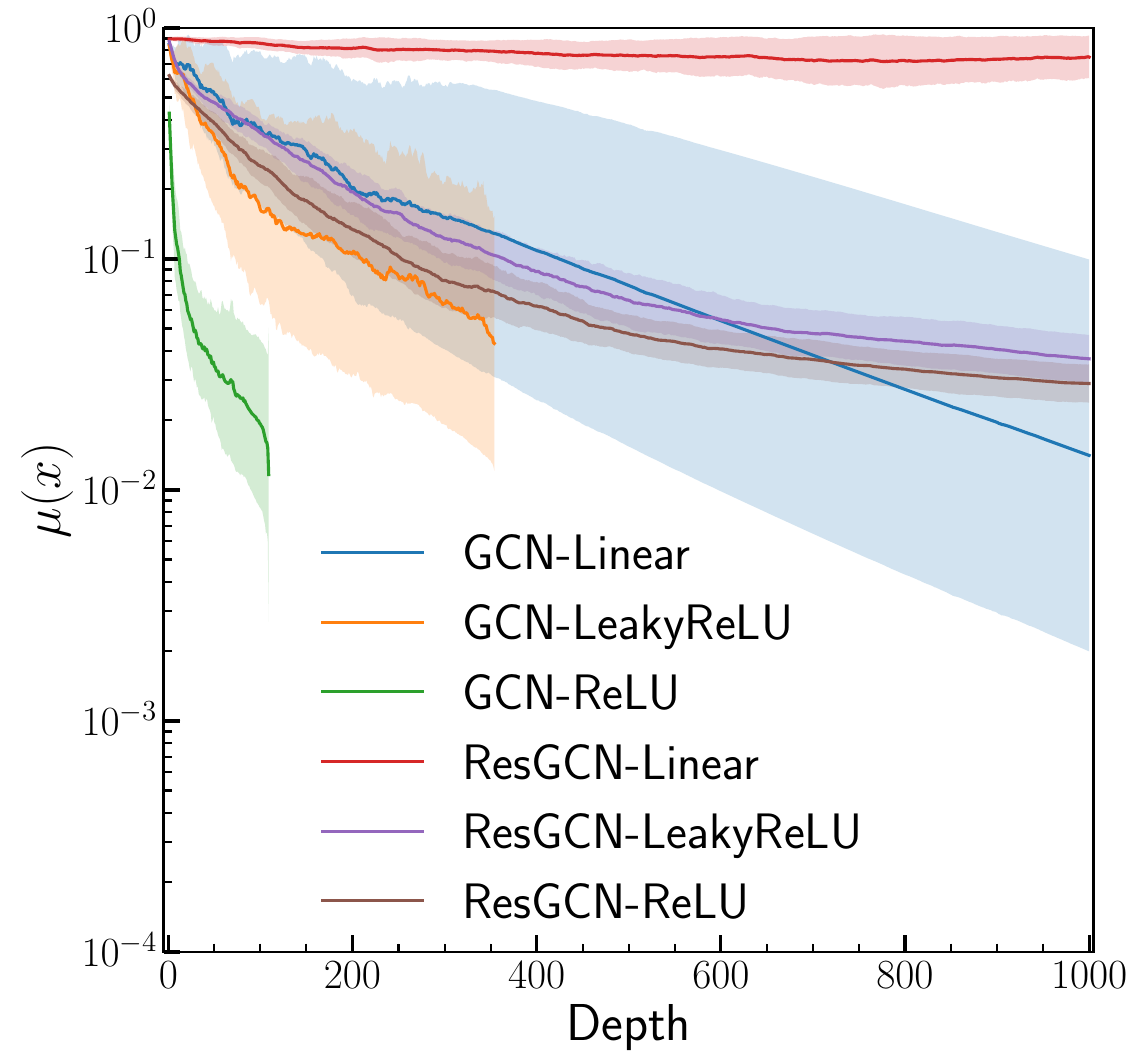}
         \caption{normalized vertex similarity on CiteSeer}
     \end{subfigure}
     \hfill
     \begin{subfigure}[b]{0.49\textwidth}
         \centering
         \includegraphics[width=0.7\textwidth]{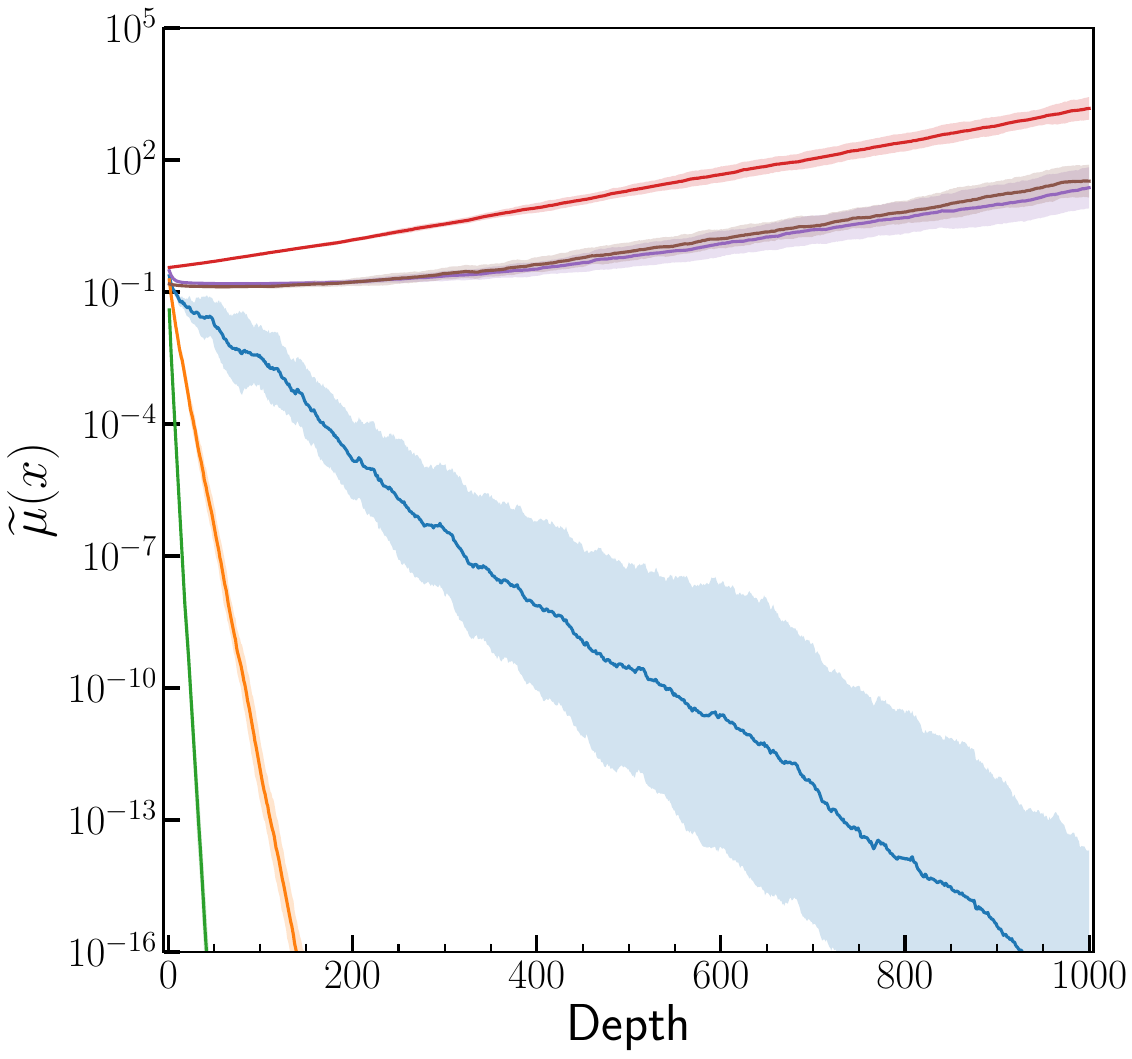}
         \caption{unnormalized vertex similarity on CiteSeer}
     \end{subfigure}
     \begin{subfigure}[b]{0.49\textwidth}
         \centering
         \includegraphics[width=0.7\textwidth]{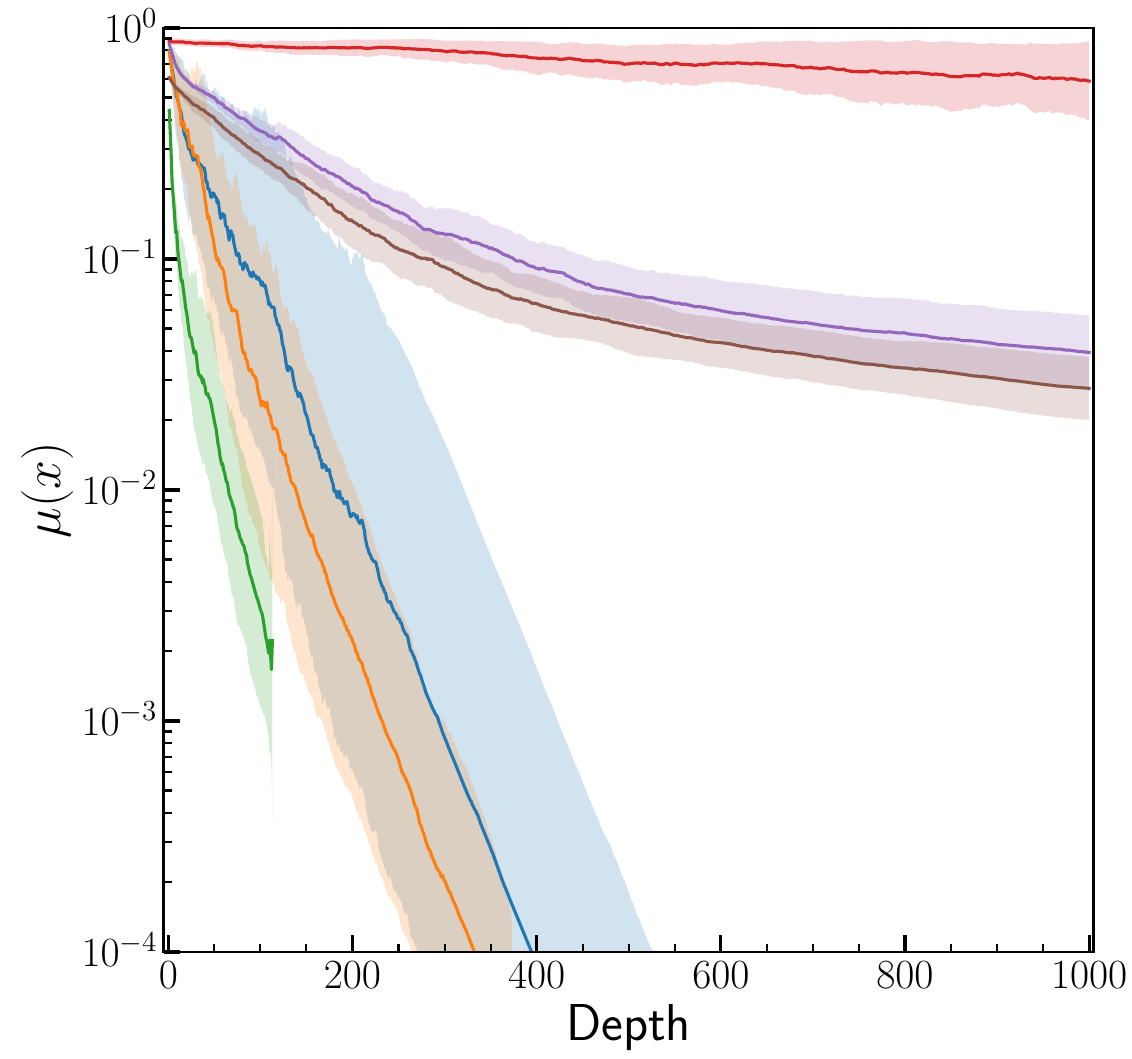}
         \caption{normalized vertex similarity on PubMed}
     \end{subfigure}
     \hfill
     \begin{subfigure}[b]{0.49\textwidth}
         \centering
         \includegraphics[width=0.7\textwidth]{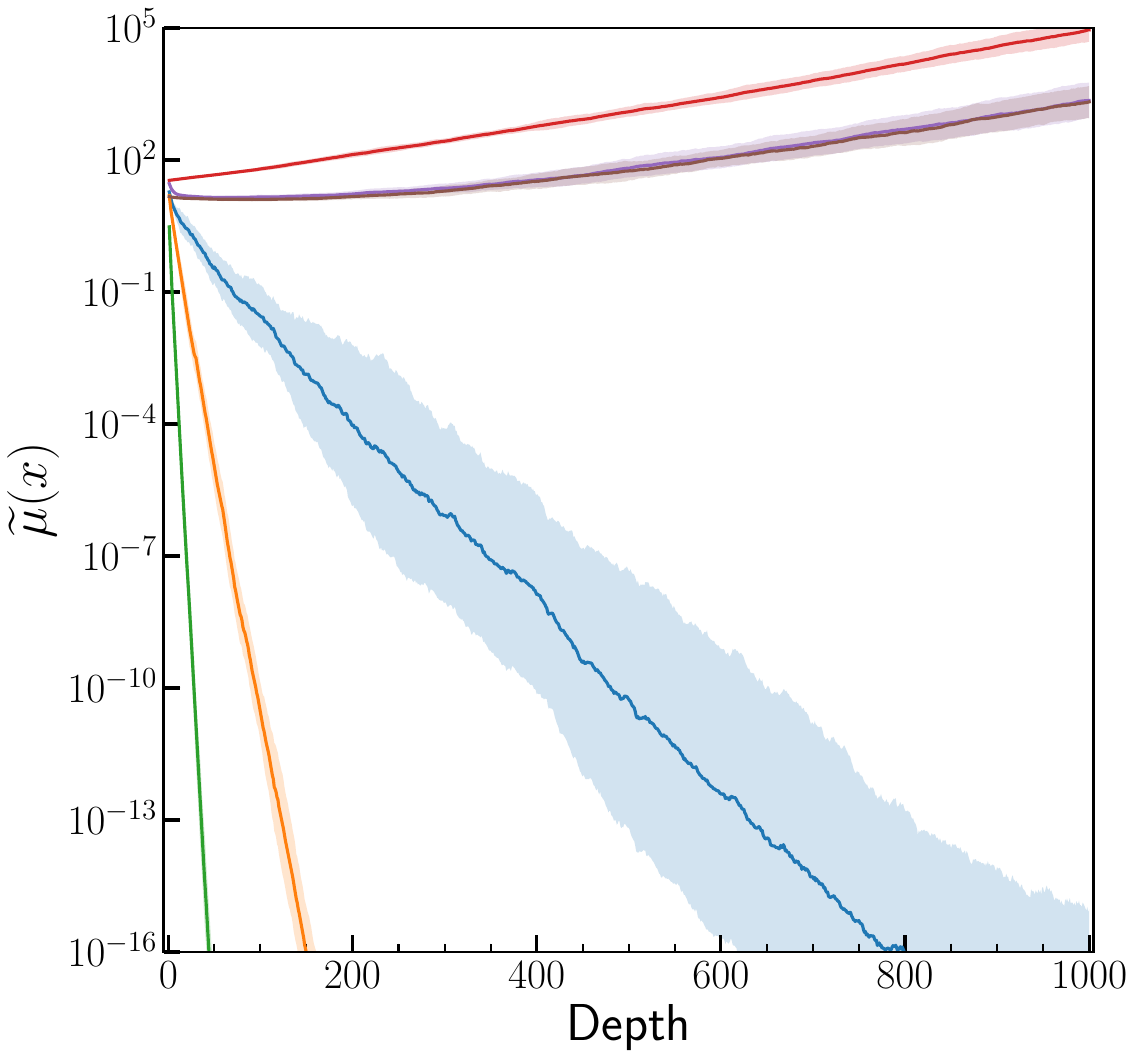}
         \caption{unnormalized vertex similarity on PubMed}
     \end{subfigure}
    \caption{Vertex similarity measure of initialized GCNs and residual GCNs on the largest connected component}
    \label{fig:init_node_sim}
\end{figure}

\begin{figure}[t]
     \centering
     \begin{subfigure}[b]{0.49\textwidth}
         \centering
         \includegraphics[width=0.7\textwidth]{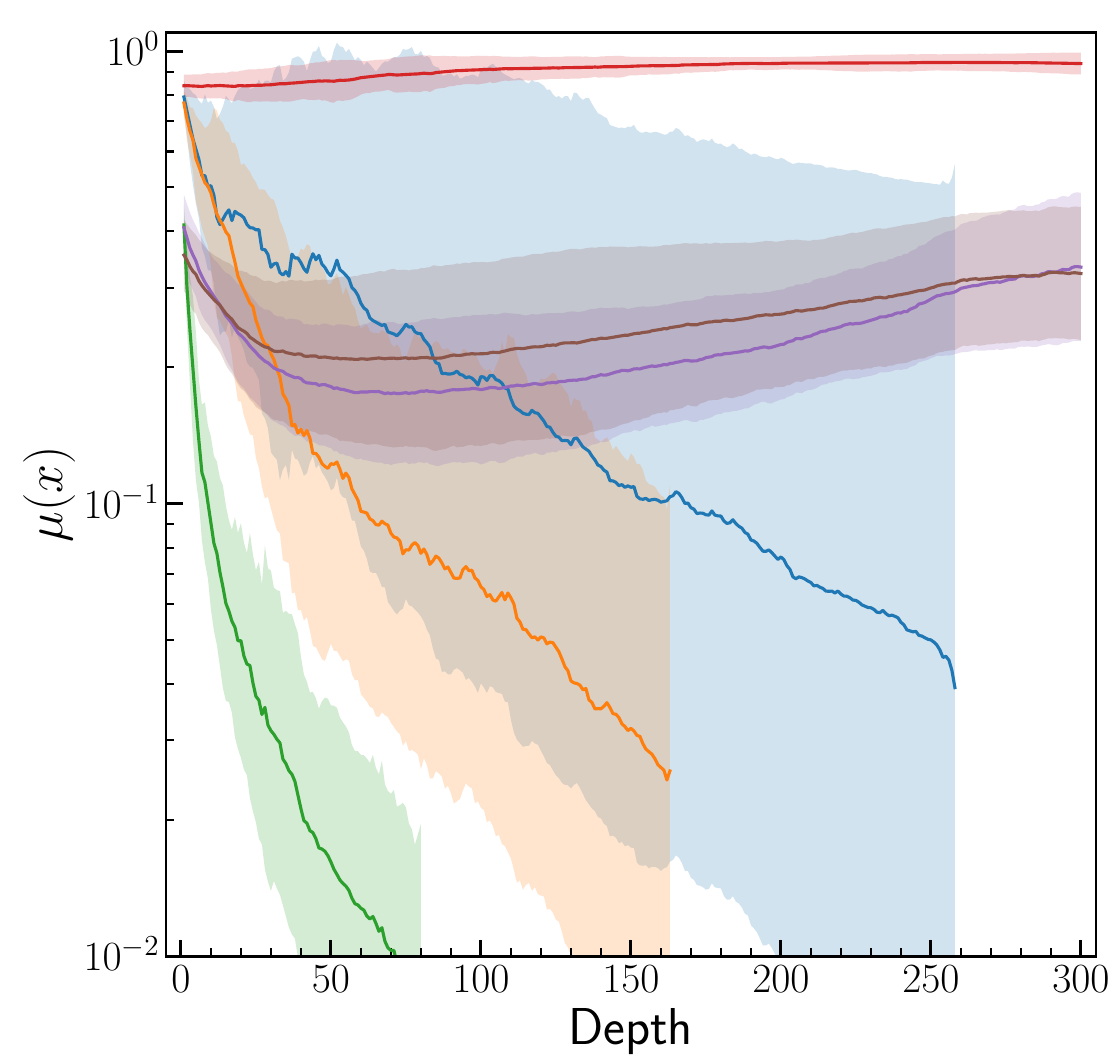}
         \caption{normalized vertex similarity on Cora}
     \end{subfigure}
     \hfill
     \begin{subfigure}[b]{0.49\textwidth}
         \centering
         \includegraphics[width=0.7\textwidth]{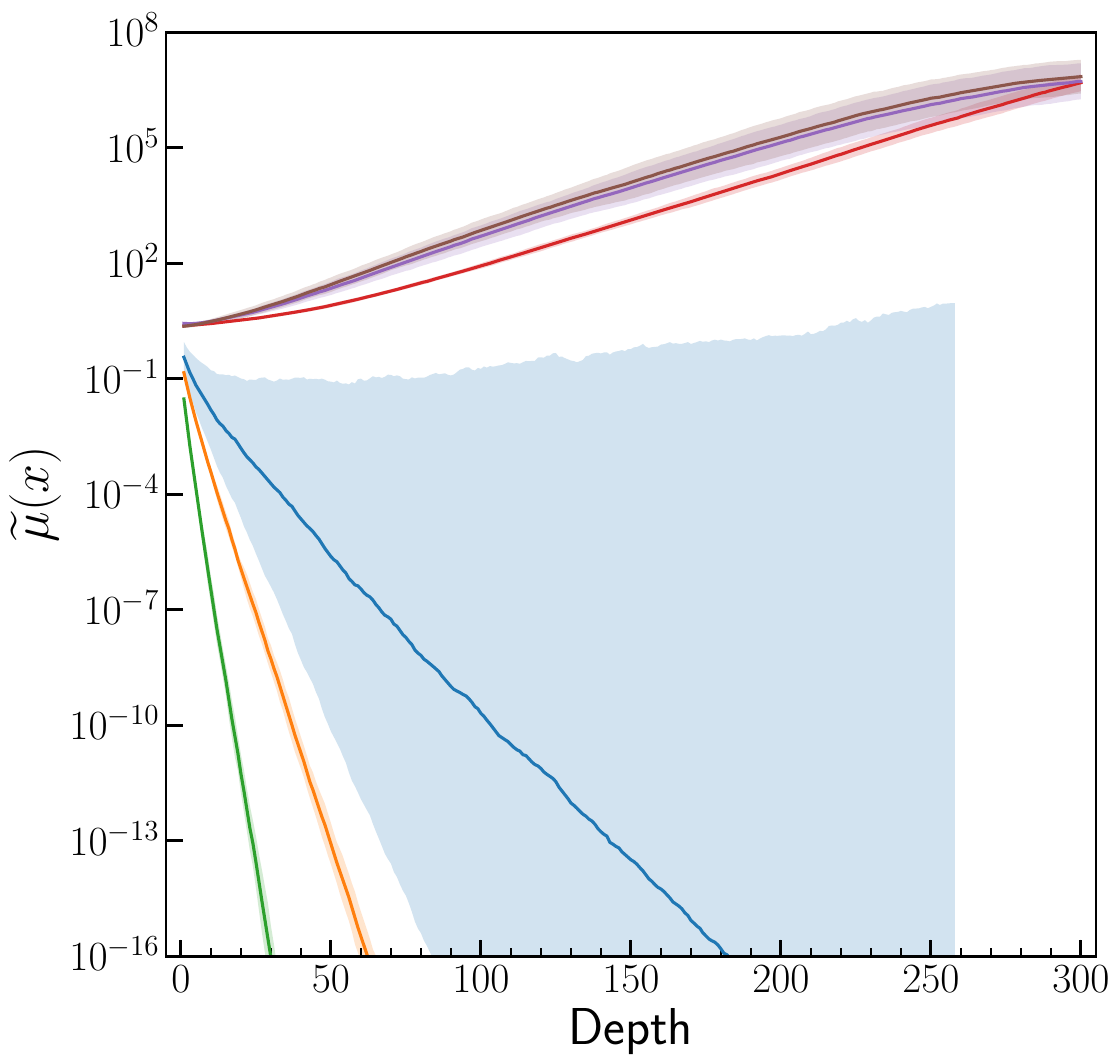}
         \caption{unnormalized vertex similarity on Cora}
     \end{subfigure}
     \begin{subfigure}[b]{0.49\textwidth}
         \centering
         \includegraphics[width=0.7\textwidth]{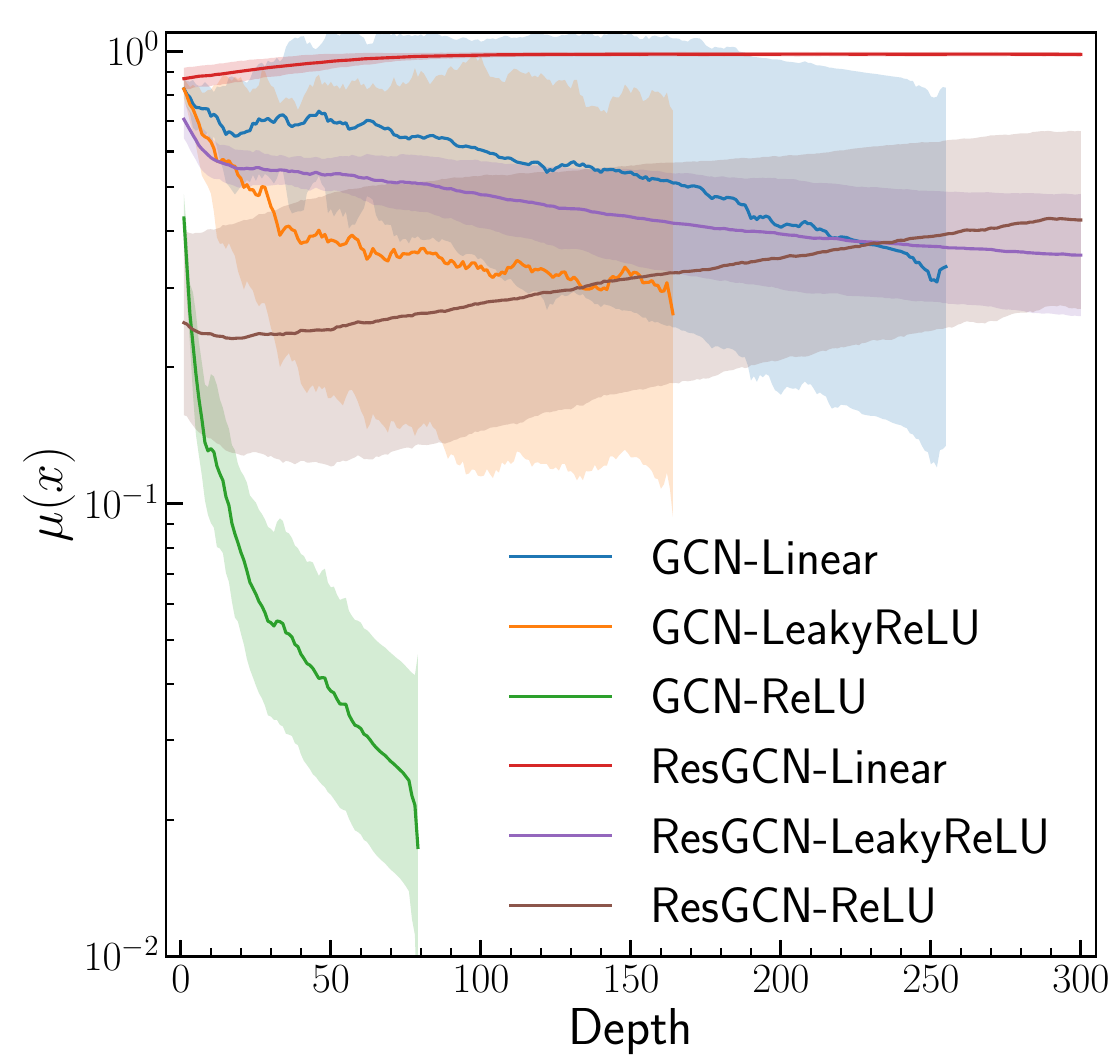}
         \caption{normalized vertex similarity on CiteSeer}
     \end{subfigure}
     \hfill
     \begin{subfigure}[b]{0.49\textwidth}
         \centering
         \includegraphics[width=0.7\textwidth]{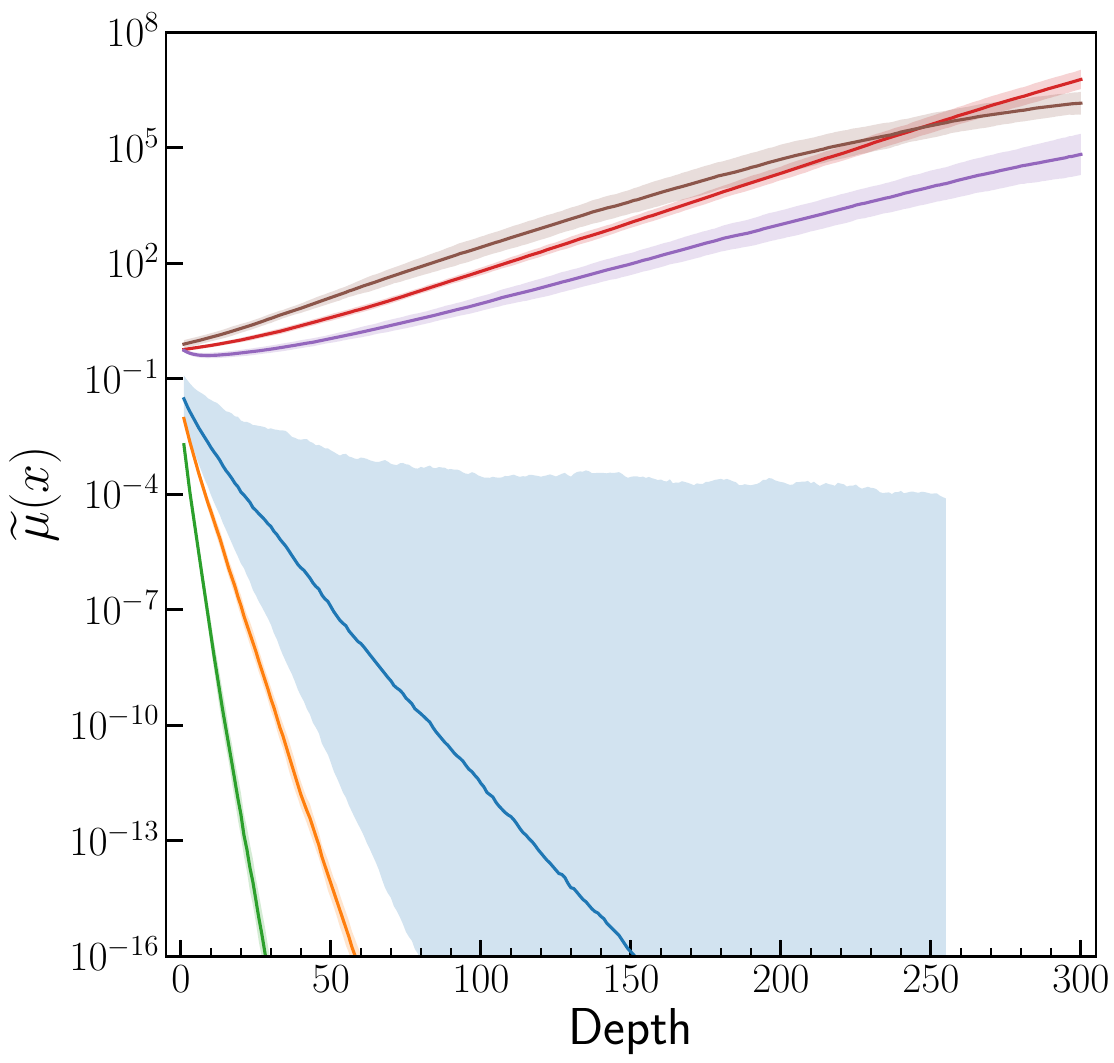}
         \caption{unnormalized vertex similarity on CiteSeer}
     \end{subfigure}
     \begin{subfigure}[b]{0.49\textwidth}
         \centering
         \includegraphics[width=0.7\textwidth]{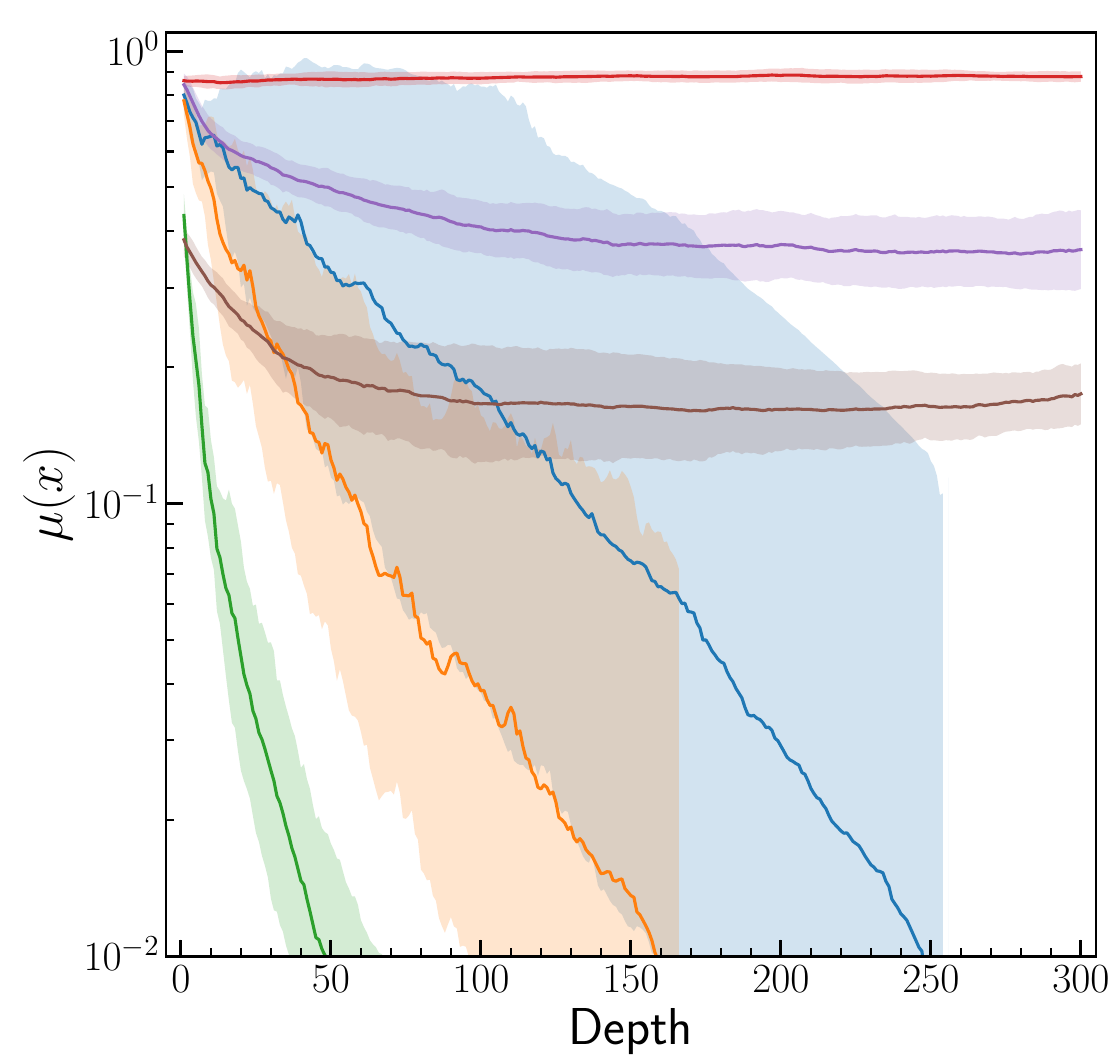}
         \caption{normalized vertex similarity on PubMed}
     \end{subfigure}
     \hfill
     \begin{subfigure}[b]{0.49\textwidth}
         \centering
         \includegraphics[width=0.7\textwidth]{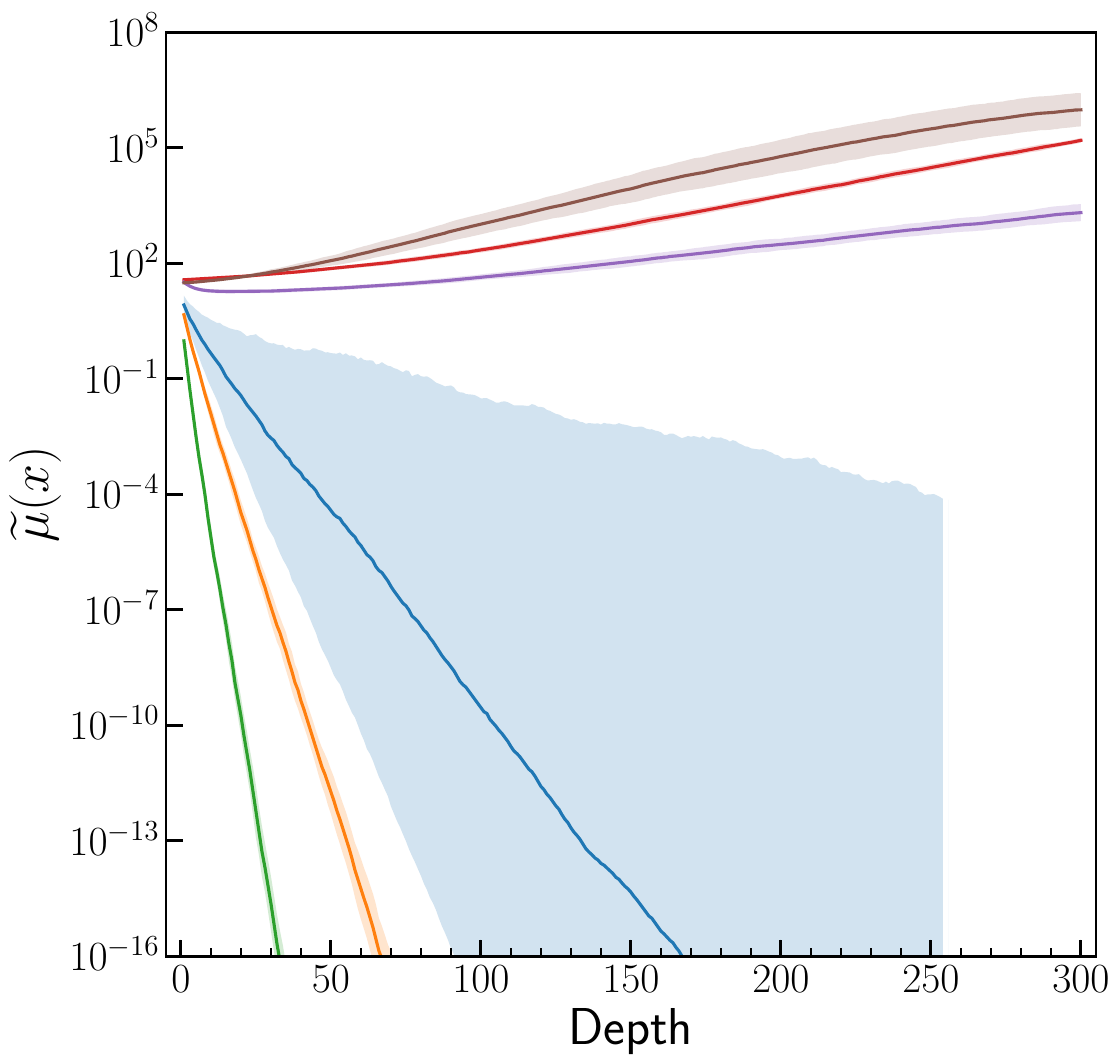}
         \caption{unnormalized vertex similarity on PubMed}
     \end{subfigure}
    \caption{Vertex similarity measure of trained GCNs and residual GCNs on the largest connected component}
    \label{fig:train_node_sim}
\end{figure}

As shown in \Cref{fig:init_node_sim}, the normalized vertex similarity measure, $\mu(x(t))$, for GCNs decays at a faster rate compared to residual GCNs across all settings, corroborating our theoretical findings in \Cref{sec:main_results}. Furthermore, the benefit of residual connections becomes more pronounced when $\lambda_{2,\mathrm{LCC}}(P)$ is smaller. This observation aligns with our theory, as \Cref{thm:nonres_GNN} establishes that $\mu(x(t))$ of deep GCNs diminishes more rapidly when $\lambda_{2,\mathrm{LCC}}(P)$ is smaller.

Additionally, due to numerical issues when the denominator in \eqref{eq:vertex_sim} becomes sufficiently small, it is challenging to observe $\mu(x(t))$ approaching a near-zero value. To address this, we plot the unnormalized vertex similarity measure, $\Tilde{\mu}(x(t))$, as defined in \eqref{eq:unormalized_vertex_sim}. As seen in \Cref{fig:init_node_sim}, $\Tilde{\mu}(x(t))$ converges to zero exponentially fast for GCNs, while for residual GCNs, it remains away from zero.

\subsection{Vertex similarity and performance of trained GCNs and residual GCNs}

In this subsection, we present the vertex similarity and performance of GCNs and residual GCNs after training. We utilize the ``full split'' mode in the \textsc{PyG} library to generate train, validation, and test masks for each dataset. Specifically, $500$ vertices are used for validation, $1200$ vertices for test, following the setup in~\cite{chen2018fastgcn}, and all remaining vertices are used for training. The cross-entropy loss is employed, and the Adam optimizer \cite{kingma2014adam} is used to minimize the loss, with a fixed learning rate of $1\times10^{-5}$ and weight decay of $5\times10^{-4}$. Each model is trained for $3000$ epochs, and all experiments are performed with 15 independent trials. The models are trained on an Nvidia RTX 4090 GPU. 

\Cref{fig:train_node_sim} shows the vertex similarity measure for trained GCNs and residual GCNs of 300 message-passing layers, with all other settings consistent with \Cref{fig:init_node_sim}. The observations are similar to those in \Cref{fig:init_node_sim}. In particular, the normalized vertex similarity measure $\mu(x(t))$ for GCNs decays more quickly than for residual GCNs. Additionally, the unnormalized vertex similarity measure $\Tilde{\mu}(x(t))$ for GCNs converges exponentially to $0$, while it remains away from $0$ for residual GCNs.

\begin{figure}[t]
     \centering
     \begin{subfigure}[b]{0.49\textwidth}
         \centering
         \includegraphics[width=0.7\textwidth]{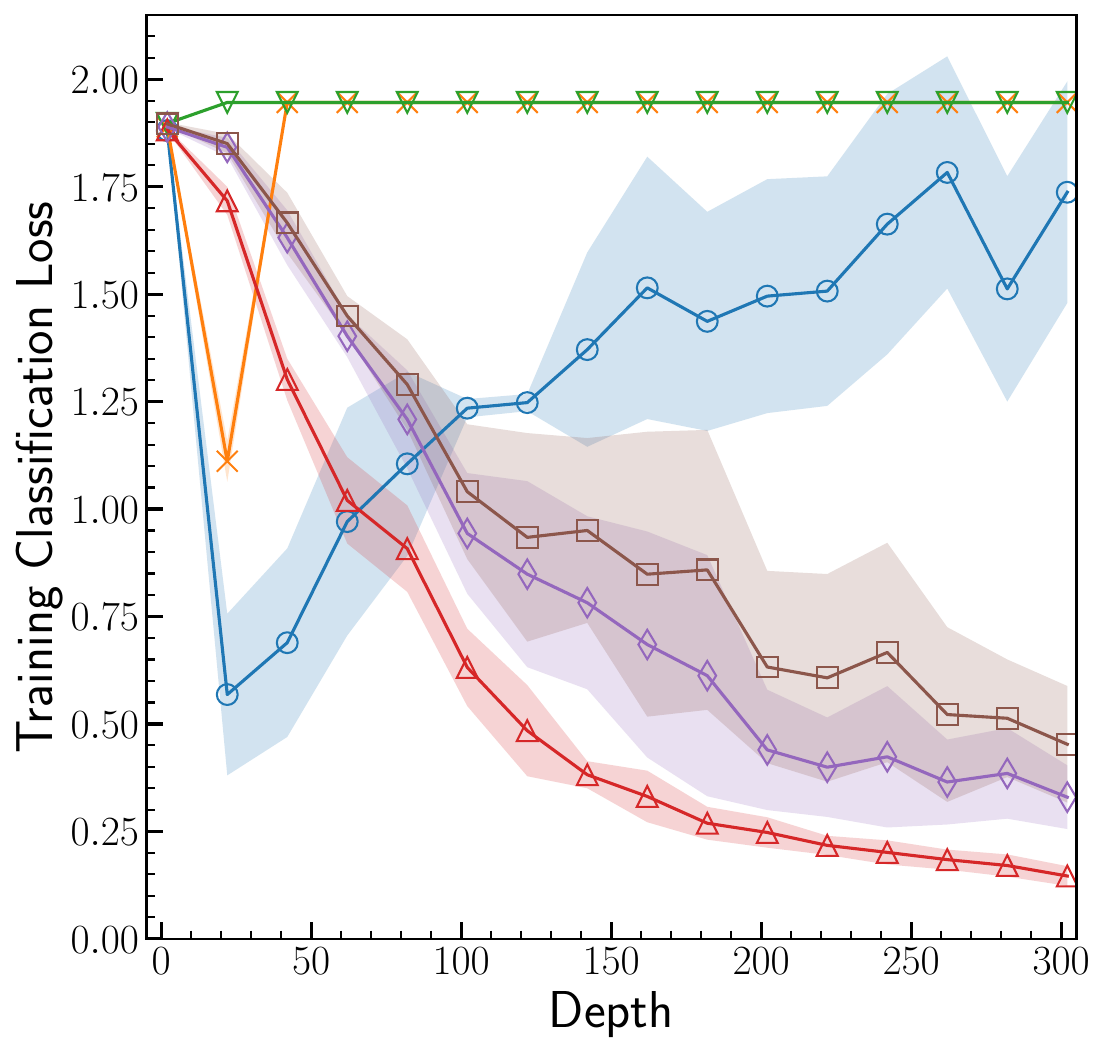}
         \caption{training loss on Cora}
     \end{subfigure}
     \hfill
     \begin{subfigure}[b]{0.49\textwidth}
         \centering
         \includegraphics[width=0.7\textwidth]{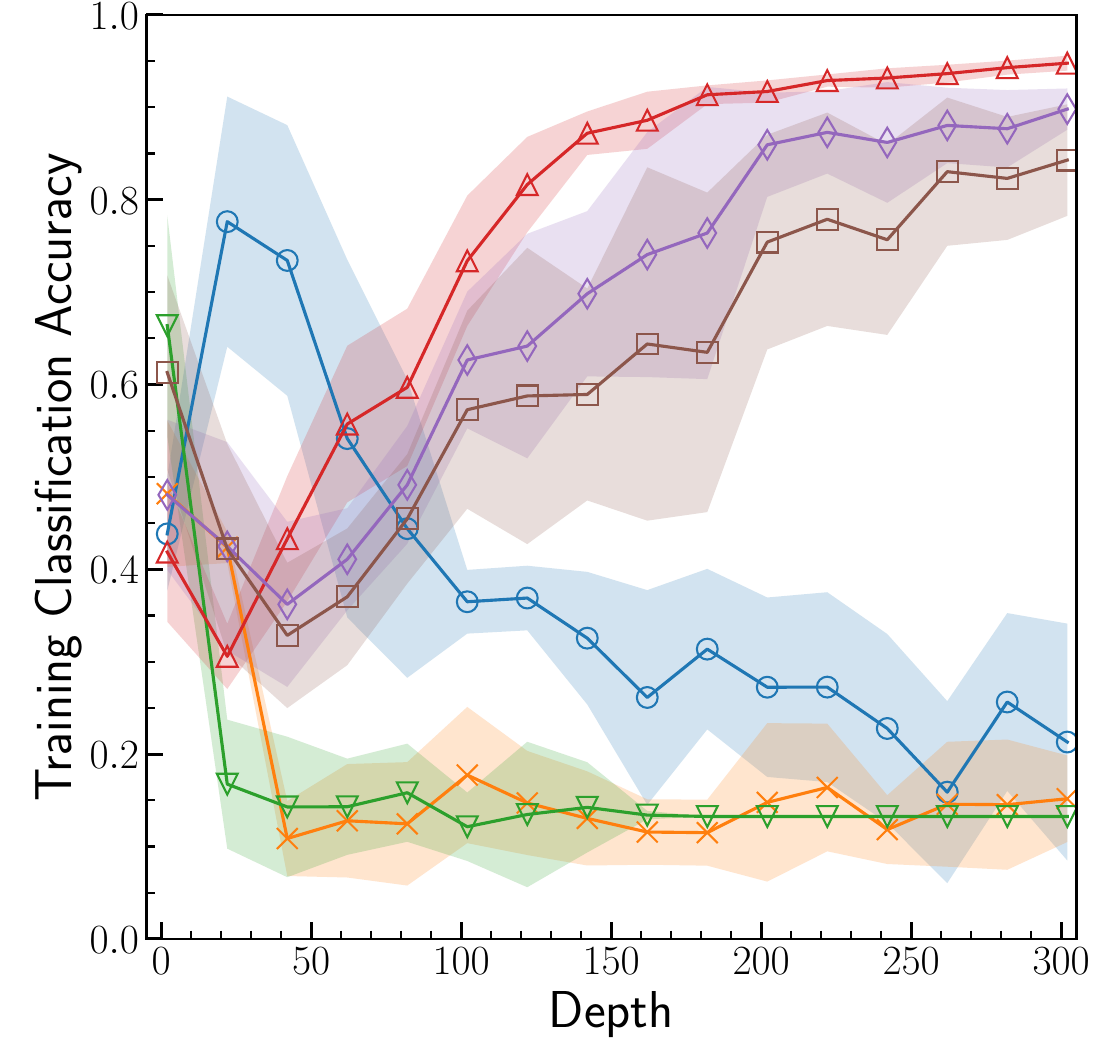}
         \caption{training classification accuracy on Cora}
     \end{subfigure}
     \begin{subfigure}[b]{0.49\textwidth}
         \centering
         \includegraphics[width=0.7\textwidth]{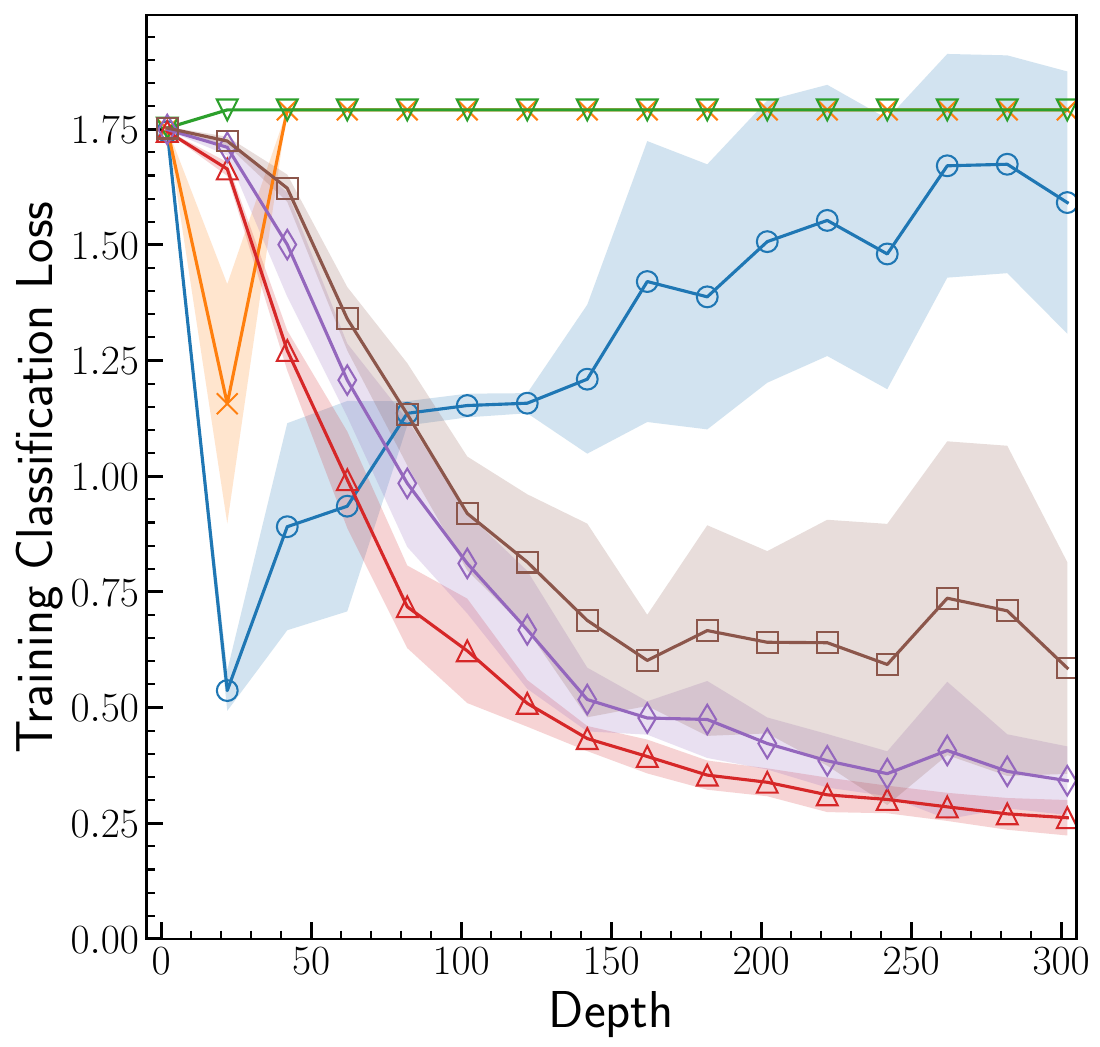}
         \caption{training loss on CiteSeer}
     \end{subfigure}
     \hfill
     \begin{subfigure}[b]{0.49\textwidth}
         \centering
         \includegraphics[width=0.7\textwidth]{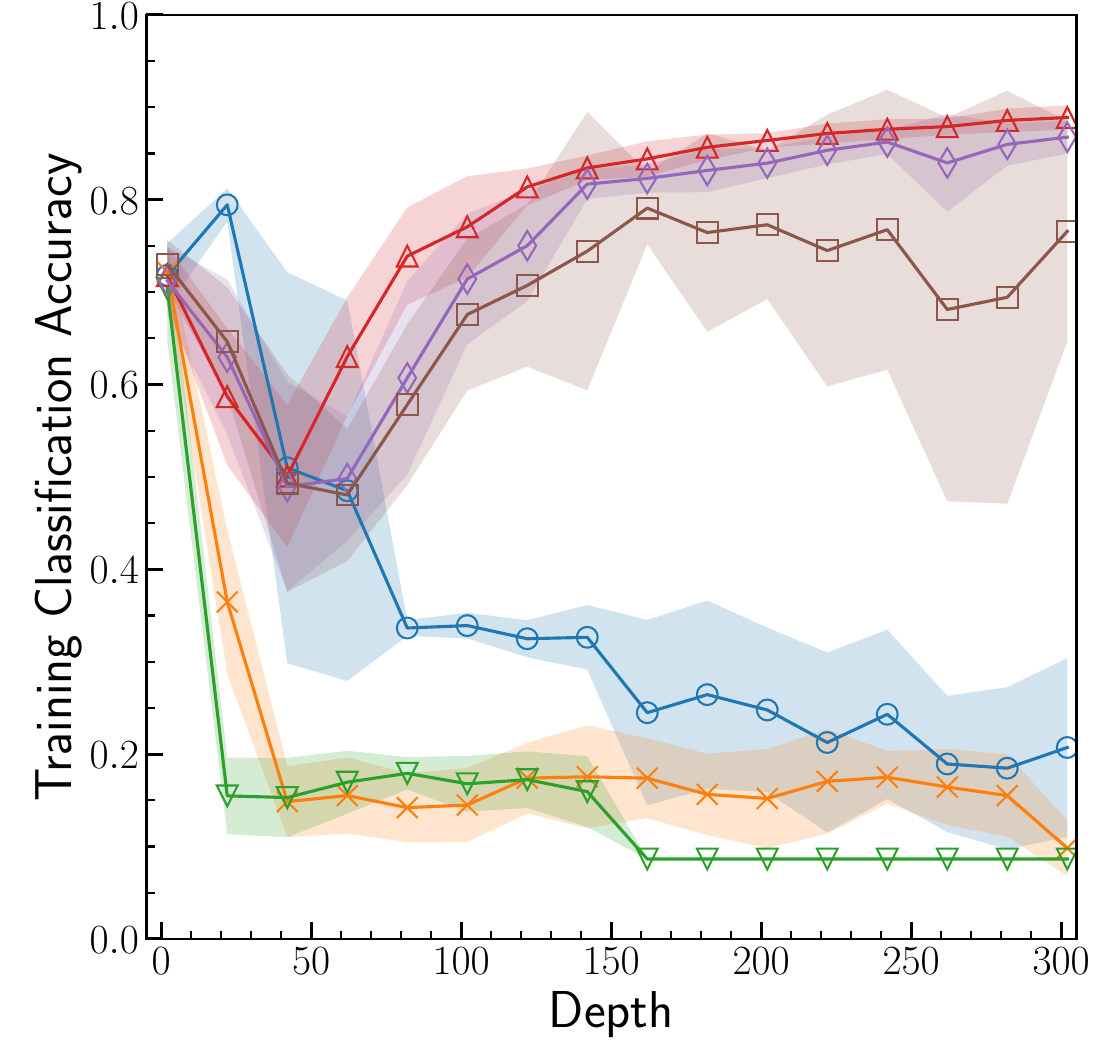}
         \caption{training classification accuracy on CiteSeer}
     \end{subfigure}
     \begin{subfigure}[b]{0.49\textwidth}
         \centering
         \includegraphics[width=0.7\textwidth]{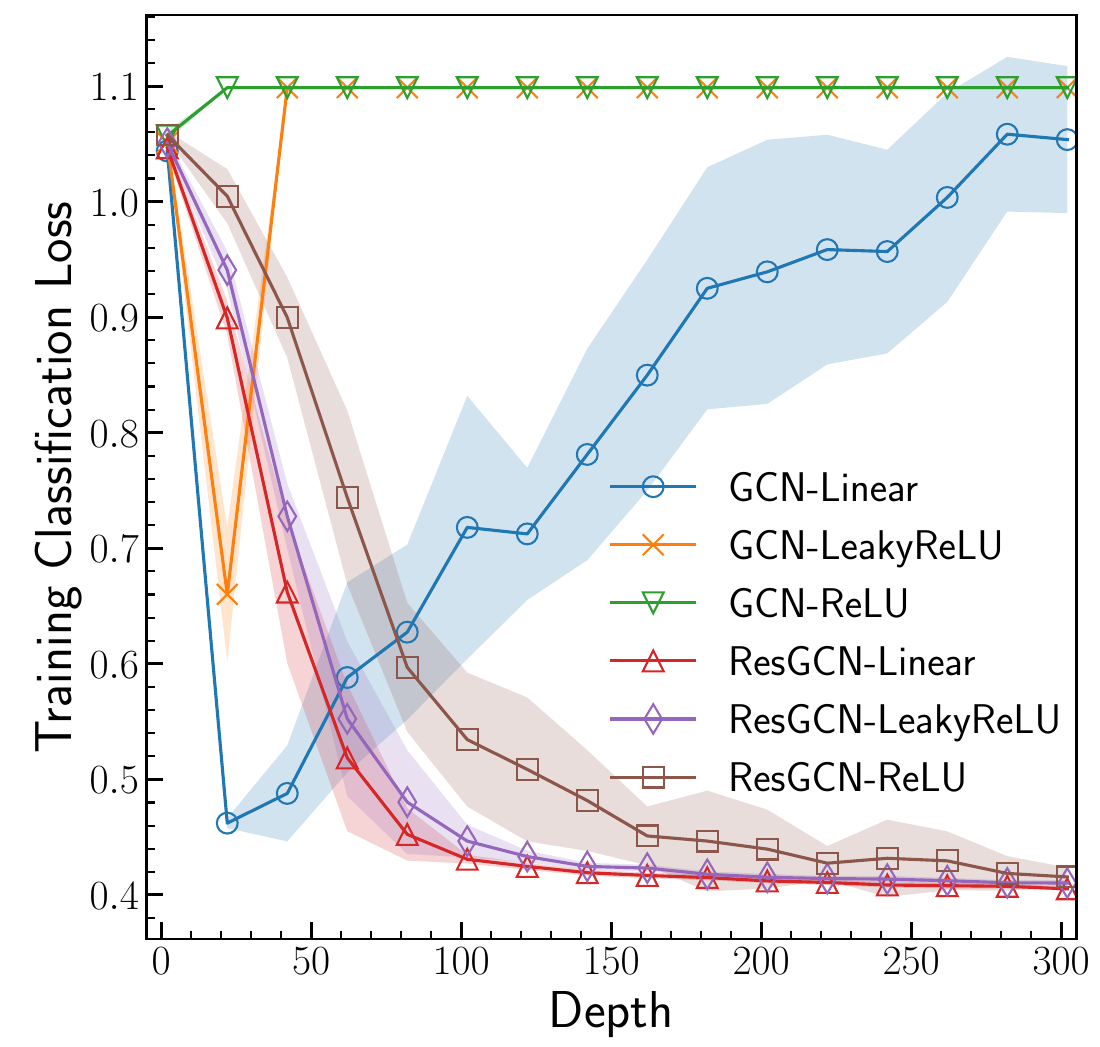}
         \caption{training loss on PubMed}
     \end{subfigure}
     \hfill
     \begin{subfigure}[b]{0.49\textwidth}
         \centering
         \includegraphics[width=0.7\textwidth]{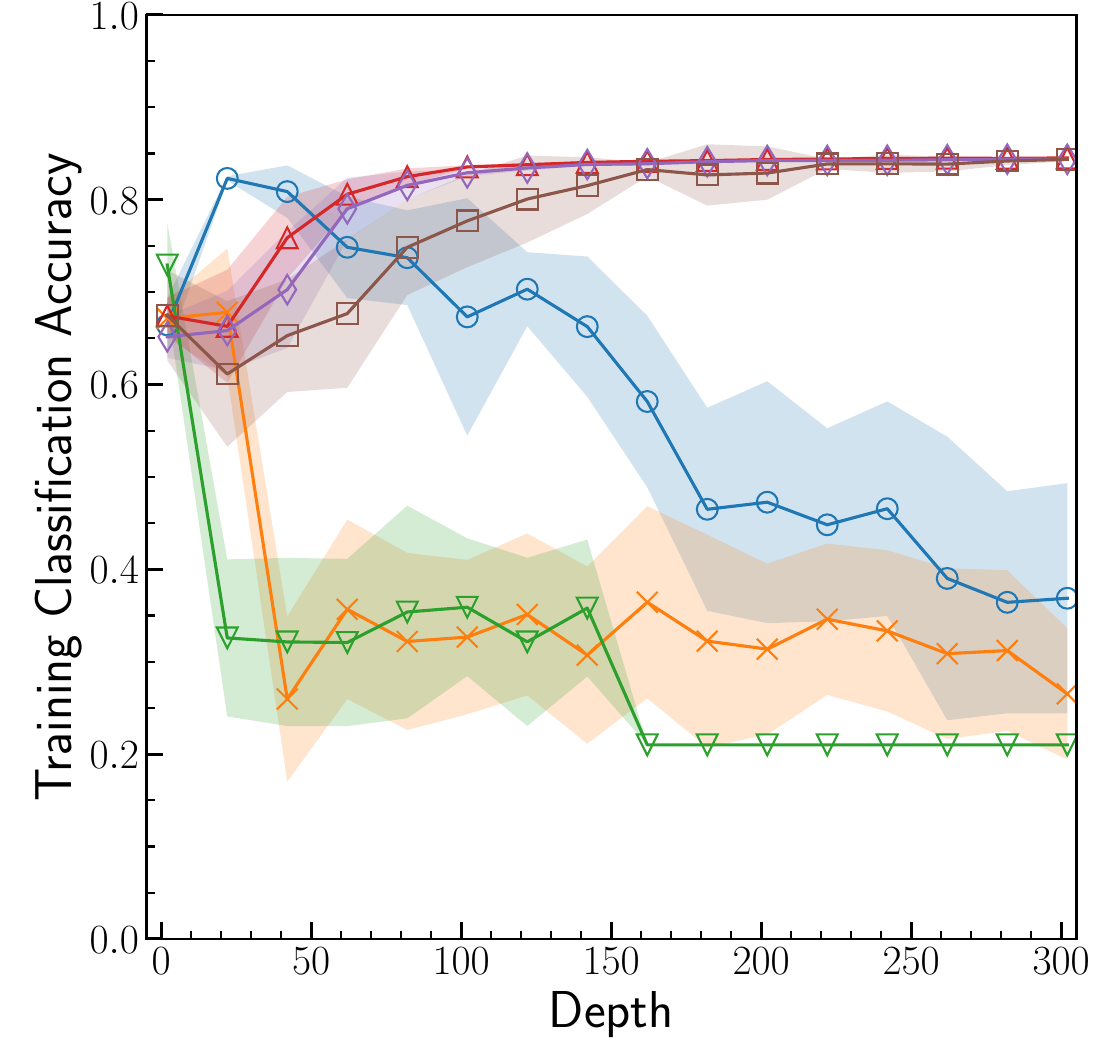}
         \caption{training classification accuracy on PubMed}
     \end{subfigure}
    \caption{Training loss and training classification accuracy of GCNs and residual GCNs}
    \label{fig:train_loss_accuracy}
\end{figure}

\begin{figure}[t]
     \centering
     \begin{subfigure}[b]{0.49\textwidth}
         \centering
         \includegraphics[width=0.7\textwidth]{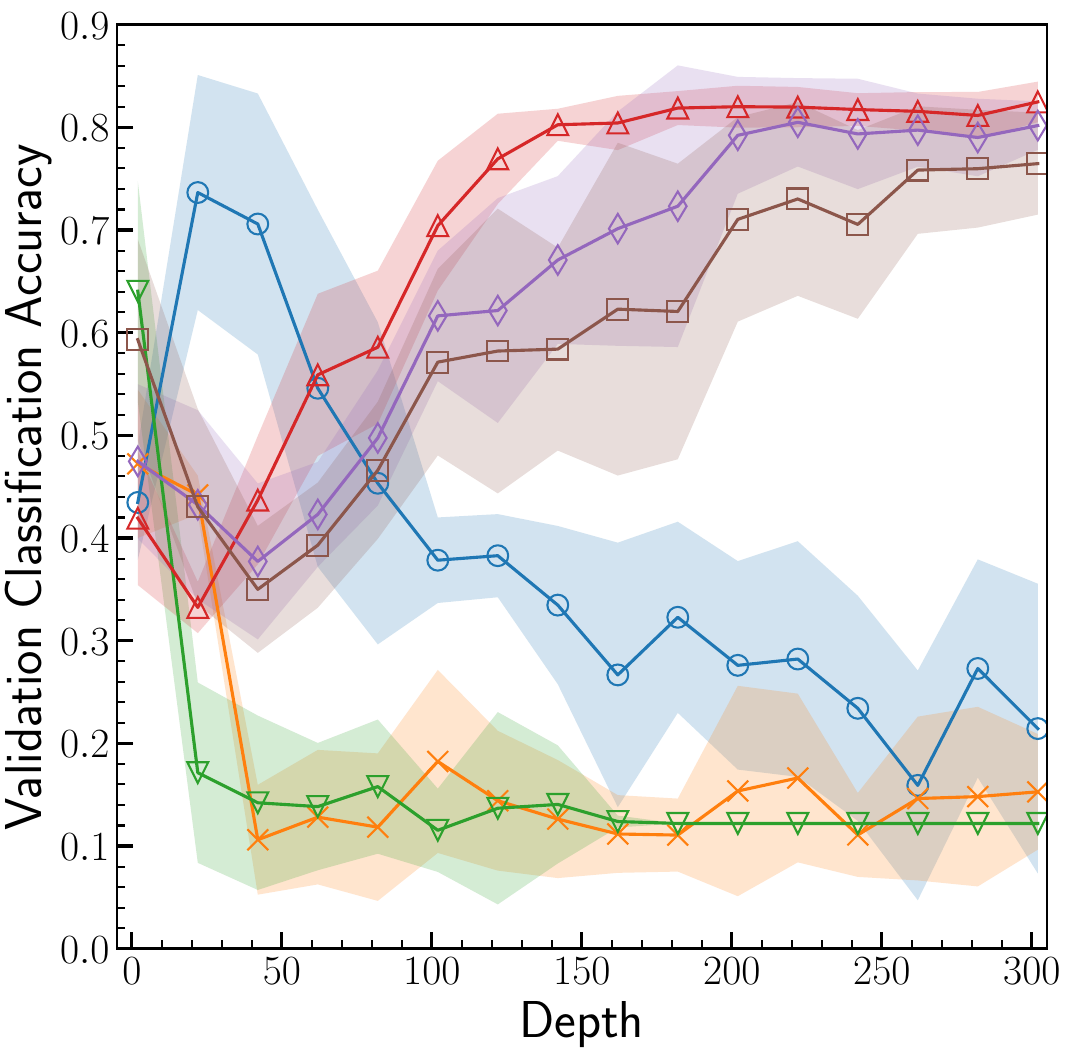}
         \caption{validation classification accuracy on Cora}
     \end{subfigure}
     \hfill
     \begin{subfigure}[b]{0.49\textwidth}
         \centering
         \includegraphics[width=0.7\textwidth]{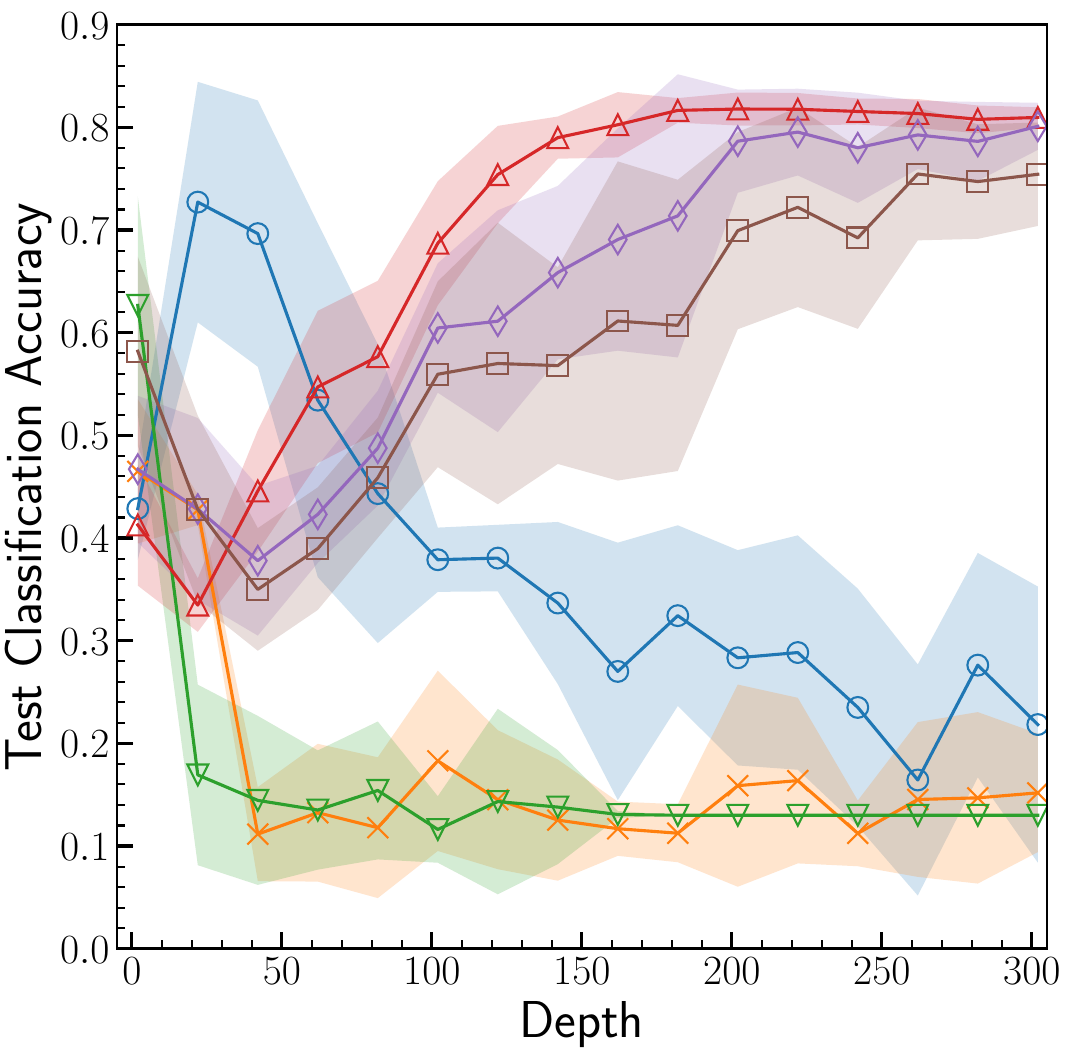}
         \caption{test classification accuracy on Cora}
     \end{subfigure}
     \begin{subfigure}[b]{0.49\textwidth}
         \centering
         \includegraphics[width=0.7\textwidth]{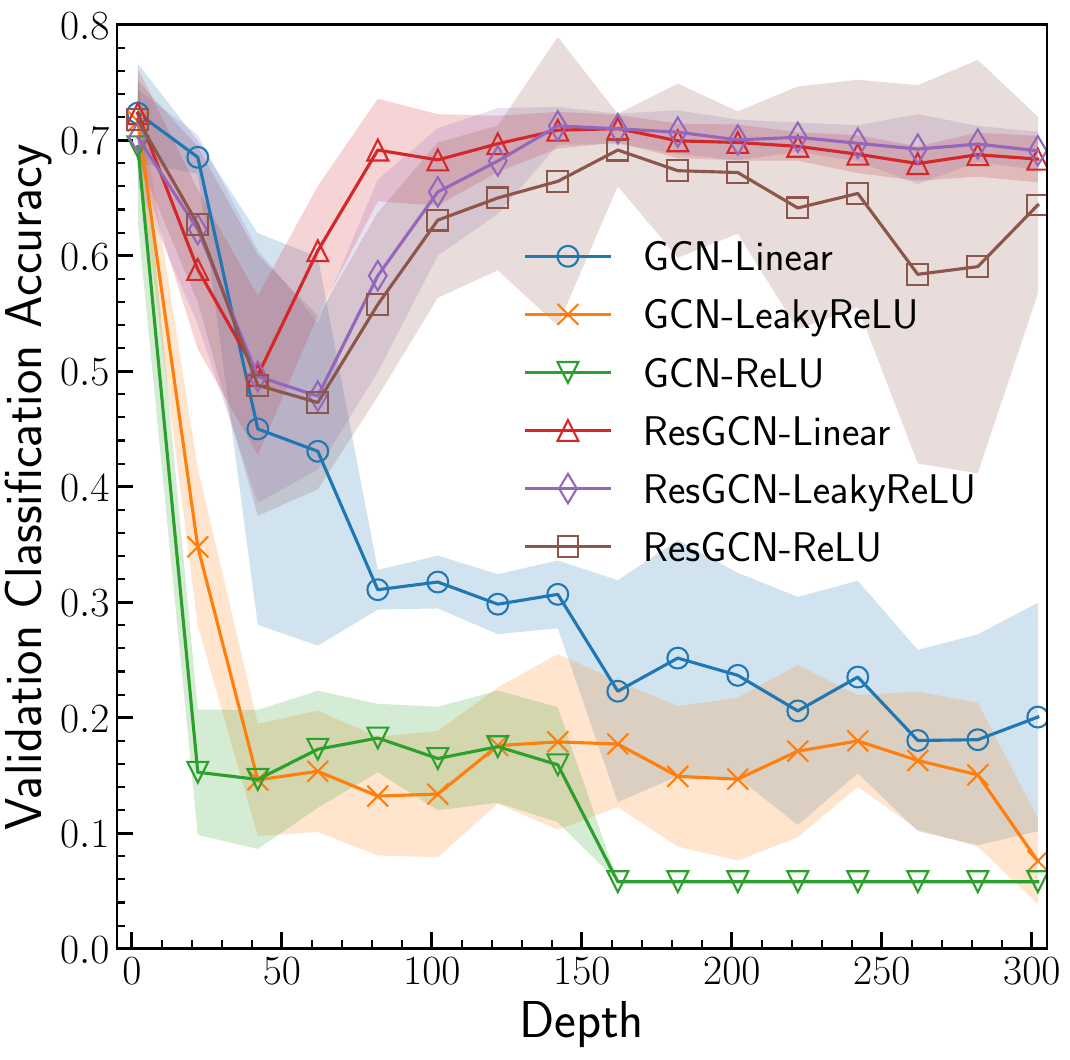}
         \caption{validation classification accuracy on CiteSeer}
     \end{subfigure}
     \hfill
     \begin{subfigure}[b]{0.49\textwidth}
         \centering
         \includegraphics[width=0.7\textwidth]{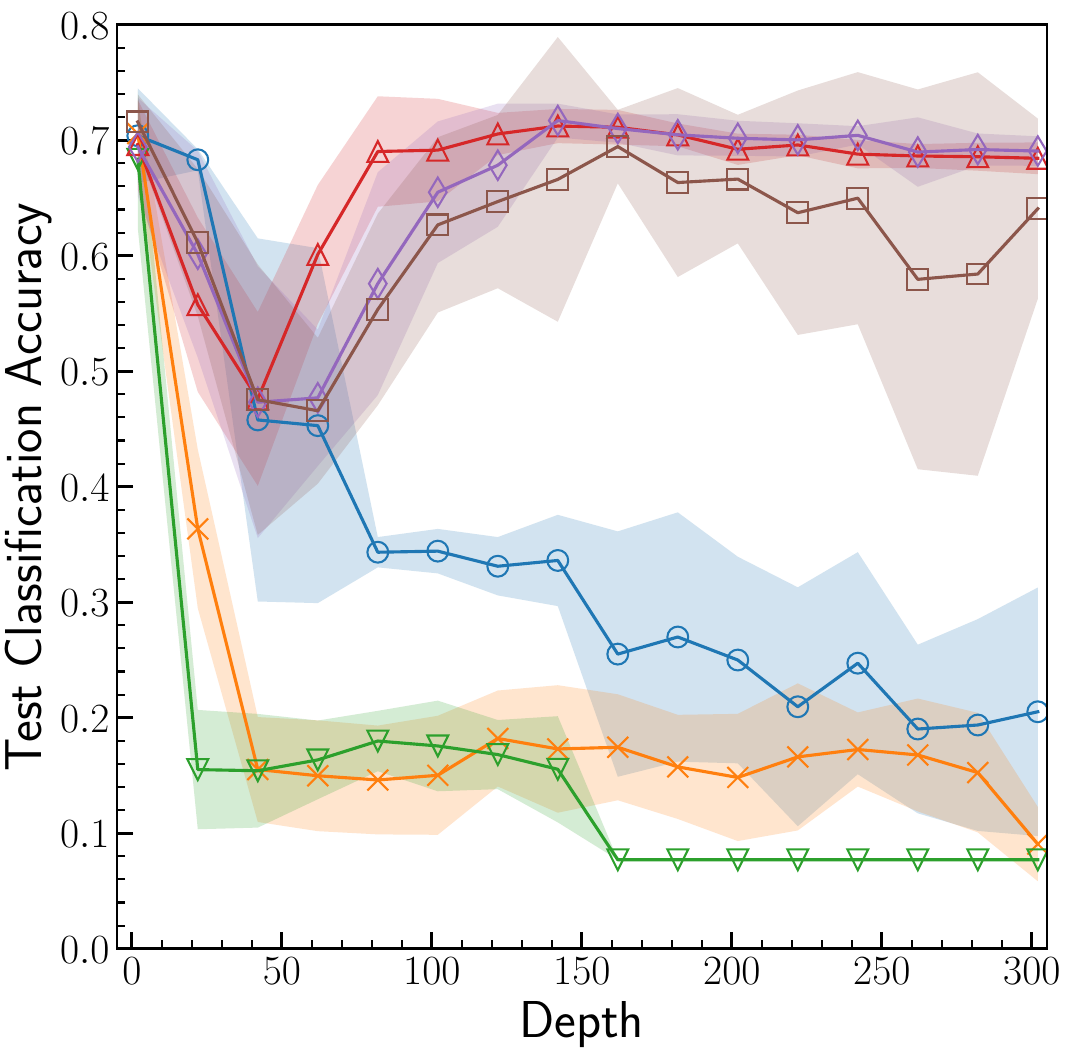}
         \caption{test classification accuracy on CiteSeer}
     \end{subfigure}
     \begin{subfigure}[b]{0.49\textwidth}
         \centering
         \includegraphics[width=0.7\textwidth]{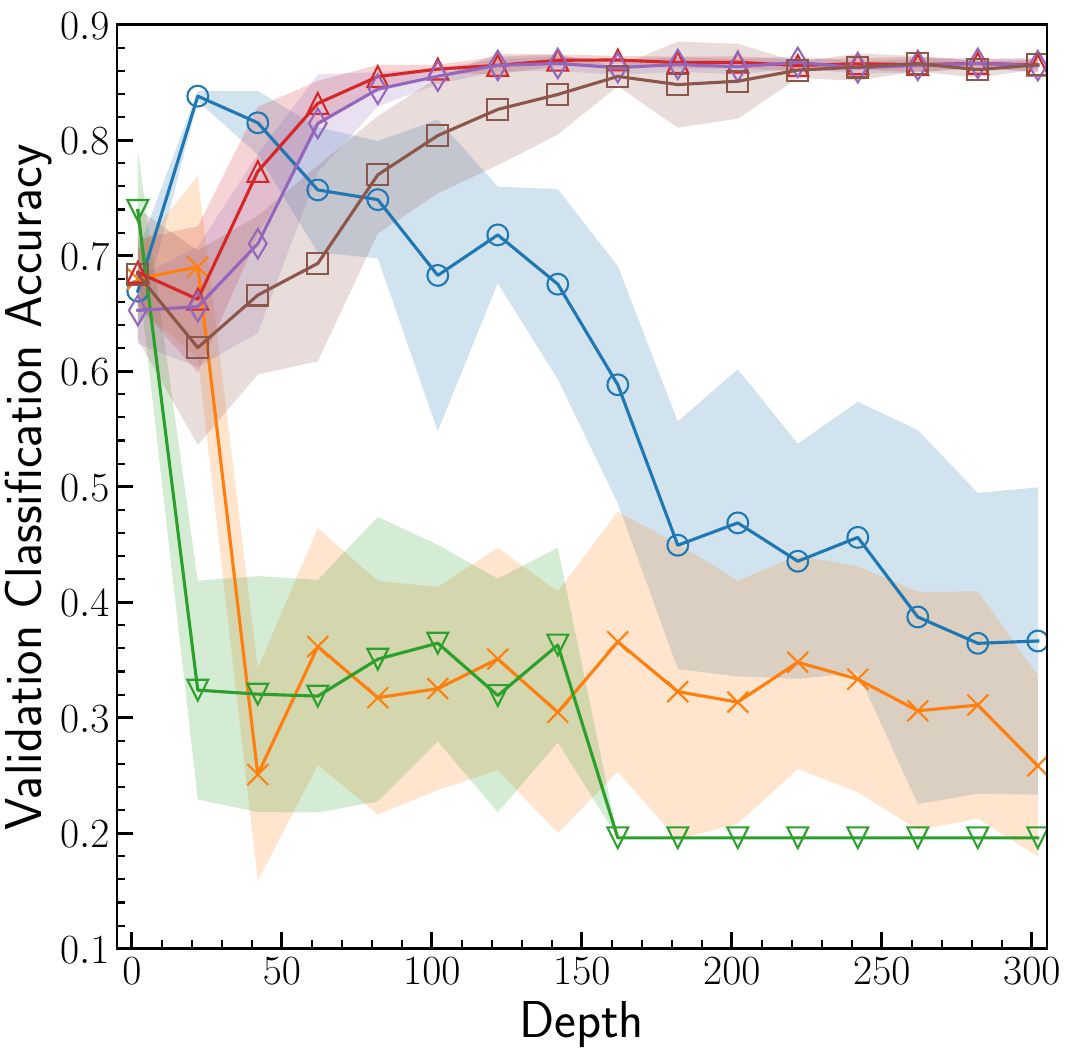}
         \caption{validation classification accuracy on PubMed}
     \end{subfigure}
     \hfill
     \begin{subfigure}[b]{0.49\textwidth}
         \centering
         \includegraphics[width=0.7\textwidth]{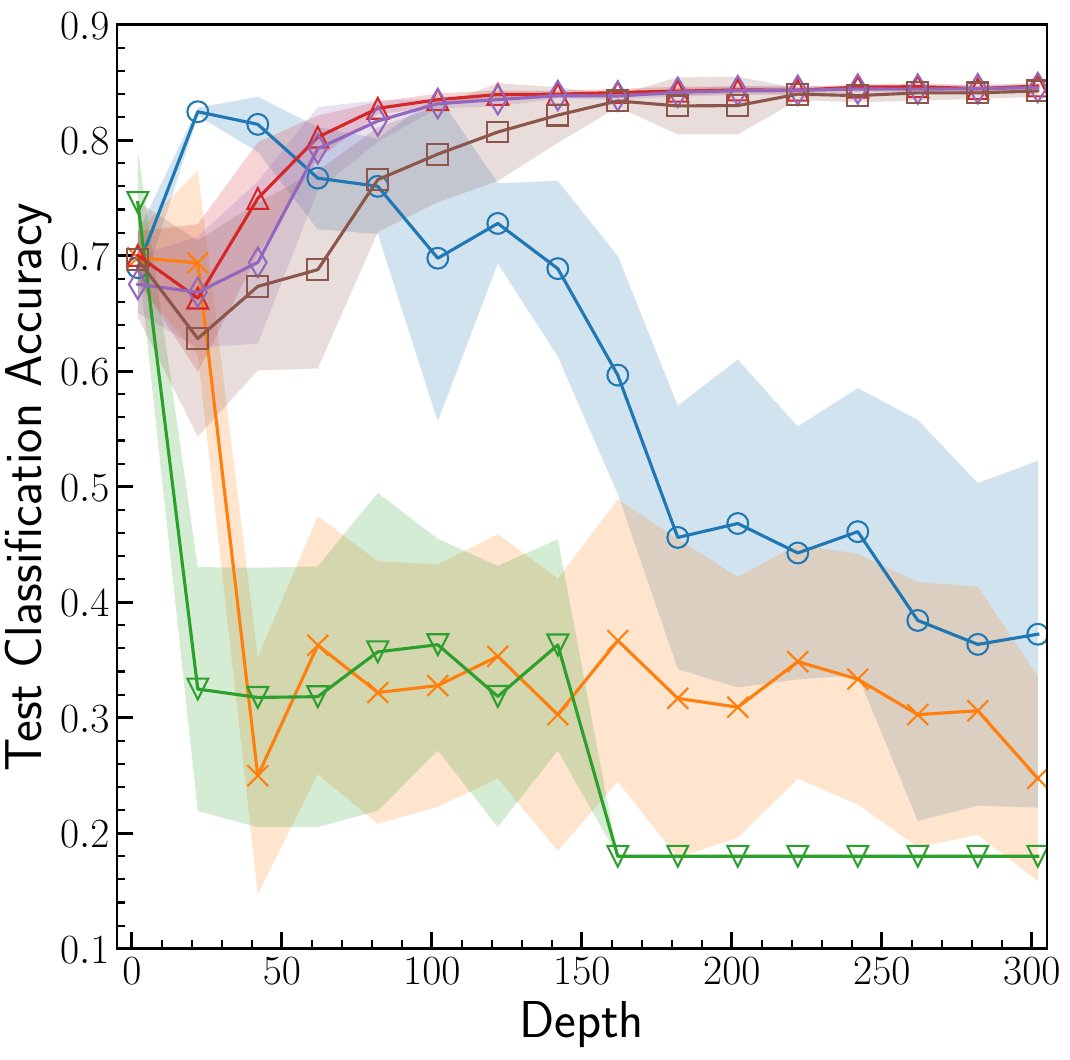}
         \caption{test classification accuracy on PubMed}
     \end{subfigure}
    \caption{Validation and test classification accuracy of GCNs and residual GCNs}
    \label{fig:val_test}
\end{figure}

Next, we train GCNs and residual GCNs with $0, 20, 40, \dots, 300$ message-passing layers, and compare their performance on the training, validation, and test sets. The training loss and training classification accuracy are shown in \Cref{fig:train_loss_accuracy}, while the classification accuracies on the validation and test sets are displayed in \Cref{fig:val_test}. In both figures, solid lines represent the average values, and shaded regions indicate the standard deviation.

As shown in \Cref{fig:train_loss_accuracy}, residual GCNs generally achieve lower training loss and higher training accuracy as the model depth increases. In contrast, deeper GCNs experience larger training loss and lower training accuracy. Similar trends are observed in the validation and test accuracies, as displayed in \Cref{fig:val_test}. Furthermore, deep residual GCNs consistently outperform their non-residual counterparts, highlighting the effectiveness of residual connections in building and training deep graph machine learning models.
\section{Conclusion}
\label{sec:conclude}

This work establishes the asymptotic oversmoothing rates for deep GNNs with and without residual connections using the multiplicative ergodic theorem. Under suitable assumptions, we show that the normalized vertex similarity of deep non-residual GNNs converges to zero at an exponential rate determined by the second-largest eigenvalue magnitude of the aggregation coefficient matrix. Furthermore, we precisely characterize the asymptotic behavior of the normalized vertex similarity of deep residual GNNs, proving that it either avoids exponential oversmoothing or exhibits a significantly slower rate compared to deep non-residual GNNs across several broad classes of parameter distributions. These findings highlight that incorporating residual connections effectively mitigates or prevents the oversmoothing issue. Our theoretical results are strongly corroborated by numerical experiments.

Let us also comment on the limitations and future works. First, the current analysis relies on the linearity of the dynamics, which does not account for nonlinear activations. Second, the assumption that the parameters of each layer are independently drawn from the same distribution may not hold for trained neural networks. Finally, our analytical framework assumes constant aggregation coefficients across layers, whereas, in practice, different aggregations could be used in different layers. These issues present important directions for future research.

\section*{Acknowledgments}
Z.C. thanks Runzhong Wang for the helpful discussions on oversmoothing. Y.P. is supported in part by NSF under CCF-2131115. P.R. is supported by NSF grants DMS-2022448 and CCF-2106377.

\bibliographystyle{alpha}
\bibliography{refs}

\end{document}